\documentclass{article} % For LaTeX2e

\PassOptionsToPackage{numbers, compress}{natbib}
\usepackage[square,numbers]{natbib}

\usepackage{iclr2022_conference, times}

\usepackage{amsmath,amsfonts,bm, esint, amsthm}
\usepackage{subcaption}

\usepackage{bbm}
\usepackage{amssymb}
\usepackage[T1]{fontenc}
\usepackage{enumitem}
\usepackage{xcolor}         % colors
\usepackage{graphicx}

\usepackage{mathtools}
\usepackage{BOONDOX-uprscr}

\definecolor{darkred_f}{RGB}{182, 85, 85}
\definecolor{darkblue_f}{RGB}{86, 116, 172}
\definecolor{darkorange_f}{RGB}{209, 136, 92}
\definecolor{darkgreen_f}{RGB}{106, 165, 110}

\newcommand{\rootg}{o_\mathcal G}

\def\eqref#1{equation~\ref{#1}}
\def\myeqref#1{Eq.~(\ref{#1})}
\def\1{\bm{1}}

\def\vd{{\bm{d}}}
\def\ve{{\bm{e}}}
\def\vf{{\bm{f}}}

\def\vk{{\bm{k}}}
\def\vl{{\bm{l}}}

\def\vr{{\bm{r}}}

\def\vt{{\bm{t}}}
\def\vu{{\bm{u}}}

\def\vx{{\bm{x}}}

\def\vP{{\bm{P}}}
\def\vY{{\bm{Y}}}
\def\vxi{{\bm{\xi}}}
\def\veta{{\bm{\eta}}}
\def\vI{{\bm{I}}}
\def\vN{{\bm{N}}}

\DeclareMathAlphabet{\mathsfit}{\encodingdefault}{\sfdefault}{m}{sl}
\SetMathAlphabet{\mathsfit}{bold}{\encodingdefault}{\sfdefault}{bx}{n}

\def\gC{{\mathcal{C}}}

\def\gE{{\mathcal{E}}}
\def\gF{{\mathcal{F}}}
\def\gG{{\mathcal{G}}}

\def\gK{{\mathcal{K}}}
\def\K{{\mathcal{K}}}

\def\gL{{\mathcal{L}}}
\def\gM{{\mathcal{M}}}
\def\gN{{\mathcal{N}}}

\def\gP{{\mathcal{P}}}
\def\gQ{{\mathcal{Q}}}

\def\gT{{\mathcal{T}}}

\def\gX{{\mathcal{X}}}
\def\X{{\mathcal{X}}}

\def\Y{{\mathcal{Y}}}

\newcommand{\E}{\mathbb{E}}

\newcommand{\R}{\mathbb{R}}

\newcommand{\inputx}{\bm \gX}

\newcommand{\infntk}{\Theta}
\newcommand{\barntk}{\overline{\Theta}}

\newtheorem{definition}{Definition}%[section]
\newtheorem{theorem}{Theorem}%[section]
\newtheorem{lemma}{Lemma}%[section]
\newtheorem{proposition}{Proposition}%[section]
\newcommand{\N}{{\mathcal N}}
\newcommand{\Ad}{{\mathscr A}}
\newcommand{\path}{{\text{path}}}

\DeclareMathOperator{\Span}{span}

\newcommand{\spatial}{{\mathscr S}}
\newcommand{\frequency}{{\mathscr F}}
\newcommand{\learning}{{\mathscr L}}
\newcommand{\shapep}{{{ \Lambda}_{\bm\gG} }}
\newcommand{\Gd}{{\mathcal G^{(d)}}}
\newcommand{\Nd}{{\mathcal N^{(d)}}}
\newcommand{\Ndzero}{{\mathcal N^{(d)}_0}}
\newcommand{\Ed}{{\mathcal E^{(d)}}}

\newcommand{\din}{{d_{\text{in}}}}

\newcommand{\textand}{{{\text{and}}}}

\newcommand{\sym}{{{\text{Sym}}}}

\newcommand{\sphere}{\mathbb S}

\newcommand{\sphered}{\mathbb S_{\din-1}}

\newcommand{\inputnodes}{{u\in\Ndzero}}

\newcommand{\barY}{{\overline Y}}

\newcommand{\nn}{{\mathscr N}}
\newcommand{\nnode}{{\mathscr n}}

\newcommand{\kk}{{\mathscr K}}
\newcommand{\hop}{{\mathscr H}_d}
\newcommand{\kkpen}{{\mathscr K}_{pen}}
\newcommand{\pkk}{\overline{\mathscr K}}
\newcommand{\residual}{{\mathscr R}}
\newcommand{\leaf}{{\mathcal N_0}}
\newcommand{\penultimate}{{\mathcal N_{-1}}}
\newcommand{\hidden}{{\mathcal N_1^{(d)}}}
\newcommand{\Hilbert}{{\mathbb H}}

\newcommand{\bmx}{{\bm\gX}} 

\newcommand{\spbar}{{\overline{\mathbb S}_{p-1}}}
\newcommand{\spherep}{{\mathbb S}_{p-1}}

\definecolor{darkred_f}{RGB}{182, 85, 85}
\definecolor{darkblue_f}{RGB}{86, 116, 172}
\definecolor{darkorange_f}{RGB}{209, 136, 92}
\definecolor{darkgreen_f}{RGB}{106, 165, 110}

\definecolor{plot_blue}{RGB}{66, 153, 225}
\definecolor{plot_orange}{RGB}{237, 137, 54} 
\definecolor{plot_red}{RGB}{245, 101, 101}
\definecolor{plot_green}{RGB}{72, 187, 120}
\definecolor{plot_purple}{RGB}{159, 122, 234}
\definecolor{plot_green2}{RGB}{56, 178, 172}
\definecolor{plot_pink}{RGB}{237, 100, 166}
\definecolor{plot_yellow}{RGB}{236, 201, 75}

\definecolor{plot_1}{RGB}{66, 153, 225}
\definecolor{plot_2}{RGB}{237, 137, 54} 
\definecolor{plot_3}{RGB}{245, 101, 101}
\definecolor{plot_4}{RGB}{72, 187, 120}
\definecolor{plot_5}{RGB}{159, 122, 234}
\definecolor{plot_6}{RGB}{56, 178, 172}
\definecolor{plot_7}{RGB}{237, 100, 166}
\definecolor{plot_8}{RGB}{236, 201, 75}

\definecolor{goldenyellow}{rgb}{1.0, 0.87, 0.0}

\definecolor{color0}{rgb}{1.        , 0.17254902, 0.1372549 }
\definecolor{color1}{rgb}{0.96078431, 0.39607843, 0.39607843 }

\definecolor{color2}{rgb}{0.78529412, 0.44705882, 0.51764706 }

\definecolor{color3}{rgb}{0.60980392, 0.49803922, 0.63921569 }

\definecolor{color4}{rgb}{0.43431373, 0.54901961, 0.76078431 }
\definecolor{color5}{rgb}{0.25882353, 0.6       , 0.88235294}
\definecolor{color6}{rgb}{0.05555556, 0.43398693, 1.   }

\usepackage{hyperref}
\hypersetup{
colorlinks,linkcolor={blue},citecolor={blue},urlcolor={red}}

\usepackage{color}

\newif\ifshowcomments
\showcommentsfalse
\ifshowcomments
    \usepackage{showlabels}
    \newcommand{\xlc}[1]{{\color{blue}[XLC: #1]}}
\else
    \newcommand{\xlc}[1]{}
\fi

\title{
Eigenspace Restructuring: 
a Principle of \\ Space and Frequency in Neural Networks
}

\author{Lechao Xiao \\
Google Research, Brain Team\\
\texttt{xlc@google.com} \\
}

\begin{document}

\maketitle

\begin{abstract}

Understanding the fundamental principles behind the massive success of neural networks is one of the most important open questions in deep learning. However, due to the highly complex nature of the problem, progress has been relatively slow. In this note, through the lens of infinite-width networks, a.k.a. neural kernels, we present one such principle resulting from hierarchical localities. It is well-known that the eigenstructure of infinite-width multilayer perceptrons (MLPs) depends solely on the concept {\it frequency}, which measures the order of interactions. We show that the topologies from deep convolutional networks (CNNs) restructure the associated eigenspaces into finer subspaces. In addition to frequency, the new structure also depends on the concept {\it space}, which measures the spatial distance among nonlinear interaction terms. The resulting fine-grained eigenstructure dramatically improves the network's learnability, empowering them to simultaneously model a much richer class of interactions, including Long-Range-Low-Frequency interactions, Short-Range-High-Frequency interactions, and various interpolations and extrapolations in-between. Additionally, model scaling can improve the resolutions of interpolations and extrapolations and, therefore, the network's learnability. Finally, we prove a sharp characterization of the generalization error for infinite-width CNNs of any depth in the high-dimensional setting. Two corollaries follow: (1) infinite-width deep CNNs can break the curse of dimensionality without losing their expressivity, and (2) 
scaling improves performance in both the finite and infinite data regimes. 

\end{abstract}

% \tableofcontents

\section{Introduction}

Learning in high dimensions is commonly believed to suffer from the curse of dimensionality, in which the number of samples required to solve the problem grows rapidly (often polynomially) with the dimensionality of the input \citep{bishop2006pattern}. Nevertheless, modern neural networks often exhibit an astonishing power to tackle a wide range of highly complex and high-dimensional real-world problems, many of which were thought to be out-of-scope of known methods \citep{krizhevsky2012imagenet, vaswani2017attention, devlin2018bert, silver2016mastering, senior2020improved, kaplan2020scaling}. What are the mathematical principles that govern the astonishing power of neural networks? This question perhaps is the most crucial research question in the theory of deep learning because such principles are also the keys to resolving fundamental questions in the practice of machine learning such as (out-of-distribution) generalization \citep{zhang2021understanding}, calibration \citep{Ovadia2019CanYT}, interpretability \citep{montavon2018methods}, robustness \citep{goodfellow2014explaining}.

Unarguably, there can be more than one of such principles. They are related to one or more of the three basic ingredients of machine learning methods: the data, the model, and the inference algorithm. Among them, the models, a.k.a. architectures of neural networks, are the most crucial innovation in deep learning that set it apart from classical machine learning methods. More importantly, the current revolution in machine learning is initialized by the (re-)introduction of convolution-based architectures \citep{krizhevsky2012imagenet, lecun1989generalization}, and subsequent breakthroughs are often driven by discoveries or applications of novel architectures (\citet{vaswani2017attention, devlin2018bert}). As such, identifying and understanding the fundamental roles of architectures are of great importance, which is the main focus of the current paper.  

In this paper, we take a step forwards by leveraging 
recent developments in overparameterized networks (\cite{poole2016exponential, daniely2016toward, schoenholz2016deep, lee2018deep, matthews2018, xiao2018dynamical, jacot2018neural,du2018gradienta, novak2019bayesian, lee2019wide, yang2019scaling} and many others.) These developments have discovered an important connection between neural networks and kernel machines: the Neural Network Gaussian Process (NNGP) kernels and the neural tangent kernels (NTKs). Under certain scaling limits, the former describes the distribution of the outputs of a randomly initialized network (a.k.a. {\it prior}), and the latter can describe the network's gradient descent dynamics.
Although recent work \citep{ghorbani2019limitations, yang2020feature} has identified several limitations of using them in studying the feature learning dynamics of practical networks, we show that they do capture several crucial and perhaps surprising properties of the architectural inductive biases.

\begin{figure}[t]
    \centering
    \includegraphics[width=\textwidth]{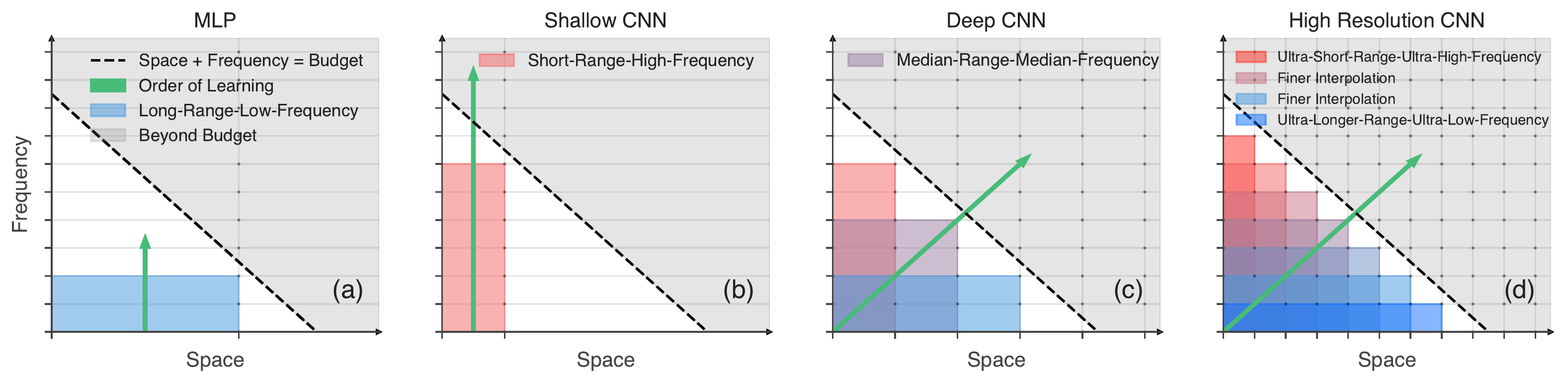}
    \caption{{\bf Architectural Inductive Biases.} A conceptual illustration of learnable functions vs architectures for four families of architectures. Each shaded box indicates the maximum learnable eigenspaces within a given compute/data budget ({\bf\color{black}{Dashed Line.}}) The complexity of an eigenspace is the sum of the complexities in space (X-axis) and in frequency (Y-axis). 
    From left to right: (a) MLPs can model {\bf\color{color5}Long-Range-Low-Frequency (LRLF)} interactions; (b) S-CNNs can model {\bf\color{color1}Short-Range-High-Frequency (SRHF)} interactions; 
    (c) Additionally, D-CNNs can also model interactions between {\bf\color{color5}LRLF} and {\bf\color{color1}SRHF}, a.k.a.,  {\bf\color{color3}Median-Range-Median-Frequency} interactions. (d) Finally, HS-CNNs can additionally model interactions of {\bf\color{color0} Ultra-Short-Range-Ultra-High-Frequency},
    {\bf\color{color6} Ultra-Long-Range-Ultra-Low-Frequency}, and {\bf\color{color2} finer} {\bf\color{color4}interpolations} in-between. The {\bf \color{plot_green}{Green Arrow}} indicates the direction of expansion of learnable eigenspaces when increasing the compute/data budget.}
    
    \label{fig:high-level-picture}  
\end{figure}

Our main contribution is the {\it eigenspace restructuring} theorem. When the input dimension is sufficiently large and the widths of the network approach infinity, it characterizes a mathematical relation between the network's architecture and its learnability through a trade-off between {\it space} and {\it frequency}, providing novel insights behind the mystery power of deep CNNs (more generally, hierarchical localities \citep{deza2020hierarchically, vasilescu2021causalx}.) We summarize our main contributions below; see Fig.~\ref{fig:high-level-picture} for a conceptual illustration.
\begin{enumerate}
    \item The learning order (see {\bf \color{plot_green}{Green Arrow}} in Fig.~\ref{fig:high-level-picture}) of eigenfunctions is governed by the learning index (LI), the sum of the frequency index (FI), and the spatial index (SI), which can be characterized precisely by the network's topology in high dimensions. 
    \item By leveraging a deep analytical result from \cite{mei2021generalization}, the LI provides a sharp characterization of the generalization error for infinite-width CNNs of arbitrary depth. As a consequence, we prove that deep CNNs (D-CNNs) can break the curse of dimensionality without losing their expressivity. In addition, we provide precise characterizations of the generalization benefits of depth for CNNs via the spatial index and the frequency index. 
    % \item In the high-dimensional setting, for a fixed budget of compute or training points, 
    % The former is defined by the order of interactions, and the latter is by the length of a minimum spanning tree (MST) that contains all interaction terms.    
    \item There is a trade-off between {\it space} and {\it frequency}: within a fixed (compute/data) budget, it is impossible to model generic Long-Range-High-Frequency interactions. MLPs can model {\bf\color{color5}Long-Range-Low-Frequency (LRHF)} (Fig.~\ref{fig:high-level-picture} (a)) interactions but fail to model {\bf\color{color1}Short-Range-High-Frequency (SRHF)} interactions, while shallow CNNs (S-CNNs) are the opposite (Fig.~\ref{fig:high-level-picture} (b)). Remarkably, D-CNNs can simultaneously model both and various interpolating interactions between them (e.g., {\bf\color{color3}Median-Range-Median-Frequency (MRMF) interactions}. (Fig.~\ref{fig:high-level-picture} (c))) 
    
    \item In addition, {\it high-resolution} CNNs (HR-CNNs), networks with EfficientNet-type \citep{tan2019efficientnet} of model scaling, further broaden the class of learnable functions to contain (1) { extrapolation}: {\bf\color{color6} Ultra-Long-Range-Ultra-Low-Frequency} and the {\bf\color{color0}dual} interactions and (2) {\bf\color{color2} finer} {\bf\color{color4}interpolations} interactions (Fig.~\ref{fig:high-level-picture} (d).)
    \item Finally, we verify the above claims empirically for neural kernels and finite-width networks using {\it SGD + Momentum} for datasets and networks of practical sizes.   
\end{enumerate}
% Our results provides a precise mathematical characterization of the inductive biases of (infinite-width) CNNs. 
% Overall, we show that the inductive biases from the architecture alone is sufficient to break the curse of dimensionality in many settings. 
The rest of the paper is organized as follows. Sec.~\ref{Sec: toy example} provides a toy example to help understand the motivations of the paper and to explain two core concepts: {\it spatial} and {\it frequency} complexities, which jointly define a function-and-architecture dependent complexity measure.
Sec.~\ref{sec:linear} briefly recaps the role of eigenstructures in studying the learning dynamics of linear models, and the connection between (in)finite-width networks and kernels (NNGP kernels and NTKs). Sec.~\ref{sec:neural-computation} introduces the main notations and expresses the neural network computations via directed acyclic graphs (DAGs). The definition of the learning index and the main results are presented in
Sec.~\ref{sec:main results}. We provide interpretations and experimental support for the main results in Sec.\ref{sec: interpretation} and Sec.\ref{sec:experiment}, resp. Finally, additional related work and further discussions are in Sec.~\ref{sec:related} and Sec.~\ref{sec:discussion}.

\section{Motivation and a Toy Example}\label{Sec: toy example}
Before diving into the technical details, it is helpful to have one toy example in mind. Consider learning the following polynomials in $\sphere_{d-1}\equiv \{x\in \mathbb R^{d}: \|x\|_2=1\}$ using (in)finite-width neural networks and for concreteness we have set $d=10$: 
\begin{align}
    f_1(x) = x_9,\,\,
    f_2(x) = x_0x_1, \,\,
    f_3(x) = x_0x_8,\,\,
    f_4(x) = x_6x_7(x_6^2 - x_7^2),  \,\,
    f_5(x) = x_2x_3x_5
\end{align}
Which architectures (e.g. MLPs, CNNs) can efficiently learn $f_i$ or the sum of $f_i$? More precisely, (1) if we have sufficient training data, how much time (compute) is required to learn $f_i$ for a given architecture? (2) Alternatively, if we have a sufficient amount of compute, how much data is needed to learn $f_i$? To address these questions, one crucial step is to provide a meaningful definition of the ``learning complexity" of a function $f_i$ under architecture $\gM$. Denote this complexity associated to compute and to data by $\gC_C(f_i; \gM)$ and $\gC_D(f_i; \gM)$, resp. With such the complexity properly defined, the questions are reduced to solving the min-max problem $\min_{\gM}\max_i\{\gC_{C/D}(f_i; \gM)\}$, if the task is, e.g., to learn the sum of $f_i$. 

Let's focus on the complexity. For infinite-width MLPs, a.k.a. inner product kernels, it is well-known that they have the inductive biases \citep{bach2017breaking, yang2020finegrained, ghorbani2020linearized, mei2021generalization} (known as the frequency biases) that the model prioritizes learning low-frequency modes (i.e., low degree polynomials) over high-frequency modes. In addition, the models require $\sim d^r$ many data points to learn {\it any} degree $r$ polynomials in $\mathbb S_{d-1}$. The frequency biases of MLPs are the consequence of the fact that the eigenspaces of inner product kernels are structured based only on frequencies. 
Specific to our example, for MLPs, the order of learning is $f_1 / f_2, f_3/f_5/f_4$ and it requires about $d/d^2, d^2/d^4$ ($d=10$ in the example here) many data points to learn the functions. Clearly, the model is very inefficient in learning $f_4$ and, more generally, any high-frequency modes. 

To design new models that improve the learning efficiency, we must take the modality of the functions into account, which MLPs have overlooked. 
We observe that:  (1) although of high-frequency, $f_4$ depends only on two consecutive terms $x_6$ and $x_7$, which are spatially close; (2) in contrast, $f_3(x)=x_0x_8$ is of low-frequency but the spatial distance between the two interaction terms $x_0$ and $x_8$ are ``large'' from each other; (3) the function $f_5(x)=x_2x_3x_5$ is somewhere in-between: the order of interaction is 3 (lower than that of $f_4$) and the spatial distance (not yet defined) among interaction terms is conceptually ``smaller" than that of $f_3(x)=x_0x_8$, but ``greater" than the terms in $f_5$. Using the terminologies from the introduction, the functions $f_2/f_3/f_5/f_4$ model interactions of types: Short-Range-Low-Frequency/Long-Range-Low-Frequency/Median-Range-Median-Frequency/Short-Range-High-Frequency. 
By {\it Range} we mean the distance among the nonlinear interacting terms, and by {\it Frequency} we mean the order(=degree) of interactions.  
Clearly, the MLPs are inefficient since they totally ignore the ``spatial structure" of the functions. 
As such, a good architecture must balance the "spatial structure" and the ``frequency structure" of the functions. 
For the same reason, a good complexity measure (1) must be able to capture both the {\it frequency} of the functions and the {\it spatial distance} among interaction terms; (2) must be able to precisely characterize the data and the computation efficiency of learning and their dependence on architectures. 
The learning index mentioned in the introduction is defined to meet these two conditions. It is the sum of the frequency index and the spatial index. The former measures the order (=degree=frequency) of interactions, depending on how the network partitions the input into patches. The latter measures the spatial distance among the interaction terms, depending on how the network hierarchically organizes these patches. 
Later we show that, in the high-dimensional setting, the learning index provides a sharp characterization for the learnability (in terms of compute and learning efficiency) of eigenfunctions, and certain CNNs can perfectly balance the learning of $f_3$ and $f_4$, i.e., informally   
\begin{align}
\mathcal C_{C/D}(f_3; \text{CNN}) \approx \mathcal C_{C/D}(f_4; \text{CNN}) \approx \mathcal C_{C/D}(f_{2/3}; \text{MLP}) << \mathcal C_{C/D}(f_4; \text{MLP})
\end{align}
A crucial step in the paper is to understand the exact meaning of {\it spatial distance} and define it properly. It turns out that this distance and thus the spatial index relies on the topology of the network through the length of the minimum spanning tree(s) (MSTs) connecting the nonlinear interacting terms of the eigenfunction of interest to the output of the network. We explain, conceptually, why MSTs emerge in the definition of the spatial index below. Heuristically, we need to compute the dependence of the output function (i.e., the function defined by the network) on an eigenfunction $f$ of interest, which is a polynomial in our setting. This amounts to computing certain mixed derivatives w.r.t. to the interacting terms in $f$. After applying the chain rule and the product rule, this dependence can be expressed as a sum of many ``paths" (possibly with repeated edges) from the output to the interacting terms. This sum is intractable in general. However, in certain scaling limits, one can indeed compute the exponent of the leading terms w.r.t. to the input dimension $d$. Not surprisingly, this exponent is the infimum of the lengths of all such ``paths", namely, the length of the MST(s) connecting interacting terms to the output. To formally define the spatial index, we need to express the neural network computations using DAGs (Sec.\ref{sec:neural-computation}) and properly define the shapes of DAGs in high dimensions (Sec.~\ref{sec:main results}).

\section{Linear and linearized models}\label{sec:linear}
% \subsection{Exact Linear Models}
As a warm-up exercise, we briefly go through the training dynamics of a linear regression model.
The goal is to explain the relation between eigenstructures and training dynamics. 
\subsection{Linear Regression}
Let $\{(\vx_i, y_i )\in\mathbb R^d \times \mathbb R: i \in [m]\}$ be the training set, where $m\in\mathbb N$. %, where $\gX\subseteq \mathbb R^d$ and $\gY\subseteq \mathbb R$ denote the inputs and the labels. 
For convenience, we also use $\X\in \mathbb R^{m\times d}$ and $\Y\in\mathbb R^{m}$ to denote the input matrix and label vector resp., where the $i$-th row of $\X$ and $\Y$ are $\vx_i^T$ and $y_i$ resp. Let $J:\mathbb R^d \to\mathbb R^{1\times n}$ be a feature map and $J(\X)\in \mathbb R^{m\times n}$ denote the features of the inputs. The task is to learn a linear function $f(\vx, \theta) \equiv J(\vx)\theta$ to minimize the MSE of the residual $\residual(\X, \theta) \equiv f(\X, \theta) - \Y$, 
$$
\frac 1 2  \|\residual(\X, \theta)\|_2^2 \equiv \frac 1 2 \sum_{i\in [m]} (f(\vx_i, \theta) - y_i)^2 \, .
$$
Here $\theta\in\mathbb R^n$, a column vector, is the (trainable) parameter of the linear model. 
Then, by the chain rule, the gradient flow dynamics can be written as 
\begin{align}\label{eq:residual dynamcis}
    \frac d {dt} \residual(\X, \theta) = - J(\X) J^T(\X) \residual(\X, \theta)  \equiv - \K(\X, \X) \residual(\X, \theta) \, . 
\end{align}
Since the feature kernel $\K(\X, \X) = J(\X) J^T(\X)\in\mathbb R^m\times \mathbb R^m$ is constant in time, the above ODE can be solved in closed form. Let $\hat\K(j)/\vu_j$ be the $j$-th eigenvalue/eigenvector of $\K(\X, \X)$ in descending order. 
By initializing $\theta=0$ at time $t=0$ and denoting the projection by $\eta_j = \vu_j^T \residual(\X, 0)$,
the dynamics of the residual and the loss can be reduced to  
\begin{align}\label{eq:linear-eigen-mode}
\residual(\X, \theta_t) = \sum_{j\in [m]}  e^{-\hat \K(j) t} \eta_j \vu_j, 
\quad \mathcal L(\theta_t) = \frac 1 2\sum_{j\in [m]}  e^{-2\hat \K(j) t} \eta_j^2 \, .
\end{align}

Therefore, 
to make the residual loss in $\vu_j$ smaller than some $\epsilon>0$, namely, $\frac 1 2 (e^{-\hat \K(j) t} \eta_j )^2 < \epsilon$,
the amount of time needed is $t > \log\frac{2\epsilon}{\eta_j^2}/(2\hat\K(j)) $. The larger $\hat \K(j)$ is, the shorter amount of time it takes to learn $\vu_j$. In addition, if we know the distribution of $\{\sum_{j>i} \eta_j^2: i\in[m]\}$, then we can plot the scaling law of the loss \citep{kaplan2020scaling, bahri2021explaining}, which is roughly $(\hat \K(i)^{-1}, \sum_{j> i} \eta_j^2)_{ i\in[m]}$.

Although simple, linear models provide us with the most useful intuition behind the relation between the eigenstructure of the kernel matrix and the learning dynamics of the associated network. 

\subsection{Linearized Neural Networks: NNGP Kernels and NT Kernels}
Let $f(x, \theta)$ be a general function, e.g. $f$ is a neural network parameterized by $\theta$. Similarly, 
\begin{align} 
% \label{eq:residual dynamcis}
    \frac d {dt} \residual(\X, \theta) = - J(\X, \theta) J^T(\X, \theta) \residual(\X, \theta)  \equiv - \K(\X, \X; \theta) \residual(\X, \theta) \,. 
    \label{eq:new-residual-dynamcis}
\end{align}
However, the kernel $\K(\X, \X; \theta)$ depends on $\theta$ via the Jacobian $J(\X, \theta)$ of $f(\X, \theta)$, and evolves with time. The above system is unsolvable in general. However, under certain parameterization methods (namely, the NTK-parameterization, see \citet{sohldickstein2020infinite}) and when the network is sufficiently wide, this kernel does not change much during training and converges to a deterministic kernel called the NTK \citep{jacot2018neural},  
\begin{align}
    \K(\X, \X; \theta) \to \Theta(\X, \X) \quad \text{as      width} \to \infty. 
\end{align}
The residual dynamics becomes a constant coefficient ODE again 
% \begin{align}
$
    % \frac d {dt} 
    \dot \residual(\X, \theta) = - \Theta(\X, \X) \residual(\X, \theta) . 
$
%     \label{eq:new-residual-dynamcis}
% \end{align}
To solve this system, we need the initial value of $\residual(\X, \theta)$. Since the parameters $\theta$ are often initialized with iid standard Gaussian variables, as the width approach infinity, the logits $f(\X, \theta)$ converge to a Gaussian process (GP), known as the neural network Gaussian process (NNGP). Specifically, $f(\X, \theta) \sim \gN(\bm 0; \kk(\X, \X))$, 
% \begin{align}
%     f(\X; \theta) \sim \gN(\bm 0; \kk(\X, \X)), \label{eq:nngp-distribution}
% \end{align}
where $\kk$ is the NNGP kernel. Note that one can also treat infinite-width networks as Bayesian models, a.k.a. Bayesian Neural Networks, and apply Bayesian inference to compute the posteriors. This approach is equivalent to training {\it only} the network's classification layer \citep{lee2019wide} and the gradient descent dynamics is described by the kernel $\kk$.  

As such, there are two natural kernels, the NTK $\Theta$ and the NNGP kernel $\kk$, associated to infinite-width networks, whose training dynamics are governed by constant coefficient ODEs. To make progress, it is tempting to apply Mercer's Theorem to eigendecompose $\Theta$ and $\kk$, e.g.,  
\begin{align} \label{eq:mercer's theorem}
    \kk(x, \bar x) = \sum \hat \kk(j) \phi_j(x)\phi_j(\bar x)\quad \textand \quad 
    \Theta(x, \bar x) = \sum \hat \Theta(j) \psi_j(x)\psi_j(\bar x)
\end{align}
One advantage of applying this decomposition is that it almost has no constraint on the kernels and the inputs. However, this decomposition is too coarse to be useful since it can hardly provide fine-grained information about the eigenstructures. E.g, it is not clear what are the corrections to \myeqref{eq:mercer's theorem}  when changing the architecture from a 2-layer CNN to a 4-layer CNN. For this reason, we choose to work on ``concrete" input spaces (product of hyperspheres) with richer mathematical structures on which we can perform calculus (namely, harmonic analysis on spheres). Our primary goal is to characterize the precise analytical dependence of the decomposition  \myeqref{eq:mercer's theorem} on the network's topology in the high-dimensional limit.  

\section{Neural computations on DAGs}\label{sec:neural-computation}
This section aims to express the finite-width and the infinite-width neural network computations (neural computations) through DAGs. The following section will rely crucially on neural computations on DAGs to define the spatial and frequency indices. 

For a positive integer $p$, let $\spherep$ denote the unit sphere in $\mathbb R^p$ and $\spbar = \sqrt p \spherep$, the sphere of radius $\sqrt p$ in $\mathbb R^p$. We introduce the normalized  sum (integral) 
\begin{align}
\displaystyle\fint_{x\in X} f(x) \equiv |X|^{-1} \sum_{x\in X}f(x)  
\quad 
\left(  \quad \fint_{x\in X} f(x) \equiv \mu(X)^{-1} \int_{x\in X}f(x) \mu(dx)      \right) 
\end{align}
where $X$ is a finite set (a measurable set with a finite positive measure $\mu$).

We find it more convenient to express the computations in neural networks, and in neural kernels via DAGs \citep{daniely2016toward}, as both computations are of {\it recursive} nature. The associated DAG of a network can be thought of as the original network by setting all its widths (or the number of channels for CNNs) to 1. As such, changing the widths of a network won't alter the associated DAG. 
Let $\mathcal G = (\mathcal N, \mathcal E)$ denote a DAG, where $\mathcal N$ and $\mathcal E$ are the nodes and edges, resp. We always assume the graph to have a unique output node $\rootg$ and is an ancestor of all other nodes. 
Denote $\leaf \subseteq \N$ the set of input nodes (leaves) of $\gG$, i.e., the collection of nodes with no child. 
Each node $u\in\N$ is associated with a pointwise function $\phi_u:\mathbb R\to \mathbb R$, which is normalized in the sense  $\E_{z\in\N(0, 1)} \phi_u^2(z) = 1 \,. $
It induces a function $\phi_u^*: I\equiv [-1, 1] \to I$ defined to be
$$
    \phi^*_u(t) = \E_{(z_1, z_2)\in\N_t} \phi_u(z_1) \phi_u(z_2) \,.
$$
Here $\N_t$ denotes a pair of standard Gaussians with correlation $t$. We associate each $u\in\gN$ a finite-dimensional Hilbert space $\Hilbert_u$, and each $uv\in\gE$ (where the first node $u$ is the parent) a bounded linear operator $\mathcal L_{uv}: \Hilbert_v\to \Hilbert_u$. Let 
\[\bmx\equiv \prod_{u\in\leaf} \bmx_u 
\equiv\prod_{u\in\leaf}\overline \sphere_{\text{dim}(\Hilbert_u)-1} \subseteq \prod_{u\in\leaf} \Hilbert_u
\quad \textand \quad 
\bm I = I^{|\leaf|} 
\]
be the input {\it tensors} and the input {\it correlations} to the graph $\mathcal G$, resp. We associate two types of computations to a DAG: finite-width neural network computation and kernel computation,  
\begin{align}
    \nn_{\gG}: \bmx \to \Hilbert_{\rootg} \quad \text{and}\quad \mathscr K_{\gG}: \bm I\to I \, ,
\end{align}
resp. They are defined recursively as follows 
\begin{align}
\nn_u(\vx) &= \phi_u \left(\sum_{v: uv\in\gE} \gL_{uv} (\nn_v(\vx)) \right) \quad  &&\text {if} \quad u\notin \leaf \quad  \text{else} \quad \nn _u(\vx) = \vx_u
\\
\kk_u(\vt) &= \phi_u^*\left(\fint_{v: uv\in\gE} \kk_v(\vt)\right)
\quad &&\text {if} \quad  u\notin \leaf \quad \text{else} \quad \kk_u(\vt) = \vt_u \label{eq:nngp-recursion}
\end{align}
where $\vx\in\bmx$ and $\vt\in \vI$. 
The outputs of the computations are 
% \begin{align}
$
    \nn_{\gG}(\vx) = \nn{\rootg}(\vx) $ and $ \mathscr K_{\gG}(\vt) = \kk_{\rootg}(\vt).
    $
% \end{align}
Note that $\mathscr K_{\gG}$ is indeed the NNGP kernel \citep{neal1996priors, lee2018deep, matthews2018gaussian}. The NTK \citep{jacot2018neural, lee2019wide} can also be written recursively as 
% \begin{align}
%     \Theta_{\gG}(\vt) = \sum_{uv\in \gE}\partial_{e_{uv}} \mathscr K_{\gG}(\vt)\Big|_{e_{uv}=1} \,.   
% \end{align}
% Using the chain rule, the tangent kernel can be obtained
% recursively with $\Theta_{\gG} = \Theta_{\rootg}$ and 
\begin{align}
\label{eq:ntk-recursion}
    \Theta_{u}(\vt) = 
    \dot \phi_u^*\left(\fint_{v: uv\in\gE}  \kk_v(\vt)\right)\fint_{v:uv\in\gE}
    \left(\kk_v(\vt) +\Theta_{v}(\vt) \right)
    \quad \text{with}\quad \Theta_{\gG} = \Theta_{\rootg}. 
\end{align}
Here, $\Theta_u=0$ if  $u\in \leaf$ and $ \dot \phi_u^*$ is the derivative of $\phi_u^*$. See Table 1 in \citet{neuraltangents2020} for more details regarding the computations of NNGP kernels and NTKs. For formal proofs of the convergence of the NNGP kernels and NTKs, see  \citet{daniely2016toward, novak2019bayesian,yang2019scaling, arora2019exact}. 

\subsection{Three Examples: MLPs, S-CNNs and D-CNNs.}
To unpack the notation, we consider three concrete examples: an $L$-hidden layer MLP, a shallow convolutional network (S-CNN) that contains only one convolutional layer and a deep convolutional network (D-CNN) that contains $(1+L)$ convolutional layers. The architectures are 
\begin{align}\label{eq:mlp-model}
&\textbf{MLP:}
&&[\textit{Input}]\to [\textit{Dense-Act}]^{\otimes L}\to [\textit{Dense}]
\\
\label{eq:s-cnn-model}
&\textbf{S-CNN:}
&&[\textit{Input}]\to [\textit{Conv}(p)\textit{-Act}]\to[\textit{Flatten-Dense}]
\\
\label{eq:d-cnn-model}
&\textbf{D-CNN:}
&&[\textit{Input}]\to [\textit{Conv}(p)\textit{-Act}]\to[\textit{Conv}(k)\textit{-Act}]^{\otimes L}\to[\textit{Flatten-Dense-Act}]\to[\textit{Dense}]
\end{align}
where $p/k$ is the filter size of the first/hidden layers and ${\it Dense}/{\it Conv}/{\it Act}/{\it Flatten}$ means an dense / convolutional / activation / flattening layer. We choose the stride to be the same as the size of the filter for all convolutional layers and choose {\it flattening} as the readout strategy rather than pooling. See Fig.~\ref{fig:nn-demonstration} (a, b, c) for the DAGs associated to a 4-layer MLP, a (1 + 1)-layer CNN (with $p=k=d^{\frac 1 2}$) and a (1+3)-layer CNN (with $p=k=d^{\frac 1 4}$).
% and Sec.~\ref{sec:code} for sample codes. 
\begin{figure}[t]
\centering
\centering
    \begin{subfigure}[b]{0.1\textwidth}
     \includegraphics[width=\textwidth]{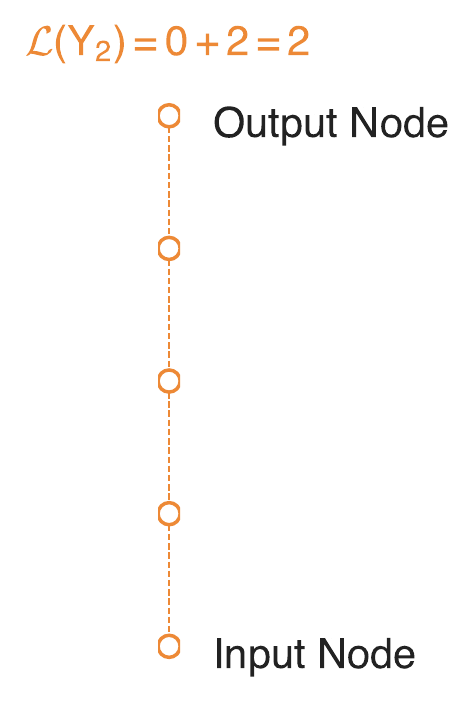}    \caption{$\text{MLP}$}
    \end{subfigure}
    \begin{subfigure}[b]{0.3\textwidth}
     \includegraphics[width=\textwidth]{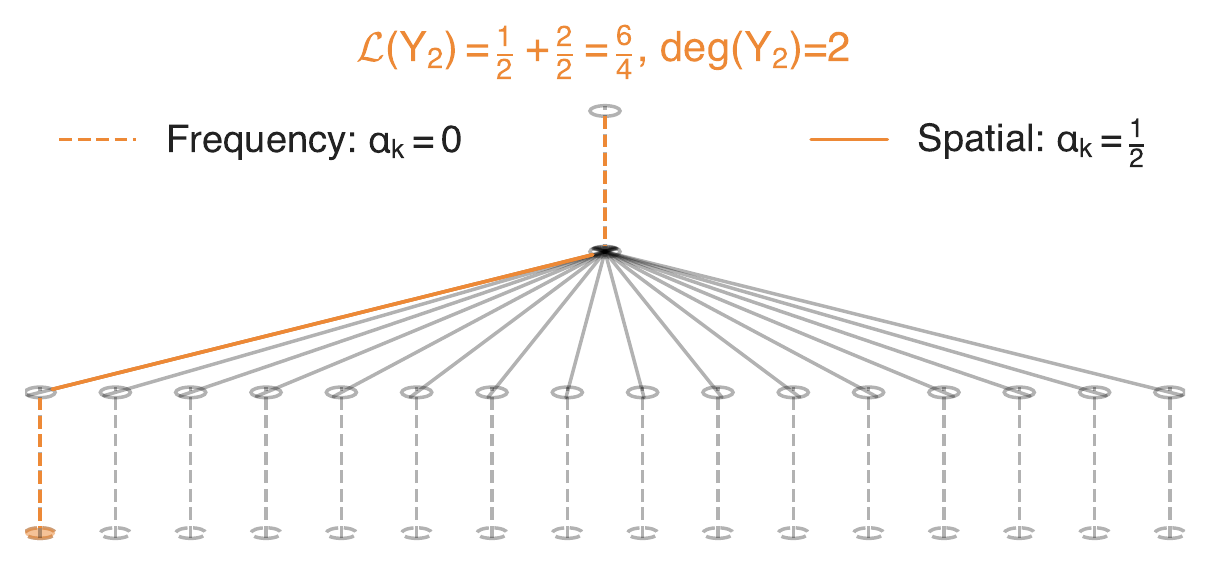}
         \subcaption{$\text{CNN}(p^2)^{\otimes{2}}$}
         \end{subfigure}
    \begin{subfigure}[b]{0.3\textwidth}
    \includegraphics[width=\textwidth]{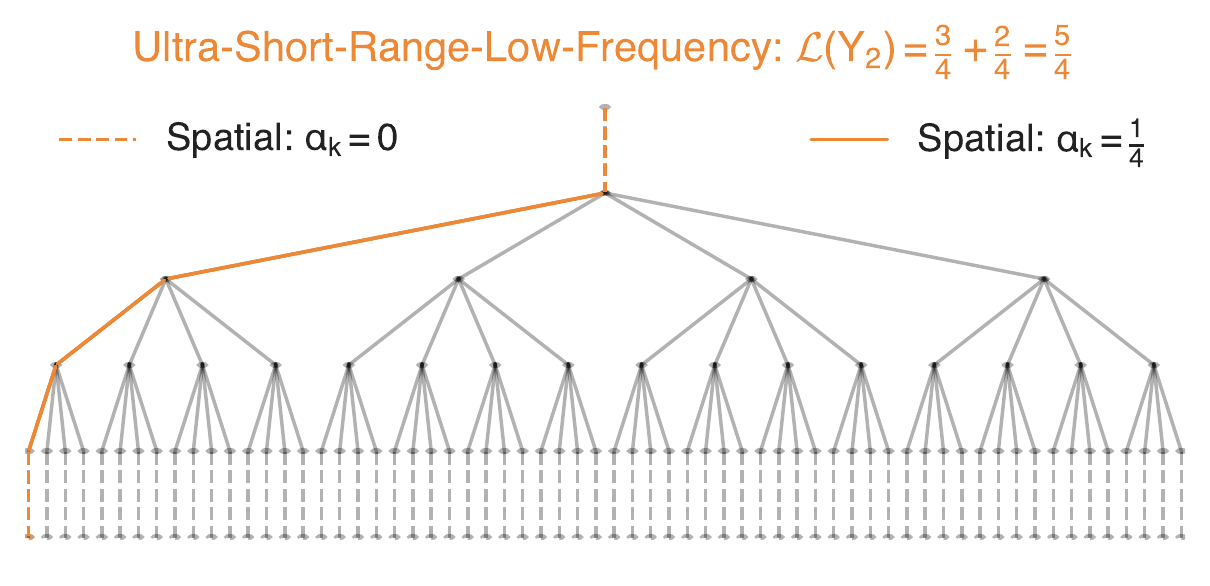}
        \caption{$\text{CNN}(p)^{\otimes{4}}$}
    \end{subfigure}
        \begin{subfigure}[b]{0.25\textwidth}
    \includegraphics[width=\textwidth]{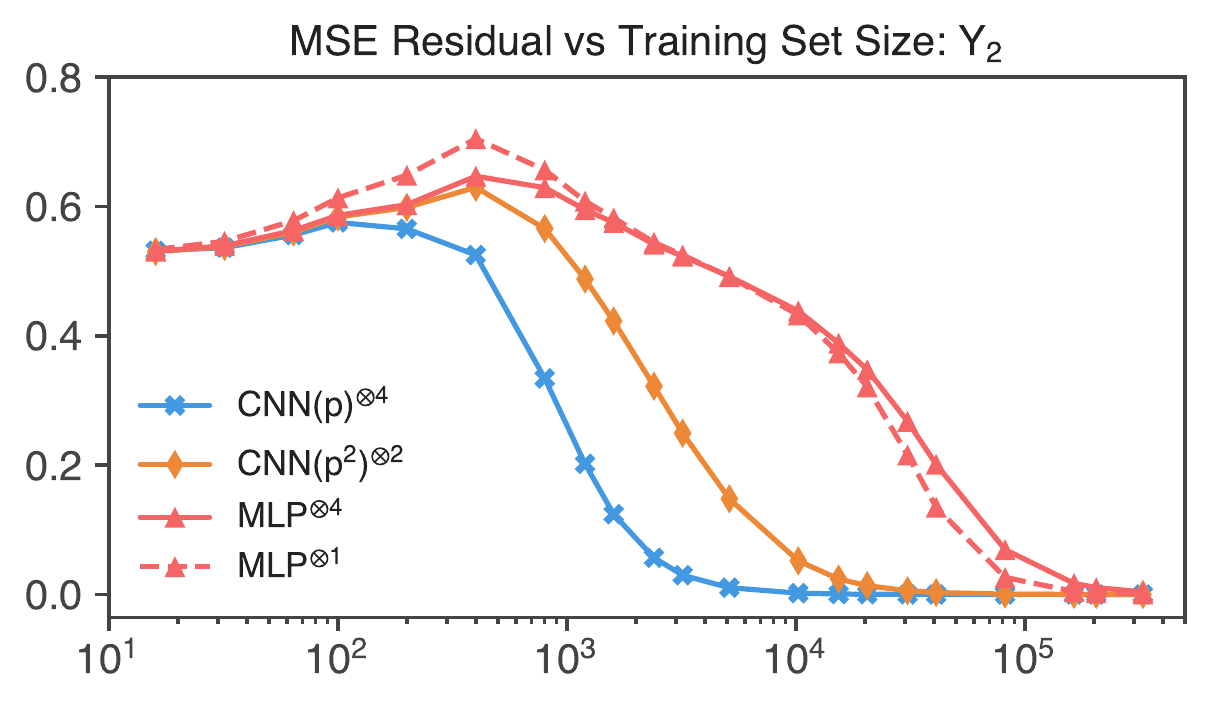}    \caption{MSE Residual}
    \end{subfigure}
    \caption{{\bf Architectures/DAGs. vs Eigenfunctions vs Learning Indices.}
    Left to right:  
     DAGs associated to (a) a four-layer MLP; (b) $\text{CNN}(p^2)^{\otimes{2}}$, a``D"-CNN that has two convolutional layers (c) $\text{CNN}(p)^{\otimes{4}}$, a ``HR"-CNN that has four convolutional layers; 
    and (d) MSE (Y-axis) vs training set size (X-axis) for ${\bf\color{plot_2} Y_2}$ obtained by NTK-regression for 4 architectures. Here ${\bf\color{plot_2} Y_2}$ is a linear combination of eigenfunctions of Short-Range-Low-Frequency interactions  ($\deg({\bf\color{plot_2} Y_2})=2$); see Sec.~\ref{Sec:learning, dags} for the expression. The DAGs are generated with $p=4$.
    In each DAG, the {\bf Dashed Lines} represent the edges with zero weights. The {\bf Solid Lines} have weights $0$, $\frac 12$ and $\frac 1 4$ in (a), (b) and (c), resp. 
    The {\bf\color{plot_2} colored path} represents the minimum spanning tree used to compute the spatial indices of ${\bf\color{plot_2} Y_2}$. Under architectures (a), (b) and (c), the spatial indices are $0$, $\frac 1 2$ and $\frac 3 4$, resp. Each input node represents an input patch of dimension $p^4=d^1$, $p^2=d^{\frac 1 2}$ and $p=d^{\frac 1 4}$ and the frequency indices are $2\times 1$, $2\times \frac 1 2$ and $2\times \frac1 4$ in (a), (b) and (c), resp.} 
    \label{fig:nn-demonstration}
\end{figure}
\paragraph{MLPs.} Let $\gG=(\gN, \gE)$ be a linked list with $(L+2)$ nodes, including the input/output nodes. Let $\gL_{uv}\in\mathbb R^{n_u\times n_v}$, where $n_{u/v}=\text{dim}(\Hilbert_{u/v})$ and the activations of the input/output nodes be the identity function. Then $\mathscr N_{\gG}$ represents a $L$-hidden-layer MLP. In addition, let $u$ and $v$ be the nodes in the $l$- and $(l-1)$-th layers and let $\gL_{uv}$ be initialized iid as 
\begin{align}
    \gL_{uv} =
    \frac{1} {\sqrt{n_v}} \left(\omega_{uv, ii'}\right)_{i\in [n_u], i'\in [n_v]} \equiv 
    \frac{1} {\sqrt{n_v}} \left(\omega_{ ii'}^{(l)}\right)_{i\in [n_u], i'\in [n_v]}
    , \quad \omega_{uv, ii'} \equiv  \omega_{ ii'}^{(l)}  \sim \gN(0, 1) \,. 
\end{align} 
Then the MLP can be written recursively as 
\begin{align}
    \mathscr N_u(\vx) = \phi_u(\gL_{uv}(\mathscr N_v(\vx))) = 
    \phi_u\left(\frac{1} {\sqrt{n_v}}\omega^{(l)}\mathscr N_v(\vx)\right)
\end{align}
Let $\theta=(\omega_{uv, ii'}:i\in [n_u], i'\in [n_v], uv\in\gE)$ denote the collection of all trainable parameters.  
Let $\vt_{\vx, \vx'}= \vx^T\vx'/{n_{u}}$ for $u\in\leaf$ and $n_v\to\infty$ for all hidden nodes, then the outputs of $\mathscr N_{\gG}(\bmx)$ converge weakly to the GP $\gG\gP(0, \kk_{\gG}(\vt_{\vx, \vx'})_{\vx, \vx'\in\bmx})$ and $ \Theta_{\gG}(\vt_{\vx, \vx'})_{\vx, \vx'\in\bmx}$ 
is the NTK in the sense  
\begin{align}
\label{eq:nngp-ntk-convergence}
\mathbb E \nn_{\gG}(\vx) \nn_{\gG}(\vx') 
\xrightarrow{\text{in prob.}} \kk_{\gG}(\vt_{\vx, \vx'})
\quad \text{and}\quad 
    \langle \nabla_{\theta} \nn_{\gG}(\vx), \nabla_{\theta} \nn_{\gG}(\vx') \rangle 
    \xrightarrow{\text{in prob.}}
    \Theta_{\gG}(\vt_{\vx, \vx'}). 
\end{align}
Indeed, note that $\deg(u)=1$ for all $u\notin\leaf$. \myeqref{eq:nngp-recursion} and \myeqref{eq:ntk-recursion} become 
\begin{align}
    \kk_u(\vt_{\vx, \vx'}) = \phi^*_u(\kk_v(\vt_{\vx, \vx'})) \quad 
    \textand  \quad 
    \Theta_u(\vt_{\vx, \vx'}) = \dot \phi^*_u(\kk_v(\vt_{\vx, \vx'})) (\kk_v(\vt_{\vx, \vx'}) + \Theta_v(\vt_{\vx, \vx'})) 
\end{align}
which are exactly the recursive formulas for the NNGP kernel and NTK for MLPs; see e.g. Sec.E in \citet{lee2019wide}.  
\paragraph{S-CNN.}
The input $\bmx= (\spbar)^{w}\subseteq \mathbb R^{d}$, where $p$ is the patch size, $w$ is the number of patches, $d=pw$ is the dimension of the inputs. Here, the inputs have been {\it pre-processed} by a patch extractor and then by a normalization operator. In words, the S-CNN has one convolutional layer with filter size $p$, followed by an activation function $\phi$ (e.g., Relu), and finally by a {\it flatten-dense} readout layer. 
Mathematically, by letting $n\in\mathbb N$ be the number of channels in the hidden layer, the output (i.e., logit) is given by  
\begin{align}
&\text{\bf Convolution + Activation:} \quad    &&z_{ij}(\vx) = \phi\left(p^{-\frac 1 2} \sum_{i'\in[p]} \omega^{(1)}_{j, i'}\vx_{i, i'} \right)  \quad \text{for} \quad i\in [w], j\in [n]
    \\
&\text{\bf Flatten + Dense:}\quad 
    % &&\mathscr N_{\text{S-CNN}}
    &&f(\vx) = (w n)^{-\frac 1 2}\sum_{i\in [w], j\in [n]}\omega^{(2)}_{ij} z_{ij}(\vx)\,, 
 \end{align}
where $\omega^{(1)}_{j, i'}$ and $\omega^{(2)}_{ij}$ are the parameters of the first and readout layers, resp.

We can associate a DAG $\gG=(\gN, \gE)$ to the above S-CNN. Let the input, hidden and output nodes be $\leaf= \{0\}\times [w]$, $\gN_1=\{1\}\times [w]$ and $\gN_2 = \{\rootg\}=\{(2)\}$, resp. and $\gN = \leaf\cup\gN_1\cup\gN_2$. Moreover, $uv\in\gE$ if $u=\rootg$ and $v\in\gN_1$ or 
$u=(1,i)\in\gN_1$ and $v=(0, i)\in\leaf$.
Let $\Hilbert_v = \mathbb R^{p}$ for $v\in\gN_0$, $\Hilbert_u = \mathbb R^{n}$ if $u\in\gN_1$ and $\Hilbert_{\rootg} = \mathbb R$. The associated linear operators are given by 
\begin{align}
    \gL_{uv} &= p^{-\frac 1 2 }\left (\omega^{(1)}_{j, i'}\right)_{j\in [n], i'\in[p]}\in\mathbb R^{n\times p} \quad &&\text{for }\quad  (u,v) \in \gN_1\times \leaf,  \quad uv\in\gE.  
    \\
    \gL_{\rootg v} &= (w n)^{-\frac 1 2}\left(\omega^{(2)}_{ij}\right)_{j\in [n]} \in \mathbb R^{n} \quad &&\text{if} \quad v = (1, i) \in\gN_1
\end{align}
Note that the weights are shared in the first layer but not in the readout layer (i.e., the network has no pooling layer). To compute the NNGP kernel and NTK, we initalize all parameters $\omega^{(1)}_{j, i'}$ and $\omega^{(2)}_{ij}$ with iid Gaussian $\gN(0, 1)$. Letting $n\to\infty$ and denoting $\vt_v =\vx_v^T\vx_v'/p$ and $\bm\vt = (\vt_v)_{v\in\leaf}$, we have 
\begin{align}
    \kk_{\gG}(\bm\vt) = \fint_{v\in\leaf} \phi^{*}(\vt_v) 
    % = \fint_{v\in\leaf} \phi^{*}(\vt_v)
    \quad \text{and} \quad \Theta_{\gG}(\bm\vt) = 
     \fint_{v\in\leaf} \phi^{*}(\vt_v) + \dot \phi^*(\vt_v) 
\end{align}

\paragraph{D-CNN.}
The input space is $\bmx = (\spbar)^{w \times k^L}\subseteq \mathbb R^{w \times k^L \times p}$, where $p$ is the patch size of the input convolutional layer, $k$ is the filter size in {\it hidden} convolutional layers, $L$ is the number of {\it hidden} convolutional layers and $w$ is the spatial dimension of the layer before flattening. The total dimension of the input is $d=p\cdot k^L\cdot w$, and the number of input nodes is $|\leaf| = k^L\cdot w$. Since the stride is equal to the filter size for all convolutional layers, the {\it spatial} dimension is reduced by a factor of $p$ in the first layer, a factor of $k$ by each hidden convolutional layer, and is reduced to 1 by the {\it Flatten-Dense-Act} layer. Similar to S-CNNs, one can associate a DAG to a D-CNN. Briefly, the input layer has $k^L\times w$ nodes and is reduced by a factor of $k$ by each convolutional layer. The layer before {\it flattening}, the second last layer and the output layer have $w$,  $1$, and  $1$ nodes, resp. More precisely, we can identify the nodes using tuples. 
\begin{itemize}
    \item For $j=0$, i.e., the input layer, $\gN_0=\{0\} \times [w]\times [k]^L$. 
    \item For $1\leq j\leq L+1$, $\gN_j=\{j\}\times [w]\times [k]^{L-j+1}$. 
\item  The second last and output layers are $\gN_{L+2}=\{(L+2)\}$ and $\gN_{L+3}=\{(L+3)\}\equiv \{\rootg\}$.
\end{itemize}
Note that the first index of a tuple specifies the layer index of the node, i.e. if $u=(j, \dots)$ then $u\in\gN_j$. The remaining indices specify its spatial location in that layer. To define the edges $\gE$, we only need to specify the parents of the nodes. 
\begin{itemize}
    \item For $j=0$, the parent of $v=(0, k_0, k_1, \dots, k_{L})\in\gN_0$ is $u=(1, k_0, k_1, \dots, k_{L})\in\gN_1$. 
    \item For $1\leq j\leq L$, the parent of $v=(j, k_0, k_1, \dots, k_{L-j+1})\in \gN_j$ is $u=(j+1, k_0, k_1, \dots, k_{L-j})\in \gN_{j+1}$. 
    \item For $j=L+1$, the parent of $v=(L+1, k_0)\in \gN_{L+1}$ is $u=(L+2)$, whose parent is $\rootg=(L+3)$. 
\end{itemize}
We write $\gE =\cup_{1\leq j\leq L+3}\gE_j$, where $\gE_j$ is the collection of edges between the $(j-1)$-th and $j$-th layer. Next, we define the neural network associated to this D-CNN. For $u\in\gN_j$, $\mathbb H_u = \mathbb R^{n_j}$, where $n_0=p$ is the size of the patch, $n_j$ is the number of channels in the $j$-th hidden layer for $1\leq j\leq L+1$, and $n_{L+2}$ is the number of features in the second last layer and $n_{L+3}=1$. Thus $n_u = n_j$ if $u\in \mathcal N_j$. For $\vx \in\bm\X$, if $u=(0, k_0, k_1, \dots, k_{L})\in\gN_0$, let 
\begin{align}
    \mathscr N_u(\vx)  \equiv \vx_{u}\equiv \vx_{k_0, k_1, \dots, k_{L}} \in \spbar
\end{align}
and otherwise, $u\in \gE_j$ for some $j\geq 1$ and let  
\begin{align}
    \mathscr N_{u}(\vx)_i = \phi_u\left( \left(\deg(u)n_v\right)^{-\frac 1 2}\sum_{v: uv\in \gE_j} \sum_{i'\in [n_v]} \omega_{uv, ii'} \mathscr N_v(\vx)_{i'}\right)\, \quad \text{for} \quad i\in[n_u].  
\end{align}
Here, for $u\in \gN_j$, the degree $\deg(u)=p, k, w $ and $1$ if $j = 1$, $ 2\leq j \leq L+1 $, $ j=L+2 $ and $ j=L+3$ resp., and the parameters are usually initialized with iid $\omega_{uv, ii'} \sim \mathcal N(0, 1)$ and the shape of $\omega_{uv}$ is $(\deg(u), n_v, n_u)$. 
It is more appropriate to call the current network a locally-connected network, or a ``convolutional network" {\it without} weight-sharing, as the parameters $\omega_{uv}$ depend on the node $u$ and are not shared within the same layer. To make it a true convolutional network, we enforce weight-sharing by setting $\omega_{uv, ii'}= \omega_{u'v, ii'}\equiv \omega^{(j)}_{v, ii'}$ if $u, u'\in\gN_j$ for some $j$. 

Denoting $\vt_u =\vx_u^T\vx_u'/p$ for $u\in\gN_0$ and $\bm\vt = (\vt_u)_{u\in\leaf}$. As $\min\{n_j\}_{1\leq j\leq L+2}\to\infty$, we have \myeqref{eq:nngp-ntk-convergence} and 
the NNGP kernel and the NTK associated to this network are given by \myeqref{eq:nngp-recursion} and \myeqref{eq:ntk-recursion}. For formal proofs, see \citet{daniely2016toward, novak2018bayesian, yang2019scaling}. 
This is true for both convolutional networks and locally-connected networks, as long as the architecture contains no pooling \citep{novak2018bayesian}. For a convolutional network with a global average pooling layer (GAP), the kernel computations need to be modified to capture the translation invariance from GAP. See Sec.~\ref{Sec:gap} for more details. 

\section{Main Results}\label{sec:main results}
The goal is to obtain a precise charaterization of the relation between the eigenstructures of $\kk$ / $\infntk$ and the DAG associated to the network's architectures in the large input dimension setting. As such we consider a sequence of graphs $\bm\gG = \left(\Gd\right)_{d\in\mathbb N}$, where $\Gd=(\Nd, \Ed)$. We associate a {\it finite} set of non-negative numbers $\shapep $ to $\bm \gG$, which is called the shape parameters of $\bm \gG$,
\begin{align}
    0\in \shapep \subseteq [0, 1] \quad \text{and} \quad |\shapep|< \infty .
\end{align}
For the rest of the paper, we will use the following notations. For $A, B: \mathbb N\to \mathbb R^{+} $, 
\begin{align}
    B(d) \gtrsim A(d) \iff A(d) \lesssim B(d) \iff	
   \exists c, \,d_0 > 0 \quad s.t. \quad B(d) \geq c A(d) >0 \quad \text{for all}\quad  d > d_0\,
\end{align}
and 
\begin{align}
    B(d) \sim A(d) \iff	
    B(d) \gtrsim A(d) \quad \text{and} \quad A(d) \gtrsim B(d)\, .
\end{align}
\subsection{Assumptions}
We need several technical assumptions on $\bm \gG$ regarding the asymptotic shapes of $\Gd$, which are summarized as {\bf Assumption-$\bm\gG$} below. 

{\bf Assumption-$\bm\gG$. } Let $\bm \gG = (\Gd)_{d}$. There are absolute constants $c, C >0$ and $d_0>0$ such that the followings hold for $d\geq d_0$. 
\begin{enumerate}[label=(\alph*.)]
\item For each {\it non-input} node $u\in\Nd$, there is $\alpha_u\in \shapep$ such that 
    \begin{align}
        \displaystyle 
       c d^{\alpha_u}  \leq \deg(u) \leq C d^{\alpha_u} \,. 
    \end{align}
For each edge $uv\in\Ed$, its weight is defined to be  $\pi_{uv}\equiv\alpha_u$.
\item For each {\it input} node $v$, there are $d_v\in\mathbb N$ and $0<\alpha_v\in\shapep$ such that 
\begin{align}
        \displaystyle 
       c d^{\alpha_v}  \leq d_v\leq C d^{\alpha_v}\quad \text{and}\quad \sum_{v\in\Ndzero} d_v =d. 
    \end{align}
\item Let $\hidden \equiv \{u: \exists v\in \gN_0^{(d)} \, \text{ s.t.}\,  uv\in \Ed \}$ be the collection of nodes in the {\it first} hidden layer. We assume that for every $u\in \hidden$, $\alpha_u=0$ and all children of $u$ are input nodes. 
\item Every node $v\in \Nd$ has at most $C$ many parents, i.e. 
$
    \displaystyle  |\{u: uv\in \Ed\}| \leq C. 
$
Moreover, the number of {\it layers} is uniformly bounded, namely, for any node $u$, any path from $u$ to $\rootg$ contains at most $C$ edges. 
\end{enumerate}
The first two assumptions also help to create spectral gaps between eigenspaces. When $d$ is not large, the ``finite-width" effect is no longer negligible, and we expect that the spectra decay more smoothly. Assumption (c.) says there are no (skip) connections from the input layer to other layers except to the first hidden layer, which is often the case in practice.

We say $\phi^*$ is semi-admissible if, for all $r\geq 1$, the $r$-th derivative of $\phi^*$ at zero is non-vanishing, i.e., 
    $
        {\phi^*}^{(r)}(0) > 0. 
    $
If, in addition, $\phi^*(0)=0$ (i.e., the activation is centered), then we say $\phi$ is admissible. An activation $\phi$ is (semi-)admissible if $\phi^*$ is (semi-)admissible. 
Note that if $\phi^*$ is (semi-)admissible, then $\dot \phi^*$ is semi-admissible. 
Finally, we say $\phi$ is poly-admissible if there is a $J$ sufficiently large such that
$ {\phi^*}^{(r)}(0) > 0$
for all $1\leq r\leq J$ and ${\phi^*}^{(r)}(0)=0$
otherwise. The exact value of $J$ will depend on the problem of interest. 
\newline 
{\bf Assumption-$\phi$.} We make the following assumptions on the activations. 
    (a.) If $\inputnodes$, $\phi_u$ is the identity function. 
    (b.) If $u\notin \Ndzero\cup\{\rootg\}$, $\phi_u$ is admissible. 
    (c.) If $u=\rootg$, $\phi_u$ semi-admissible. 

This assumption is stronger than what are needed in most of the cases. E.g., if we are only interested in learning polynomials of degree at most $l$, it suffices to assume ${\phi^*}^{(r)}(0)>0$ for $1\leq r \leq l$. 
\subsection{Spatial, Frequency and Learning Indices}
We introduce the key concept which defines the {\it spatial distance} among nodes. It is the length of the minimum spanning tree (MST) of the nodes.   
\begin{definition}[Spatial Index of Nodes]
Let $\mathscr n \subseteq \Nd$. The spatial index of $\mathscr n$ is defined to be 
\begin{align}
\displaystyle
    \spatial(\mathscr n) \equiv \min_{{\mathscr n \subseteq \gT \leq \Gd }} \,\, \sum_{uv\in \gE(\gT)} \pi_{uv} 
    =
    \min_{{\mathscr n \subseteq \gT \leq \Gd }} \,\, \sum_{uv\in \gE(\gT)} 
    \alpha_u 
\end{align}
where $\mathscr n \subseteq \gT \leq \Gd$ means $\gT$ is a sub-graph containing $\mathscr n$.
%and $\deg(u; \gT)$ is the degree of $u$ in $\gT$. 
By default, $\spatial(\mathscr n) = 0$ if $\mathscr n$ contains only one or zero node. 
\end{definition}

% Let $\Xd= \prod_{v\in \Ndzero} \sbar_{d_v-1}$, $I^{\Gd}:= I^{\Ndzero}$ and $p:I^{\Ndzero} \to I $ be a {\it monomial} with
Next, we define the spatial/frequency/learning indices for a multi-index $\vr\in\bm I=  I^{|\Ndzero|}$. 
Let $\vt^\vr: \bm I \to I $ be a monomial, where $\vr: \Ndzero\to \mathbb N^{|\Ndzero|}.$ We use $\nnode(\vr) = \{v\in\Ndzero: r_v\neq 0\}$ to denote the support of $\vr$ and $\nnode(\vr;\rootg) = \nnode(\vr)\cup \{\rootg\}$. 

\begin{definition}[Spatial, Frequency and Learning Indices of $\vr$]
We say $\vr\in\mathbb N^{|\Ndzero|}$ is $\Gd$-learnable, or learnable for short, if there is a common ancestor node $u$ of $\nnode(\vr)$ such that $\phi_u$ is semi-admissible. We use $\mathscr{A} (\Gd) \equiv \{\vr\in\mathbb N^\Ndzero: \vr \text{ is learnable} \}$.  
For $\vr\in \mathscr A(\Gd)$, the spatial index, frequency index and the learning index 
%associated to the sequence $\bm\gG = (\Gd: d\in\mathbb N)$ 
are defined to be, 
\begin{align}
    \spatial(\vr) \equiv \spatial(\mathscr n (\vr; \rootg)) , \quad
\frequency (\vr) \equiv \sum_{v\in\Ndzero} \vr_v \alpha_v
 \quad 
 \text{and}\quad 
    \learning (\vr) \equiv \spatial(\vr) + \frequency(\vr) \, , 
\end{align}
resp. If $\vr\notin \Ad(\Gd)$, we set $\spatial(\vr)=\frequency(\vr)=\learning(\vr)=+\infty$. Let $\learning({\Gd})$ denote the sequence of learning indices in non-descending order, i.e.  
\begin{align}
    \learning ({\Gd}) \equiv \left( \learning (\vr):\vr\in\mathscr A(\Gd) \right) \equiv (\dots \leq r_j\leq r_{j+1}\leq \dots ) 
\end{align}
\end{definition}
Finally, we specify the eigenfunctions of the kernels and the indices associated them. 
For each $u\in\Ndzero$, let $\{\barY_{r, l}\}_{l\in [N(d_u, r)], r\in\mathbb N}$ be the family of normalized spherical harmonics in $\overline\sphere_{d_u-1}$, where $N(d_u, r)$ is the number of degree $r$ spherical harmonics in $\sphere_{d_u-1}$; see Sec.~\ref{sec:spherical harmonics} for more details about spherical harmonics. Define 
\begin{align}
        \bm\barY_{\vr, \vl}(\bm\xi) = \prod_{\inputnodes}\barY_{\vr_u, \vl_u}(\bm\xi_u), \quad \vl =(\vl_u)_{\inputnodes} \in [\vN(\vd, \vr)] \equiv \prod_{\inputnodes}[N(d_u, \vr_u)] \, 
\end{align}
for $\bm\xi = (\bm\xi_u)\in\bmx$, where $\vd=(d_u)_{\inputnodes}$ and, for consistency, $\vd_u = d_u$.  The spatial, frequency and learning indices of $\bm\barY_{\vr, \vl}$ are
$\spatial(\bm\barY_{\vr, \vl})\equiv \spatial(\vr)$, 
$\frequency(\bm\barY_{\vr, \vl})\equiv \frequency(\vr)$ and
$\learning(\bm\barY_{\vr, \vl})\equiv \learning(\vr)$, resp. 
\subsection{Eigenspace Restructuring}
The following is our main theorem. It describes an analytical relation between the architecture of a network and the eigenstructure of its inducing kernels in high dimensions.  
\begin{theorem}[\bf {Eigenspace Restructuring}]
\label{thm:main}
Assume 
% $\bm\gG = (\Gd)$ and $\{\phi_u: u\in \Nd\}$ satisfy 
{\bf 
Assumption-$\bm\gG$} and {\bf Assumption-}$\phi$. 
We have the following eigen-decomposition for $\gK = \kk_{\Gd}$ or $\Theta_{\Gd}$. For $ \bm \xi, \bm\eta \in \bmx $
\begin{align}
    \gK(\bm\xi,\bm \eta) = \sum_{\vr\in \mathbb N^{|\Ndzero|}}
    \lambda_\gK(\vr) 
    \sum_{\vl\in \vN(\vd, \vr)} \bm\barY_{\vr, \vl}(\bm\xi) \bm\barY_{\vr, \vl}(\bm\eta),  
\end{align}
where 
\begin{align}\label{eq:eigen-decay}
    \lambda_\gK(\vr)\sim d^{-\learning(\vr)} \quad  \text{if} \,\, \vr\neq \bm 0\, \quad \text{and} \quad |\vr| \lesssim 1. 
\end{align}
\end{theorem}
The theorem says that the eigenfunctions of $\gK$ are  $\{\bm\barY_{\vr, \vl}\}$, whose eigenvalues are of order $\{d^{-\learning(\vr)}\}$, which is determined by the learning index. When the DAGs are more {\it regular}, e.g. the network architectures are the CNNs, MLPs in Sec.~\ref{sec:neural-computation}, then one can also prove that the dimension of eigenspaces with eigenvalues $\sim d^{-\learning(\vr)}$ is $\sim d^{\learning(\vr)}$ (see Lemma \ref{lemma:dimension}.)

Coupling with the observation in \myeqref{eq:linear-eigen-mode}, Theorem \ref{thm:main} implies that within $t\sim d^r$ amount of time for gradient flow, when $d$ is sufficiently large, only the eigenfunctions $\bm\barY_{\vr, \vl}$ with $\learning(\vr)\leq r$ can be learned. More precisely, let $\bm\sigma$ be the uniform (product) probability measure on $\bmx$ and denote $L^p(\bmx) \equiv L^p(\bmx, \bm\sigma)$. For $r\notin \learning(\Gd)$, define the projection operator to be 
\begin{align*}
    % \bm R_{X}(f)(x) = \gK(x, X) \gK(X, X)^{-1}f(X)
    % \,\, \textand
    % \,\, 
    \bm P_{<r}(f) =
    \sum_{\vr: \learning(\vr)< r}\sum_{\vl\in \vN(\vd, \vr)}\langle f, \bm\barY_{\vr, \vl}\rangle_{L^2(\bmx)} \bm\barY_{\vr, \vl} \, , 
\end{align*}
and similarly for $ \bm P_{>r}$. Let $\bm \gF:L^2(\bmx)\to L^2(\bmx)$ be the solution operator associated to the kernel gradient descent
\begin{align}
    &\dot h = -\gK (h -f)  \quad \text{with initial value} \quad 
h_{t=0}=\bm 0
\end{align}
i.e. 
$h_t \equiv \bm \gF_t(f) \equiv ( { \text{\bf Id}} - e^{-\gK t}) f$, where  
\begin{align}
   ( { \text{\bf Id}} - e^{-\gK t}) f (\vx) 
    \equiv  \int_{\bar \vx 
    \in\gX} (1 - e^{-\gK(\vx, \bar \vx)t})  f (\bar \vx) d\bar \vx  
\end{align}
The following theorem says that when $t\sim d^r$,  $\bm \gF_t\approx \bm P_{<r}$. 
\begin{theorem}\label{thm:inf-data}
Assume 
% $\bm\gG = (\Gd)$ and $\{\phi_u: u\in \Nd\}$ satisfy 
{\bf 
Assumption-$\bm\gG$} and {\bf Assumption-}$\phi$ and $\gK = \kk_{\Gd}$ or $\Theta_{\Gd}$. Let $r\notin \learning(\Gd)$ and $t\sim d^r$. Then for $0<\epsilon< \inf\{|r-\bar r|: \bar r \in \learning(\Gd)\}$ and $f\in L^2(\bmx)$ with $\mathbb E_{\bm\sigma} f = 0$, we have 
\begin{align}\label{eq:gradient-flow-learnability}
    \|\bm \gF_t ({\bm P}_{<r} f )- {\bm P}_{<r} f \|^2_2 \lesssim e^{-d^{\epsilon}}\|{\bm P}_{<r} f\|_2^2
    \quad \textand 
    \quad \|\bm \gF_t ({\bm P}_{>r} f )- {\bm P}_{>r} f\|^2_2 \gtrsim e^{-d^{-\epsilon}}\|{\bm P}_{>r} f\|_2^2 \,.
\end{align}
\end{theorem}
In words, in the infinite-training-data regime, within $t\sim d^r$ amount of time only the eigenfunctions $\bm \barY_{\vr, \vl}$ with $\learning(\vr)<r$ are learnable. The proof follows directly from the arguments in Sec.~\ref{sec:linear}: an eigenfucntion $
\bm Y_{\vr, \vl}$ is learnable if $e^{-t \lambda_\K(\vr)} \sim e^{- d^{r-\learning(\vr)}}$ is sufficiently small.   
Therefore, the learning index is the correct complexity measure that describes the compute efficiency mentioned in Sec.\ref{Sec: toy example}.

\subsection{Generalization: CNN without pooling}
Our next theorem concerns the finite-training-data regime, in which 
we will leverage a deep analytical result from \citet{mei2021generalization} (Sec. 3 Theorem 4). 
We restrict the architectures to D-CNNs \myeqref{eq:d-cnn-model} which contain S-CNNs and MLPs as special cases. This restriction could be relaxed by proving a dimension counting lemma similar to Lemma \ref{lemma:dimension} in the appendix.

For $X\subseteq \bmx$, define the regressor to be 
\begin{align*}
    \bm R_{X}(f)(x) = \gK(x, X) \gK(X, X)^{-1}f(X). 
    % \,\, \textand
    % \,\, \bm P_{>r}(f) =
    % \sum_{\vr: \learning(\vr)> r}\sum_{\vl\in \vN(\vd, \vr)}\langle f, \bm\barY_{\vr, \vl}\rangle_{L^2(\bmx)} \bm\barY_{\vr, \vl} \, . 
\end{align*}
Let $\alpha_p, \alpha_k, \alpha_w\in [0, 1]$ and $\alpha_p\neq 0$. 
In what follows, we assume $p\sim d^{\alpha_p}$, $k\sim d^{\alpha_k}$ and $w\sim d^{\alpha_w}$ as $d\to\infty$. Thus the shape parameters are $\Lambda_{\bm \gG} = \{0, \alpha_p, \alpha_k, \alpha_w\}$.
\begin{theorem}\label{thm:generalization}
Let $\bm \gG = \{\Gd\}_d$, where each $\Gd$ is a DAG associated to a D-CNN \myeqref{eq:d-cnn-model}. Let $r\notin \learning(\Gd)$ be fixed. 
Assume {\bf Assumption-}$\phi$ and $\alpha_p>0$. 
Let $f\in L^2(\bmx)$ with $\mathbb E_{\bm\sigma} f = 0$. Then for $\epsilon >0$, 
\begin{align}\label{eq:generalization-finite-data}
\left|    \left\|\bm R_X(f) - f\right\|_{L^2(\bmx)}^2 - 
    \left\|\bm P_{>r}(f)\right\|_{L^2(\bmx)}^2 \right| 
    = c_{d, \epsilon} \|f\|_{{L^{2+\epsilon}(\bmx)}}^2, \quad 
\end{align}
where $c_{d, \epsilon} \to 0$ in probability as $d\to\infty$ over $X\sim\bm\sigma^{[d^r]}$. 
\end{theorem} 
In words, with $[d^{r}]$ many training examples where $r\notin\learning(\Gd)$, the NNGP kernel and the NTK are able to learn all $\bm\barY_{\vr, \vl}$ with $\learning(\vr) < r$ but not any eigenfunctions with $\learning(\vr) > r$.  Therefore, the learning index is the correct complexity measure that describes the data efficiency mentioned in Sec.\ref{Sec: toy example}.  
\subsection{Generalization: CNN with global average pooling}
We discuss the benefits of pooling in this section. Let a CNN with and without a global average pooling (GAP) be defined as 
\begin{align}\label{eq:model-cnn-gap}
&\textbf{CNN+GAP}&&[\textit{Input}]\to [\textit{Conv}(p)\textit{-Act}]\to[\textit{Conv}(k)\textit{-Act}]^{\otimes L}\to[\textit{GAP}]\to[\textit{Dense}]\\
&\textbf{CNN+Flatten}&&[\textit{Input}]\to [\textit{Conv}(p)\textit{-Act}]\to[\textit{Conv}(k)\textit{-Act}]^{\otimes L}\to[\textit{Flatten}]\to[\textit{Dense}]
\end{align}
resp. Note that there isn't any activation after the GAP/Flatten layer. The DAGs associated to {\it CNN+GAP} and {\it CNN+Flatten} are identical, and the associated kernel and network computations in each layer are also identical but the last layer. 
Weight-sharing together with GAP induces translation invariance symmetry to the function spaces given by the networks, and one needs to modify the kernel computation in the last layer to take into account this symmetry group; see Sec.~\ref{Sec:gap}.

Our last theorem states that the GAP improves the data efficiency by a factor of $w$, the window size of the pooling, but does not improve the training efficiency in the infinite-training-data regime. The former is due to the dimension reduction effect of GAP in the eigenspaces, and the latter is because the GAP layer does not modify the eigenvalues of the associated eigenspaces.  

Let $\K_\sym = \kk_\sym$ or $\Theta_\sym$ be the NNGP kernel or NTK associated to \myeqref{eq:model-cnn-gap},  $L^p_\sym(\X)\leq L^p(\X)$ be the subspace of ``translation-invariant" functions, whose co-dimension is $w$, and $\{\barY^\sym_{\vr, \vl}\}$ be the basis elements in $L^2_\sym(\X)$. Let $\bm \gF^\sym$,
$\bm P^{\sym}$ and $\bm R^{\sym}$ be the solution operator, projection operator and regressor associated to $\K_\sym$, resp.
See Sec.~\ref{Sec:gap} for the definitions. 
\begin{theorem}[Informal]\label{thm:cnn+gap}
For the architectures defined as in by \myeqref{eq:model-cnn-gap}, we have 
\begin{enumerate} [label=(\alph*)]
\item Similar to Theorem \ref{thm:main}, we have
\begin{align}
    \gK_{\sym}(\bm\xi,\bm \eta) = \sum_{\vr\in \mathbb N^{|\Ndzero|}}
    \lambda_{\gK_\sym}(\vr) 
    \sum_{\vl\in \vN(\vd, \vr)} \bm\barY^\sym_{\vr, \vl}(\bm\xi) \bm\barY^\sym_{\vr, \vl}(\bm\eta),  
\end{align}
with 
\begin{align}\label{eq:eigen-decay}
    \lambda_{\gK_{\sym}}(\vr)\sim d^{-\learning(\vr)} \quad  \text{if} \,\, \vr\neq \bm 0\, \quad \text{and} \quad |\vr| \lesssim 1. 
\end{align}
    \item Theorem \ref{thm:inf-data} holds with $L^2(\X)$,  $\bm \gF$ and 
$\bm P$ replaced by $L^2_\sym(\X)$, $\bm \gF^\sym$ and 
$\bm P^{\sym}$, resp. 
\item \myeqref{eq:generalization-finite-data} holds with $L^p(\X)$,  $\bm R$ and 
$\bm P$ replaced by $L^p_\sym(\X)$, $\bm R^\sym$ and 
$\bm P^{\sym}$, resp and with $X\sim \bm\sigma^{[d^{r-\alpha_w}]}$ under the assumptions that all activations in the hidden layers are poly-admissible. 
\end{enumerate}
\end{theorem}
Thus, the major benefit of GAP is the $w=d^{\alpha_w}$ improvement in terms of data efficiency. 
\subsection{Remarks}
Several remarks are in order. 
\begin{enumerate} %[leftmargin=0.5cm]
    \item In terms of order of learning, our results say that  (infinite-width) neural networks progressively learn more complex functions, where complexity is defined to be the learning index $\learning(\vr)$. This (in the infinite-width setting) rigorously justifies several empirical observations that neural networks biases towards {\it simple} functions \citep{rahaman2019spectral, DBLP:journals/corr/abs-1905-11604, vall2019deep}. Regardless of architectures (MLPs vs CNNs), linear functions always have the smallest learning index\footnote{Ignoring the constant functions.} $\learning(\vr)=1$ and are always learned first, which was first proved in \citet{hu2020surprising} for MLPs. After learning the linear functions, the learning dynamics of MLPs and CNNs diverge. MLPs start to learn quadratic functions ($\learning(\vr)=2$, require $d^{2+}$ many training points), then cubic functions ($\learning(\vr)=3$) and so on. While for CNNs, $\learning(\vr)$ can be fractional numbers (e.g. $\learning(\vr)=5/4, 6/4$) and it is possible for the network to learn higher order functions (with shorter range of interactions) before lower order functions (with longer range of interactions). See Sec.\ref{Sec:learning, dags} and Sec.~\ref{sec:experiment} for more details.    
    \item Our results could be translated into the overparameterized network setting by coupling an NTK-style convergent argument (e.g., \citet{du2018provably, allen2018convergence-fc, zou2018stochastic}) with a standard Gronwall-type argument \citep{lee2019wide, chizat2019lazy}. However, the proof requires the learning rate of the gradient descent to be small and the number of channels be sufficiently large (polynomially in $d$ \citep{Huang2020OnTN}) so that the network remains close to the linearized regime \citep{lee2019wide, chizat2019lazy}, which is often unrealistic and undesirable. It remains open to prove similar results for networks with realistic widths.  
    \item Theorem \ref{thm:cnn+gap} states that {\it CNN+GAP} is better than {\it CNN+Flatten} in terms of data efficiency but not compute efficiency. In particular, when the dataset size is sufficiently large, {\it CNN+GAP} and {\it CNN+Flatten} perform equally well; see Sec.~\ref{sec:exp2} for empirical verification. Therefore, in the large (infinite) data set regime, {\it CNN+Flatten} might be preferable since the associated function class is larger.  

    \item 
Finally, to help understand the main results, we provide a (very) conceptual explanation of the eigenspace restructuring theorem and the follow-up theorems by drawing an analogy between eigenstructures of kernels and the internal structures of high-rise buildings. In the baseline case, a ``MLP building" is a building in which each floor is a single giant room (``no spatial structure'') that contains only one stair to the floor above. Moreover, the stairs on any two consecutive floors are always located at the two different ends of the building (bad architecture design!). As such, within a fixed amount of time, one can only visit (``learn'') lower floors (``eigenspaces''). In contrast, in a ``D-CNN building'', each floor is restructured to have more rooms (``subspaces") and more {\it stairs} that are evenly distributed. Technically speaking, with the same budget of time, one can only visit one of (a) all rooms on lower floors (``long-range-low-frequency interactions'',) (b) rooms on higher floors which are horizontally close to the entrance (``short-range-high-frequency interactions",) and (c) any one interpolation between them (``median-range-median-frequency interactions''.) Indeed, one can visit all of (a), (b), and (c) in a ``D-CNN building'' because the ``horizontal distance'' and the ``height'' are in log scales. Finally, a ``HR-CNN building`` is even more finely restructured to have more rooms and stairs. 
\end{enumerate}

\section{Interpretation of the main results}\label{sec: interpretation}
% \begin{figure}
%     \centering
%         \includegraphics[width=0.09\textwidth]{iclr2022/figures/nov-19/MLP.pdf}
%     \includegraphics[width=0.28\textwidth]{iclr2022/figures/nov-19/S-CNN.pdf}
%     \includegraphics[width=0.28\textwidth]{iclr2022/figures/nov-19/D-CNN.pdf}
%     \includegraphics[width=0.28\textwidth]{iclr2022/figures/nov-19/HR-CNN.pdf}
%     \caption{The DAGs for MLP, S-CNN, D-CNN \& HR-CNN}
%     \label{fig:dags vs architecture}
% \end{figure}
We say $r$ is the budget index if (1) ({\bf Finite Training Set}) the training set size $m\sim d^{r}$, or (2) ({\bf Finite Compute Time}) the training set $X=\bmx$ and the total number of training steps/time $t\sim d^{r}$. 
% In what follows, we investigate the set of learnable eigenfunctions within the budget index $r$ for four types of architectures. 
% For the ease of presentation, we assume every node $u$ in the {\it first} hidden layer has degree 1. 
\paragraph{MLPs and {\color{color5}LRLF Interactions} Fig.~\ref{fig:high-level-picture} (a).}
Each $\Gd$ is a linked list. Let $v$ be the input node. Clearly, we have $d_v= d$, i.e., $\alpha_v=1$ and $\alpha_u=0$ (since $\deg(u)=1$) for all other nodes $u$. As such 
\begin{align}
\Ad(\Gd) = \mathbb N\backslash\{0\}\, , \,\, 
    \frequency(\vr) = |\vr|, \,\, \spatial(\vr)=0 ,\,\,  \learning(\vr) = |\vr| \,\,
    \text{and} \,\, \learning({\Gd}) = (|\vr|: \vr\in\mathbb N\backslash\{0\}). 
\end{align}
There isn't any {\it spatial} structure since $\spatial(\vr)=0$. 
Given a budget index $r$, the kernels can only learn $\Span\{\bm\barY_{\vr, \vl}: \vr<r \, ,\,\vr \in 
\mathbb N\backslash\{ 0\}\}$ and MLPs are good at modeling LRLF interactions. Note that infinite-width MLPs require $\sim d^{l}$ many training points / SGD steps to learn any degree $l$ polynomials, which quickly becomes infeasible as $d$ becomes large, e.g. when $d=224^2\times 3$ (ImageNet) and $l=3$, $d^3\sim 10^{18}$. Thus MLPs suffer from the curse of dimensionality.   

\paragraph{S-CNNs and {\color{color1}SRHF Interactions} Fig.~\ref{fig:high-level-picture} (b).}
For one-hidden layer convolutional networks, we have $d = p\times k^0\times w$, (i.e. $L=0)$. The number of input nodes is $w\equiv d^{\alpha_w}$ and for each input node $v$ we have $d_v = p\equiv d^{\alpha_p}$. We have $\alpha_p+\alpha_w=1$. Since the activation function of the last layer is the identity function, there is no non-linear interactions between different patches and $\Ad(\Gd) = \{\vr \in\mathbb N^{|\Ndzero|}\backslash\{\bm 0\}, |\nnode(\vr)|= 1\}$. Therefore for $\vr\in \Ad(\Gd)$, 
\begin{align}
   \spatial(\vr) = \alpha_w\, ,  \quad \frequency(\vr) = |\vr|\alpha_p, \quad \learning_{{\gG}} =\{ \alpha_w + |\vr|\alpha_p, \vr\in \Ad(\Gd)\}
\end{align}
Given a budget index $r$, the learnable $\vr$ are the ones $\vr\in\Ad(\Gd)$ with 
\begin{align}\label{eq:short-range-interaction}
  r > \alpha_w + |\vr|\alpha_p = 1 + (|\vr|-1) \alpha_p \Rightarrow |\vr| < 1 + (r -1) / \alpha_p \,. 
\end{align}
When $\alpha_p=1$ (and thus $\alpha_w=0$), this S-CNN is essentially a shallow MLP and we have $|\vr|<r$. On the other hand, if $\alpha_p$ is small (say $\alpha_p=0.1$), this network can model interactions with much higher frequencies at the cost of giving up long-range interactions. Therefore, there is a trade-off between space and frequency, and S-CNNs with small patch sizes are good at modeling SRHF interactions. That is to say, by choosing the patch size small, S-CNNs can break the curse of dimensionality (learning certain very high frequency interactions) at the cost of dramatically reducing the networks' expressivities. 

\paragraph{D-CNNs and {\color{color1}Inter}{\color{color3}ter}{\color{color5}polation} 
Fig.~\ref{fig:high-level-picture} (c).} Recall that $d=p \times k^L \times w$. Let $p = d^{\alpha_p}, k=d^{\alpha_k}$ and $w=d^{\alpha_w}$. Then we have the constraint $\alpha_p + L \alpha_k + \alpha_w=1$.
% The $\alpha_k=0$ and $\alpha_w=0$ case corresponds essentially to $L$-layer fully-connected networks. 
The learnable terms are the ones with 
\begin{align}\label{eq:deep-cnn-budget}
r > \spatial(\vr) + \frequency(\vr) = 
  \spatial(\vr) + |\vr| \alpha_p . 
\end{align}
D-CNNs can simultaneously model interpolations (including the two ends) between LRLF and SRHF, i.e., MRMF interactions. 
Consider two extreme cases. (i.) $|\nnode(\vr)|=1$, i.e. there is an input node $v$ s.t. $\vr_v=|\vr|$ and thus the non-linear interaction happens within one patch. In this case, the length of the MST reaches its infimum $\spatial(\vr)= \alpha_w + L\alpha_k=1-\alpha_p$ and \myeqref{eq:deep-cnn-budget} implies $|\vr| < (r-1)/\alpha_p + 1$, which is exactly \myeqref{eq:short-range-interaction}. Thus D-CNNs can model SRHF interactions.
% see Fig.~\ref{fig:nn-demonstration-2} ${\bf\color{plot_5}{Y_5^*}}$. 
(ii.) $|\nnode(\vr)|=|\vr|$, i.e. $\vr_v=0$ or $1$ for all $v\in\Ndzero$. Then $\spatial(\vr) = |\vr| (1-\alpha_p)$ and \myeqref{eq:deep-cnn-budget} implies $|\vr|<r$, which is the constraint for MLPs. In particular, D-CNNs can model LRLF interactions; see ${\bf\color{plot_8}{Y_5}}$. By varying $|\vr|$ and then varying $|\nnode(\vr)|$ in $[1, |\vr]$, D-CNNs can model various interpolating interactions between LRLF and SRHF.

% Let $|\nnode(\vr)|=\bar r\in [1, r]$ be fixed. Then the range of $\spatial(\vr)$ is 
% \begin{align}
%     \gR_{\bar r}(\alpha_p, \alpha_w, \alpha_k)\equiv \left\{a \alpha_w +  b\alpha_k: a\in[1, \bar r] ,\, \,
%     b\in [a(L-1) + \bar  r, \, \bar r L], \text{ and }  a, b\in \mathbb N\right\}
% \end{align}
% The length can be as small as $(\alpha_w + L \alpha_k + (\bar  r-1) \alpha_k)$ and as large as $(\bar  r(\alpha_w + L\alpha_k))$. The former/latter happens when $\nnode(\vr)$ is chosen in a way that the paths from $\nnode(\vr)$ to $\rootg$ is ``maximally"/``minimally" shared (see e.g.  Fig.~\ref{fig:nn-demonstration-2} ${\bf\color{plot_5}{Y_5^*}}$/${\bf\color{plot_8}{Y_5}}$). The former/latter leaves more/less budget for learning high-frequency interactions. In general, the set of learnable indices includes various interpolations between these two extreme spectra (see Fig.~\ref{fig:high-level-picture} (c)):   
% \begin{align}
% \label{eq: learnable index}
%   \left\{\vr: a \alpha_w +  b\alpha_k +|\vr| \alpha_p < r,  a\in[1, |\nnode(\vr)|] ,\, \,
%     b\in [a(L-1) + |\nnode(\vr)|, \, |\nnode(\vr)| L], \,\, a, b\in \mathbb N\right\} 
% \end{align}
\paragraph{HR-CNNs {\color{color6}Extra}{\color{color0}polations} {and} {\color{color2}Finer} {\color{color4}Interpolations} Fig.~\ref{fig:high-level-picture} (d).} The resolution of the learning indices $\learning(\mathcal \gG^{(d)})$ can be improved by decreasing $\alpha_k$, $\alpha_p$ and increasing $L$ accordingly. E.g., changing $\alpha_p\to\frac {\alpha_p} 2$, $\alpha_k\to\alpha_k/2$ and doubling the number of convolutional layers accordingly,  
then the resolution of the range of spatial/frequency/learning indices is doubled. 
This empowers the network to model finer-grained interpolating modes and extrapolating modes. 
E.g., the last equation in \myeqref{eq:short-range-interaction} becomes $|\vr| < 1 + 2(r-1)/\alpha_p$ (almost doubles the upper bound of $|\vr|$) and the network can additionally model
{\bf\color{color0}Ultra-Short-Range-Ultra-High-Frequency} interactions (see {\color{plot_5} ${\bf Y_5^*}$} in Fig.~\ref{fig:learning-eigen-modes}) without sacrificing its expressivity, which isn't the case for S-CNNs due to the space-frequency trade-off. 
We refer to such networks as high-resolution CNNs (HR-CNNs). For practical networks, e.g., ResNet \citep{he2016deep}, the filters/patches are already quite small, and there isn't much room to reduce them. Equivalently, one increases the resolution of the input images instead, i.e., increasing $d$. From this point of view, HR-CNNs and, therefore, our theorems justify the additional performance gain from model scaling, e.g. EfficientNet \citep{tan2019efficientnet}.     

In Sec.~\ref{Sec:learning, dags}, we provide several concrete examples to illustrate how to compute the spatial/frequency/learning indices. Overall, finer network architectures imply finer eigenstructures, which imply better performance in terms of compute and data.    

\section{Experiments}\label{sec:experiment}
There are many practical consequences due to the above theorems. For Theorem \ref{thm:main} and Theorem \ref{thm:generalization}, 
we focus on two of them which are about {\it the impact of architectures to learning / generalization}: 
\begin{enumerate}
    \item {\bf Order of Learning ({Fig.~\ref{fig:high-level-picture} \color{plot_green}{Green Arrow}}.)} The order of learning is restructured from frequency-based (MLPs) to space-and-frequency-based (CNNs). 
    % for MLPs depends only on the {\it frequency} (from low to high). In contrast, for D-CNNs it depends monotonically on the sum of the frequency and the spatial indices.     
    \item {\bf Learnability
    ({Fig.~\ref{fig:high-level-picture} Cells under the budget line.)}} With the same budget index, MLP-Learnable $\subsetneq$ D-CNN-Learnable $\subsetneq$ HR-CNN-Learnable. Moreover, the set differences between these learnable sets are captured as in Sec.\ref{sec: interpretation}. 
    % We want to observe (1) D-CNNs are better than MLP as they can learn a much broader family of functions with various space-frequency combinations; (2) Adding more layers to MLPs do not improve their performance, while HR-CNNs improve over D-CNN as they can model functions with finer space-frequency combinations.      
    % MLPs are highly inefficient in learning high-frequency interactions. Deep CNNs are able to simultaneously model the SRHF interactions, the SRLF interactions, and their interpolations. Moreover, higher resolution CNNs (increasing the number of layers and decreasing the filter sizes) can model a larger class of modes, e.g., ultrashort-range-ultrahigh-resolution modes.   
\end{enumerate}
For Theorem \ref{thm:cnn+gap}, we focus on the training and data efficiency of {\it CNN+GAP} (GAP for short) and {\it CNN+Flatten} (Flatten for short): 
\begin{enumerate}
    \item When the dataset size is sufficiently large, GAP and Flatten are equally efficient.
     \item When the dataset size is relatively small compared to the learning index, GAP is more data-efficient.
\end{enumerate}

Overall, we see excellent agreements between predictions from our theorems and experimental results from both practical-size networks and kernel methods using NNGP/NT kernels, even when $d=256$ is moderate-size. We detail the setup, results, corrections, etc., for the experiments below. 

\subsection{Experimental Setup}\label{sec:setup}
Set $d=p^4$ and the input $\bm \X = (\spbar)^{p^3} \subseteq \mathbb R^{p^4}$, where $p\in\mathbb N$. Note that $\alpha_p=1/4$. The task is learning a function $Y\in L^2(\bmx)$ by minimizing the MSE, where 
\begin{align}
    { Y} = {\bf\color{plot_1}Y_1} + 
    {\bf\color{plot_2}Y_2} + 
    {\bf\color{plot_3}Y_3} + 
    {\bf\color{plot_4}Y_4} + 
    {\bf\color{plot_8}Y_5} + 
    {\bf\color{plot_5}Y_5^*} + 
    {\bf\color{plot_6}Y_6} +
    {\bf\color{plot_7}Y_7}  
\end{align}
and each eigenfunction $Y_i$ is a normalized so that $\|Y_i\|_2^2=\|Y\|_2^2/8=1$. See Sec.~\ref{Sec:learning, dags} for the expressions of these functions. 
% See Fig.~\ref{fig:nn-demonstration-2} for more details about $Y_i$.
% \begin{figure}
%     \centering
%     \includegraphics[width=0.24\textwidth]{figures/sep27-NN-graphs/Y-cnn4 (1).pdf}
%     \includegraphics[width=0.24\textwidth]{figures/sep27-NN-graphs/Y-cnn2 (1).pdf}
%     \includegraphics[width=0.24\textwidth]{figures/sep27-NN-graphs/Y-mlp (1).pdf}
%     \includegraphics[width=0.22\textwidth]{figures/Sep27-residual/NTK-Residual (2).pdf}
%     \caption{Eigen-functions with various space-frequency combination.}
%     \label{fig:my_label}
% \end{figure}

We optimize finite-width networks by {\it SGD+Momentum} and infinite-width networks (NNGP and NTK) by kernel regression. We investigate three types of architectures: (1) $\text{MLP}^{\otimes 4}$, a four hidden layer MLP; (2) $\text{Conv}(p^2)^{\otimes 2}$, a ``deep" CNN with filter size/stride $k=p^2$, and (3) $\text{Conv}(p)^{\otimes 4}$, a ``HR"-CNN with filter size/stride $k=p$. See Fig.\ref{fig:high-level-picture} for a visualization of the associated DAGs. There is an activation $\phi$ in each {\it hidden} layer, which is chosen so that $\phi^*$ is the Gaussian kernel. For the CNNs, the readout layer(s) is {\it Flatten-Dense-Act-Dense}. We carefully chose the eigenfunctions $\{Y_i\}$ so that they cover a wide range of space-frequency combinations $(\spatial(Y_i), \frequency(Y_i))$ w.r.t. $\text{Conv}(p)^{\otimes 4}$. Under $\text{Conv}(p)^{\otimes 4}$, the corresponding learning indices are $\learning(Y_i) = \spatial(Y_i) + \frequency(Y_i) = 3\alpha_p + i\alpha_p=(3+i)/4$. For the learning indices of $Y_i$ under $\text{Conv}(p^2)^{\otimes 2}$ or $\text{MLP}^{\otimes 4}$, see the legends in Fig.\ref{fig:learning-eigen-modes}. 
The purpose of doing so is to create a ``separation of learning" under $\text{Conv}(p)^{\otimes 4}$, since in the large $p$ limit, learning $Y_i$ requires $d^{(3+i)/4 +\epsilon}$ examples/SGD steps. The relation between architectures and learning indices can be ``seen" in Fig.~\ref{fig:nn-demonstration} for { \color{plot_2}${\bf Y_2}$}, which is short-range-low-frequency ( $\deg({\color{plot_2}{\bf Y_2}})=2$). Fig.~\ref{fig:nn-demonstration} (d) plots the test residual of { \color{plot_2}${\bf Y_2}$} vs training set size for NTK regression, as one example to showcase the relation among architectures, learning indices and generalization. See Fig.~\ref{fig:nn-demonstration-2} in the appendix for other eigenfunctions.  

% As such, we expect a {\it separation of learning} w.r.t. to training step or training set size when $p$ is sufficiently large.
% Unfortunately, due to the compute constraint of kernel regression, we are only able to choose $p$ to be a small number\footnote{For example, when $p=4$, then $d^2=4^8=65536$, around which we see progress on learning quadratic eigenfunctions for infinite width MLP. In our experiments, about $m=160,000$ data points are required to learn quadratic modes.} $p=4$. 
In the experiments, the width/number of channels is set to 2048/512 for MLPs/CNNs. We sample $m_t=32\times 10240$ ($m_v=10240$) data points randomly from $\bmx$ as training (test) set with $p=4$ and $d=256$. 
The SGD training configurations (batch size (=10240), learning rate (=1.), momentum (=0.9) etc.) are identical across architectures. To compute the kernels, we rely crucially on {\it NeuralTangents} \citep{neuraltangents2020} which is based on {\it JAX} \citep{jax2018github}. 
% Sec.\ref{sec:code} contains more details regarding the architectures used in the paper. 

\subsection{Experimental Results I: Architectures vs Learnability.}
In Fig.\ref{fig:learning-eigen-modes}, for each eigenfunction $Y_i$, we plot $\frac 1 2 \E|\hat Y_i(x, t) - Y_i(x)|_2^2$ against $t$, where $\hat Y_i(x, t)$ is the projection of the prediction onto $Y_i$ and $t$ is either the training steps (SGD) or training set size (kernels). The expectation is taken over the test set. 
The budget index $r=\log(m_t)/\log(d)\approx 2.28$. As $d=256$ ($p=4$) is far from the asymptotic limit, we expect $r=2.28$ being a {\it soft} cut-off between learnable and non-learnable indices. Although the theorems assume $d, p\to\infty$, they do provide good predictions even when $d$ and $p$ are far from $\infty$.
We summarize several key observations below. 
\begin{enumerate}
    \item {\bf $\text{MLP}^{\otimes 4}$} (1st Row.) This architecture can capture all low-frequency interactions ($\deg=1,2$, ${\bf\color{plot_1}Y_1}$, ${\bf\color{plot_2}Y_2}$, ${\bf\color{plot_3}Y_3}$, ${\bf\color{plot_8}Y_5}$) but fail to learn $\deg \geq 3$ interactions, as expected. For MLPs, making the network deeper won't improve its learnability much; see Fig.~\ref{fig:mlp-2} in the appendix. 
    % In stark contrast, the learnability of CNNs is improved by having more layers. Without locality, the hierarchy may make little sense. 
    
    \item {\bf $\text{Conv}(p^2)^{\otimes 2}$} (2nd Row.)
    Learning curves of ${\bf\color{plot_2}Y_2}$/${\bf\color{plot_3}Y_3}$ are separated from ${\bf\color{plot_8}Y_5}$ because the spatial indices of them are different. Higher-frequency ($\deg=3, 4$) shorter-range interactions (${\bf\color{plot_4} Y_4}$, {$\bf\color{plot_6} Y_6$}) become (partially) learnable. Note that $\gL({\bf\color{plot_4} Y_4})=\frac{8}{4} < r\approx 2.28$, $\gL({\bf\color{plot_6} Y_6})=\frac{10}{4}<r\approx 2.28$. 
    
    \item {\bf$\text{Conv}(p)^{\otimes 4}$} (3rd Row). We see $\gL(Y_i)$ capture the order of learning very/reasonably well in the kernel/SGD setting. To test the ability of $\text{Conv}(p)^{\otimes 4}$ in modeling ultra-short-range-ultra-high-frequency interactions, we trace the learning progress of ${\bf\color{plot_5} Y_5^*}$ ($\deg({\bf\color{plot_5} Y_5^*})=5, \gL({\bf\color{plot_5}Y_5^*})=\frac 84$.)
    As expected, while other architectures completely fail to make progress, the NTK/NNGP of $\text{Conv}(p)^{\otimes 4}$ makes good progress and the SGD even completes the learning process. Interestingly, ${\bf\color{plot_8}Y_5}$\footnote{Ultra-Long-Range-Low-Frequency under $\text{Conv}(p)^{\otimes 4}$, with $\deg({\bf\color{plot_8}Y_5})=2$ and $\gL({\bf\color{plot_8}Y_5})=8/4=\gL({\bf\color{plot_5}Y_5^*})$ } is learned faster than ${\bf\color{plot_5} Y_5^*}$ in the kernel setting but slower in the SGD setting (even slower than $\gL({\bf\color{plot_6}Y_6})=\frac {10}4$), which is unexpected.
    We suspect it might be due to certain ``implicit" effects of SGD. Further investigation is needed to understand it.
    % Able to learn ultra-high frequency (fifth-order terms within a patch). For MLP kernels, it would require $d^5$ many examples. For 4layers CNN, it needs only about  $d^{3\alpha_p + 5\alpha_p}=d^2$. For 2-layer CNN, it requires $p^2 * (p^2)^5 = d^3$. Higher-resolution CNNs provide a higher resolution decomposition of the eigenspace of the kernels. 
    
\end{enumerate}

\begin{figure}[t]
    \centering
    % \includegraphics[width=\textwidth]{figures/sep27-320K/dense1.pdf}
    % \\
    \includegraphics[width=0.9\textwidth]{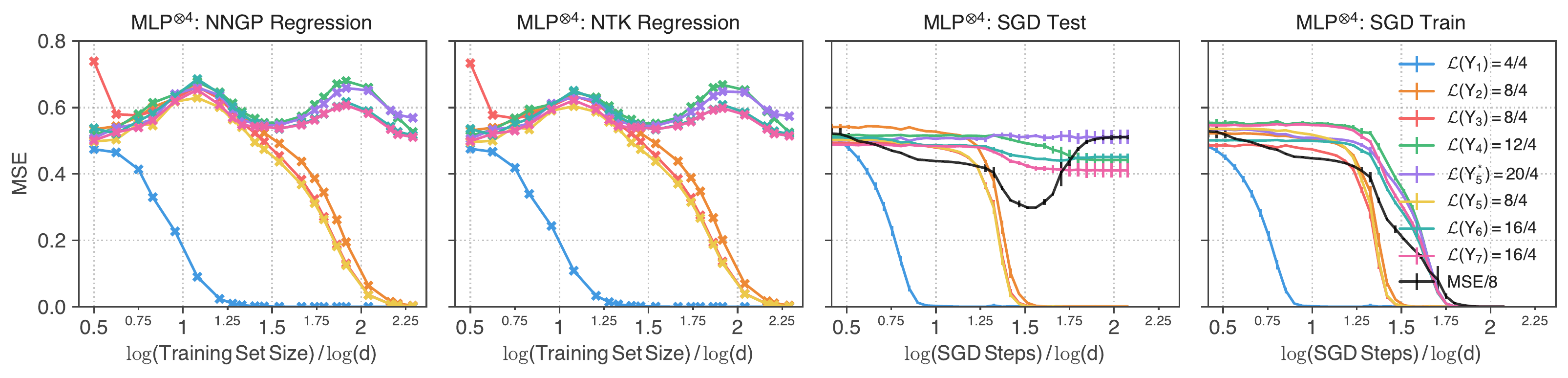}
    \\
    \includegraphics[width=0.9\textwidth]{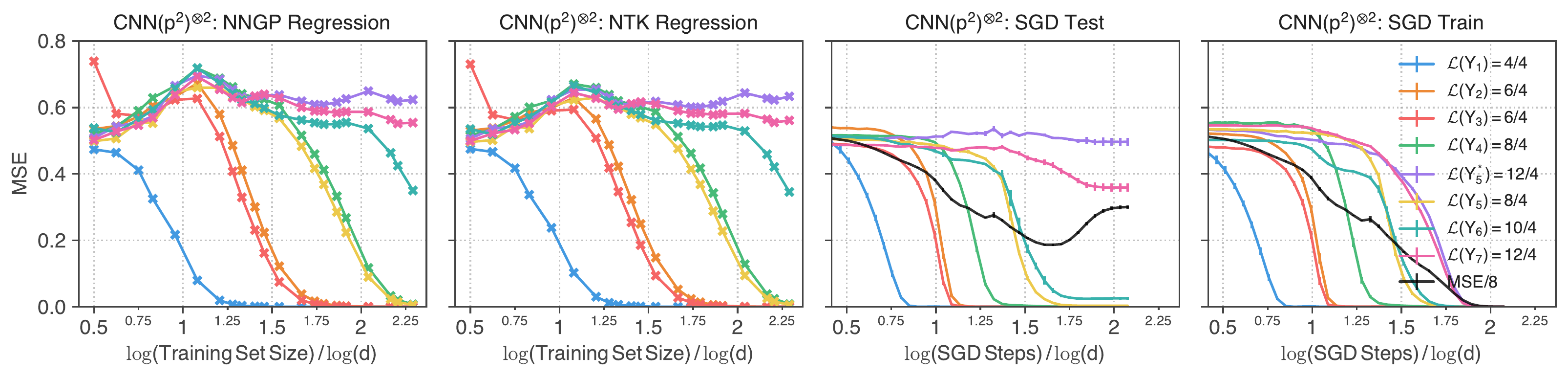}\\
    \includegraphics[width=0.9\textwidth]{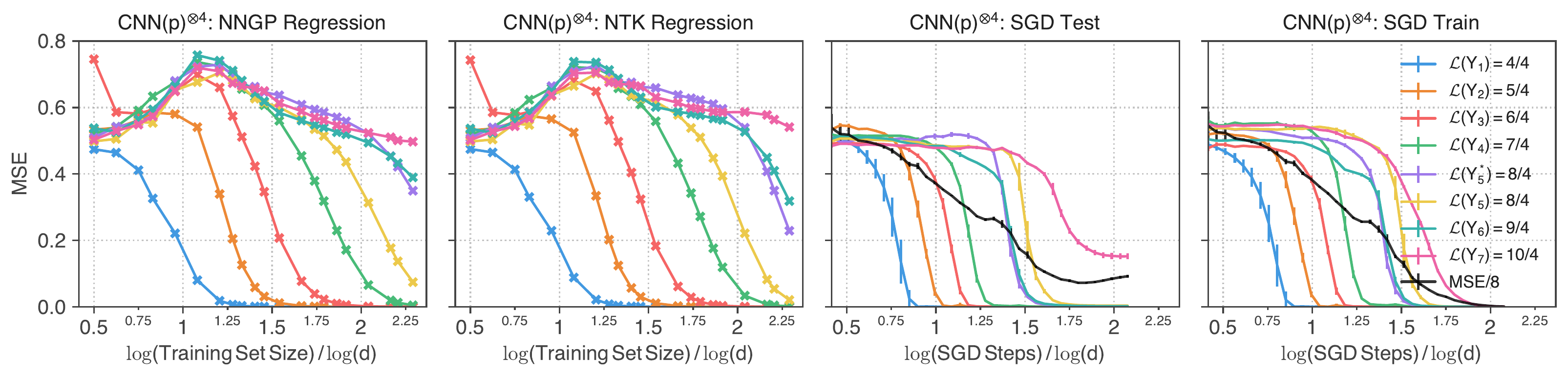}
    \caption{{\bf Learning Dynamics vs Architectures vs Learning Indices.} We plot the learning/training dynamics of each eigenfunction $Y_i$. From top to bottom: a 4-layer MLP $\text{MLP}^{\otimes 4}$, a 2-layer CNN $\text{Conv}(p^2)^{\otimes 2}$ and a 4-layer CNN $\text{Conv}(p)^{\otimes 4}$. From left to right: residual MSE (per eigenfunction) of NNGP/NTK regression, test/training MSE of SGD. The learning indices of $Y_i$ in each architecture are shown in the legends.} 
    \label{fig:learning-eigen-modes}
\end{figure}

\subsection{Experimental Results II: GAP vs Flatten.}\label{sec:exp2}
We compare the SGD learning dynamics of two convolutional architectures: $\text{Conv}(p)^{\otimes 3}\text{-Flatten}$ and $\text{Conv}(p)^{\otimes 3}\text{-GAP}$. The experimental setup is almost the same as that of Sec.~\ref{sec:setup} except the eigenfunctions $\{Y_i\}$ are chosen to be in the RKHS of the NNGP kernel/NTK of $\text{Conv}(p)^{\otimes 3}\text{-GAP}$ and thus of $\text{Conv}(p)^{\otimes 3}\text{-Flatten}$, i.e., they are shifting-invariant (the invariant group is of order $p$). 
Moreover, we still have $\learning(Y_i)=(i+3)/4$. For each $Y_i$, we plot the validation MSE of the residual vs SGD steps in Fig.~\ref{fig:gap_vs_flatten}. Overall, the predictions from Theorem \ref{thm:cnn+gap} gives excellent agreement with the empirical result. With training set size $m_t=32\times 10240$ Fig.~\ref{fig:gap_vs_flatten} (a), the residuals of $Y_i$ for GAP and Flatten are almost indistinguishable from each other for $i\leq 6$ (recall that $\learning({\bf\color{plot_6}Y_6})=9/4=2.25 < r\approx2.28)$. However, when $i=7$ the dataset size ($r\approx 2.28$) is relatively small compared to the learning index ($\learning({\bf\color{plot_7}Y_7})=2.5$), 
GAP outperforms Flatten in learning ${\bf\color{plot_7}Y_7}$.  In Fig.~\ref{fig:gap_vs_flatten} (b), we increase the training set size by a factor of 4 to $m_t= 4\times 32\times 10240\approx d^{2.53}$. We see that the dynamics of ${\bf\color{plot_7}Y_7}$ between GAP and Flatten become indistinguishable.  
\begin{figure}[h]
    \centering
    \begin{subfigure}[b]{0.45\textwidth}
    \centering    
    \includegraphics[width=\textwidth]{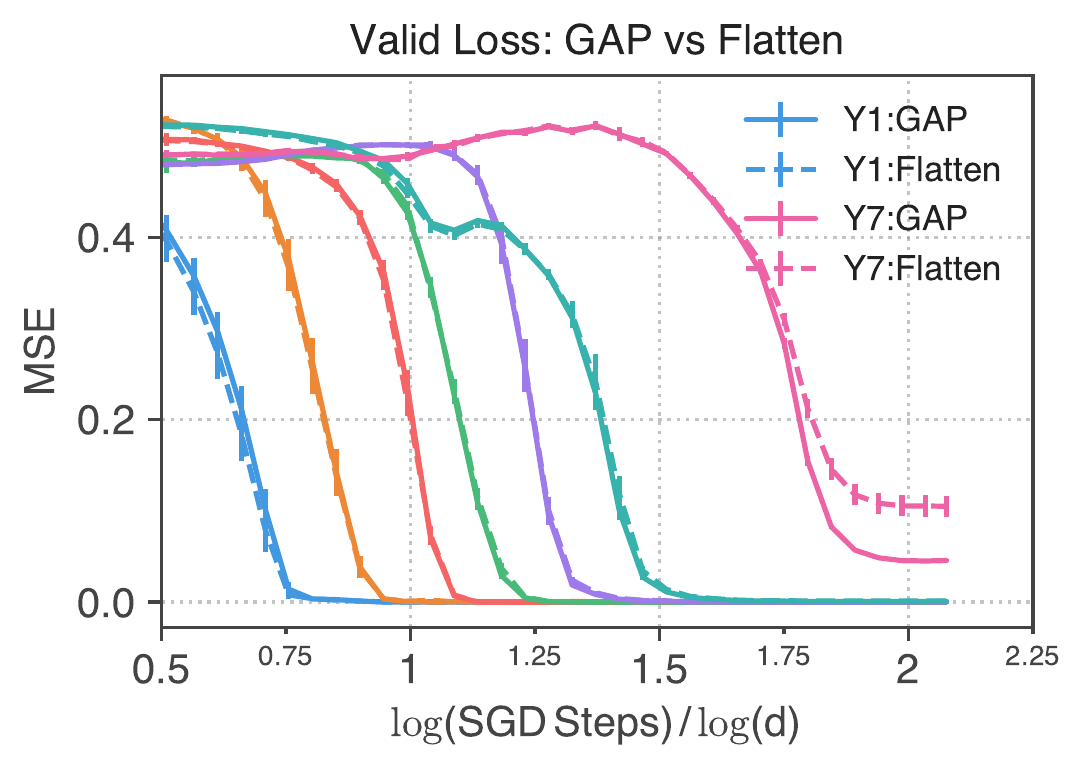}
    \caption{Training Set Size: $32 \times 10240$}
    \end{subfigure} 
        \begin{subfigure}[b]{0.45\textwidth}
    \centering    
    \includegraphics[width=\textwidth]{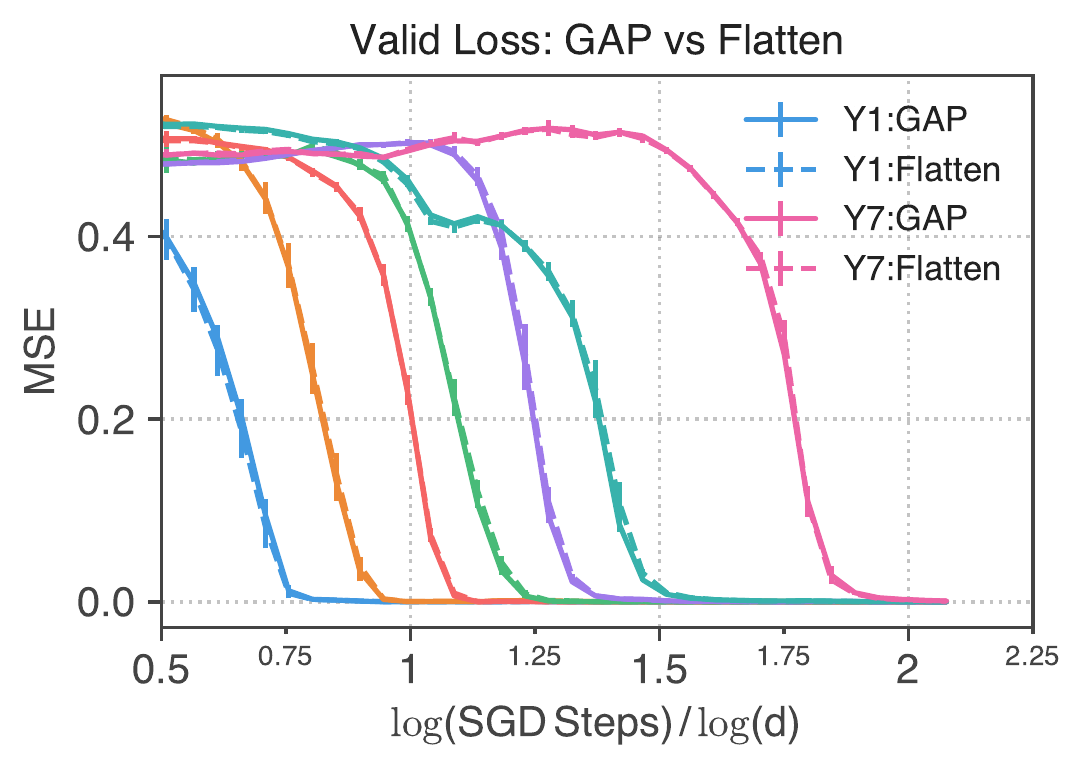}
    \caption{Training Set Size: $128 \times 10240$}
    \end{subfigure} 
    \caption{{\bf Learning Dynamics: GAP vs Flatten.} We plot the validation MSE of the residual of each $Y_i$ ({\bf\color{plot_1}{left}} $\to$ {\bf\color{plot_7}right}: $i=1\to 7$) for GAP (Solid lines) and Flatten (Dashed lines). The mean/std in each curve is obtained by 5 random initializations. Left: when training set size is $32\times 10240$, the residual dynamics of GAP and Flatten are almost indistinguishable for $Y_i$ with $i\leq 6$. But GAP outperforms Flatten when learning ${\bf\color{plot_7}Y_7}$.
    Right: the training set size is increased by a factor of $4$ and the learning dynamics of ${\bf\color{plot_7}Y_7}$ become identical for GAP and Fatten.
    } 
    \label{fig:gap_vs_flatten}
\end{figure}

To test the robustness of the prediction from Theorem \ref{thm:cnn+gap} (b) on more practical datasets and models, we perform an additional experiment on ImageNet~\citep{5206848} using ResNet~\citep{he2016deep}. We compare the performance of the original ResNet50, denoted by ResNet50-GAP, and a modified version ResNet50-Flatten, in which the GAP readout layer is replaced by Flatten. We use the ImageNet codebase from FLAX\footnote{\hyperlink{https://github.com/google/flax/blob/main/examples/imagenet/README.md}{https://github.com/google/flax/blob/main/examples/imagenet/README.md}}\citep{flax2020github}. 
In order to see how the performance difference between ResNet50-GAP and ResNet50-Flatten evolves as the training set size increases, we make a scaling plot, namely, we vary the training set sizes\footnote{Standard data-augmentation is applied for each $m_i$; see \hyperlink{https://github.com/google/flax/blob/main/examples/imagenet/input_pipeline.py}{input\_pipeline.py}.} $m_i = [m \times 2^{-i/2}]$ for $i=0, \dots, 11$, where $m=1281167$ is the total number of images in the training set of ImageNet.
The networks are trained for 150 epochs with batch size 128. 
We plot the validation accuracy and loss (averaged over 3 runs) as a function of training set size $m_i$ in Fig.~\ref{fig:imagenet}. Overall, we see that the performance gap between ResNet50-GAP and ResNet50-Flatten shrink substantially as the training set size increases. E.g, 
using 1/8 of the training set (i.e. $i=6$), the top 1 accuracy between the two is 19.3\% (57.7\% GAP vs 38.4\% Flatten). However,  
with the whole training set (i.e., $m_0$), this gap is reduced to 2\% (76.5\% GAP vs 74.5\% Flatten). To demonstrate the robustness of this trend,  
we additionally generate the same plots for ResNet34 and ResNet101; see Fig.~\ref{fig:imagenet-34-101} in the Appendix. 
\begin{figure}[ht]
    \centering
    \includegraphics[width=0.9\textwidth]{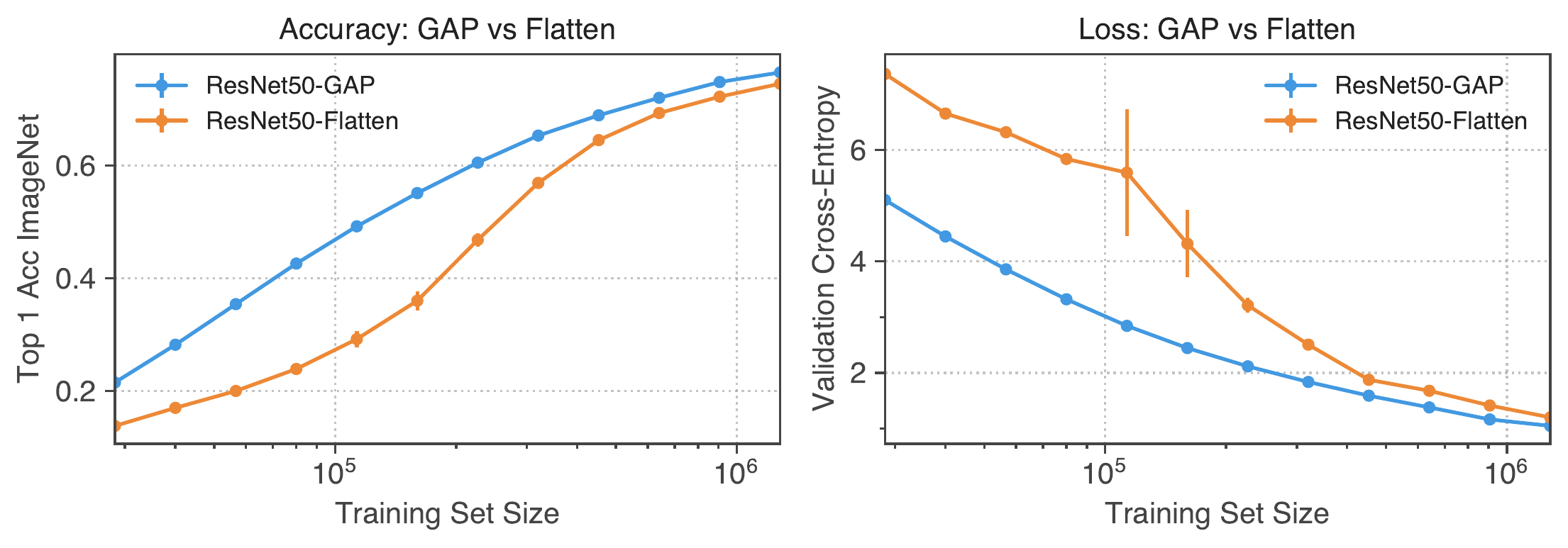}
    \caption{{\bf ResNet50-GAP vs ResNet50-Flatten.} As the training set size increases the performance (accuracy and loss) gap between the two shrinks.}
    \label{fig:imagenet}
\end{figure}

\section{Additional Related work}\label{sec:related} 
\citet{bietti2021approximation} studies the approximation and learning properties of CNN kernels via the lens of RKHS. Impressively, they demonstrate that a 2-layer CNN kernel can reach 88.3\% validation accuracy on CIFAR-10, matching the performance of a 10-layer Myrtle kernel \cite{Shankar2020NeuralKW}. 
\citet{DBLP:journals/corr/abs-2010-01369} and \citet{DBLP:journals/corr/abs-2010-08515} study the algorithmic benefits of shallow CNNs and show that they outperform MLPs in certain tasks. \citet{xiao2021what} and \citet{favero2021locality} study the benefits of locality in S-CNNs and argue that localities are the key ingredient to defeat the curse of dimensionality. However, the models they studied are essentially a linear combination of one-layer MLPs, each of which is defined on a batch. The models have very limited expressivity and can not capture any nonlinear interactions between different patches. \citet{mei2021learning} and several papers mentioned above also study the benefits of pooling in (S-)CNNs in terms of data efficiency. \citet{mei2021learning} is the first to rigorously prove the benefits and limitations of pooling when the patch size is equal to the spatial dimensions. Their conclusion is similar to that of Theorem \ref{thm:cnn+gap} (c): pooling improves data efficiency by a factor of the pooling size. In addition, we show that (Theorem \ref{thm:cnn+gap} (b)) pooling does not improve training efficiency for D-CNNs, extending a result from \citet{xiao2021what} which concerns S-CNNs. Finally, \citet{scetbon2020harmonic} also studies the eigenstructures of certain CNN kernels without pooling. Their kernels can be considered as a particular case of the NNGP kernels, where the associated networks have only one convolutional layer and multiple dense layers. The key contribution that sets the current work apart from existing work is that our work offers a precise mathematical characterization of the fundamental role of architectures in (infinite-width) networks through a space-frequency analysis. In particular, we rigorously prove that D-CNNs (more precisely, hierarchical localities) break the curse of dimensionality without losing their expressivity and scaling improves the performance in both finite and infinite data regimes.

\section{discussion and conclusion}\label{sec:discussion}

We establish a precise relation among networks' architectures, eigenstructures of the inducing kernels, and generalization of the corresponding kernel machines in the high-dimensional setting. We show that deep convolutional networks restructure the eigenspaces of the inducing kernels, which empowers them to learn a dramatically broader class of functions, covering a wide range of space-frequency combinations without extra data. We rigorously prove that infinite-width convolutional networks break the curse of dimensionality without losing any expressivity. On the practical side, our results suggest that using images with higher resolutions, networks with more layers and smaller filter sizes, and increasing the training set size\footnote{This may seem obvious, but our results suggest that HR-CNN type of networks could {\it smoothly} utilize the increment of training data since the set of learning indices has higher resolutions.} can improve the performance of the models, which have already been observed by practitioners \citep{tan2019efficientnet, kaplan2020scaling}.      

We believe our framework can be extended to study architectural inductive biases for other families of topologies, such as RNNs, GNNs, and even self-attention. It is expected that the exact mathematical formulation of the {\it learning index} will depend on the structure of the model of interest. We have not covered the learning dynamics of SGD in the feature learning regime, which is a very challenging topic. In addition, it is of great interest and importance to study the combined effect of SGD and architectural inductive biases in the future. Our results suggest that better architectures (such as D-CNNs, HR-CNNs) provide a better search space for SGD in the sense that the initial descent direction, namely, the NTK, biases the model towards learning functions with lower complexity (defined by the learning index). Another limitation of the current paper lies in the strong assumption on the input space, which requires the data to be drawn {\it uniformly} from a product of hyperspheres. There are two orthogonal directions to relax this assumption. One is to restrict the support of each patch in some low dimensional compact manifold, and we expect tools from studying eigenfunctions on Riemannian manifolds can be helpful here \citep{zelditch2017eigenfunctions}. The other one is to impose dependence among patches, e.g., {\it spatially close} patches have strong correlations.

Finally, our results also suggest the need for rethinking approximation power of {\it finite-width} neural networks via the lens of {\it space-frequency} analysis, at least for convolution-based architectures. While the frequency structures, namely, the smoothness of the target functions in terms of the Sobolev norm, have been broadly studied during the last decades \citep{Cucker2001OnTM, von2004distance, bach2017breaking}, spatial structures are mostly overlooked. Consequently, the required widths of the networks to approximate certain functions often grow rapidly (often polynomially) in the input dimensions \citep{eldan2016power, daniely2017depth}, i.e., the growth rates of the widths and the total number of parameters are cursed by the dimensionality. Nevertheless, the widths used in SotA models are much smaller than the input dimensions, e.g., the largest width in ResNets \citep{he2016deep} is $2048 \approx d^{0.64}$, where $d=224^2 \times3$ is the (default) input dimension of ImageNet used by the models. We conjecture that hierarchical localities play a similar fundamental role in breaking the curse of dimensionality for approximation in neural networks, bringing down the growth rate of the widths to a more realistic range.

% In practice, (neural) kernel regressions are computationally highly inefficient compared to neural networks due the prohibitive expensiveness of kernel computations. E.g, NNGP kernels (or NTKs) for CNNs with pooling take more than 100 GPU hours to obtain $90\%-$ test accuracy on CIFAR10, while WideResNets use $<5$ GPU hours of training time to reach $95\%+$ accuracy, which is $>20\times$ more efficient. This gap can be much larger for more practical datasets (e.g. ImageNet).  

% a very similar result (such as Theorem \ref{thm:generalization}) that relates the topological structures of the networks to the widths required to approximate certain eigenfunctions, and hierarchical locality breaks the curse of dimensionality in this setting. 

\section*{Acknowledgments and Disclosure of Funding}
We thank Atish Agarwala, Ben Adlam,  Jaehoon Lee, Jascha Sohl-dickstein, Jeffrey Pennington, Roman Novak and Sam Schoenholz for discussions and feedbacks on the project. We are also grateful to Yasaman Bahri for providing valuable feedbacks on a draft.

We acknowledge the Python community~\citep{van1995python} for developing the core set of tools that enabled this work, including NumPy~\citep{numpy}, SciPy~\citep{scipy}, Matplotlib~\citep{matplotlib}, Pandas~\citep{pandas}, Jupyter~\citep{jupyter}, JAX~\citep{jaxrepro}, Neural Tangents~\citep{neuraltangents2020}, FLAX~\citep{flax2020github}, Tensorflow datasets~\citep{TFDS} and Google Colaboratory~\citep{colab}.

This work was performed at and funded by Google. No third party funding was used.

\bibliography{iclr2022_conference}
\bibliographystyle{iclr2022_conference}

\newpage 
\appendix
% \section{Appendix}
\section{DAGs, Eigenfunctions, Spatial Index, and Frequency Index.}\label{Sec:learning, dags}
In this section, we provide more details regarding the DAGs and the eigenfunctions used in the experiments, and how the spatial, frequency and learning indices are computed. 

Let $ (\spbar)^{p^3} \subseteq \mathbb R^{p^4}$ be the input space, where $d=p^4$ is the input dimension. We use $\vx \equiv (\vx_\vk)_{\vk\in [p]^4} \in  (\spbar)^{p^3}$ to denote one input (an image), where  $\vk = [\vk_1, \vk_2, \vk_3, \vk_4] \in [p]^4$. In addition, we treat $[p]^4$ as a group (i.e. with circular boundaries) and let $\ve_1 = [1, 0, 0, 0]$,  $\ve_2 = [0, 1, 0, 0]$ , 
$\ve_3 = [0, 0, 1, 0]$ and $\ve_4 = [0, 0, 0, 1]$ be a set of generator/basis of the group. Note that each  input $\vx_\vk$ is partitioned into $p^3$ many patches: $\{\vx_{\vk_1, \vk_2, \vk_3, :}: \vk_1, \vk_2, \vk_3\in [p] \}$. 

The eigenfunctions used in the experiments and the associated space/frequency indices (will be explained momentarily) are given as follows
\begin{align*}
&\text{\bf Eigenfunction} &&{\text{\bf degree}}   &&& &&&& \text{\bf Space/Freq Index}
\\
&    &&  &&&\text{MLP}
&&&&\text{CNN}(p^2)^{\otimes{2}} &&&&&\text{CNN}(p)^{\otimes{4}}
\\
{\bf \color{plot_1}Y_1}(\vx) &= \sum_{\vk\in[p-1]^4} c^{(1)}_{\vk}\vx_{\vk} && {\bf \color{plot_1}1} &&&  {\bf\color{plot_1}0/1}  &&&&{\bf\color{plot_1}\frac 12 / \frac 12} &&&&&{\bf\color{plot_1}\frac 34 /\frac 14}
\\ 
{\bf \color{plot_2}Y_2}(\vx) &= \sum_{\vk\in[p-1]^4} c^{(2)}_{\vk} \vx_{\vk}  \vx_{\vk+\ve_4}
&& {\bf \color{plot_2}2 }   &&& {\bf\color{plot_2}0/2 } 
&&&& {\bf\color{plot_2}\frac 12/ \frac 22} &&&&&{\bf\color{plot_2}\frac 34 /\frac 24 }
\\ 
 {\bf \color{plot_3}Y_3}(\vx) &= \sum_{\vk\in[p-1]^4} c^{(3)}_{\vk} \vx_{\vk+\ve_3}  \vx_{\vk+\ve_4}
 && {\bf \color{plot_3} 2} 
  &&& {\bf\color{plot_3}0/2}  &&&& {\bf\color{plot_3} \frac 12 / \frac 22} &&&&&{\bf\color{plot_3}\frac 44 /\frac 24}
\\ 
  {\bf \color{plot_4}Y_4}(\vx) &= \sum_{\vk\in[p-1]^4} c^{(4)}_{\vk} 
  \vx_{\vk+\ve_3+\ve_4}
  \vx_{\vk+\ve_4}  \vx_{\vk}
  && {\bf \color{plot_4}3} 
   &&& {\bf\color{plot_4} 0/3}  &&&& {\bf\color{plot_4} \frac 1 2  / \frac 32}  &&&&&
   {\bf\color{plot_4} \frac 44 /\frac 34} 
\\ 
  {\bf \color{plot_5}Y_5^*}(\vx) &= \sum_{\vk\in[p-1]^4} c^{(5^*)}_{\vk} \vx_{\vk}  \vx_{\vk+\ve_4}  \vx_{\vk+2\ve_4}
  (\vx_{\vk}^2 - \vx_{\vk+\ve_4}^2)
   && {\bf \color{plot_5}5}    &&& {\bf\color{plot_5} 0/5}  &&&&{\bf\color{plot_5} \frac 12  / \frac 52 } &&&&&{\bf\color{plot_5} \frac 34 /\frac 54} 
\\ 
  {\bf \color{plot_8}Y_5}(\vx) &= \sum_{\vk\in[p-1]^4} c^{(5)}_{\vk} \vx_{\vk}  \vx_{\vk+\ve_1}
   && {\bf \color{plot_8}2}
   &&& {\bf\color{plot_8} 0/2}  &&&&{\bf\color{plot_8} \frac 2 2   / \frac 22}  &&&&&{\bf\color{plot_8}\frac 64 /\frac 24 } 
\\ 
  {\bf \color{plot_6}Y_6}(\vx) &= \sum_{\vk\in[p-1]^4} c^{(6)}_{\vk} \vx_{\vk}  \vx_{\vk+\ve_3}  ( 3\vx_{\vk - \ve_3}^2 - \vx_{\vk}^2)
   && {\bf \color{plot_6}4}
   &&& {\bf\color{plot_6}0/4}   &&&& {\bf\color{plot_6}\frac 1 2  / \frac 42 } &&&&&
   {\bf\color{plot_6}\frac 54 /\frac 44 } 
\\
  {\bf \color{plot_7}Y_7}(\vx) &= \sum_{\vk\in[p-1]^4} c^{(7)}_{\vk}   \vx_{\vk-\ve_3 + \ve_4} \vx_{\vk + \ve_2}(3\vx_{\vk}^2 - \vx_{\vk-\ve_3 + \ve_4}^2)
   && {\bf \color{plot_7}4}
    &&& {\bf\color{plot_7}0/4  }&&&& {\bf\color{plot_7}\frac 12  / \frac 42}  &&&&&{\bf\color{plot_7}\frac 64   / \frac 44}
\end{align*}
Each (eigen)function $Y_i$ is a linear combination of basis eigenfunctions of the same type, in the sense they have the same eigenvalue and the same spatial/frequency/learning indices.  
The coefficients $c_{\vk}^{(i)}$ are first sampled from standard Gaussians iid and then multiplied by an $i$-dependent constant so that $Y_i$ has unit norm\footnote{In our experiments, they are normalized over the test set.}. In the experiments, $p$ is chosen to be $4$ and the target is defined to be sum of them 
\begin{align}
    { Y} = \sum  Y_i
\end{align}
Since they are orthogonal to each other, $\|Y\|_2^2/8 = \|Y_i\|^2_2=1$, where the $L^2$-norm is taken over the uniform distribution on $(\spbar)^{p^3}$. 
In the experiment when we compare Flatten and GAP ( Fig.~\ref{fig:gap_vs_flatten}), we set all $c_{\vk}^{(i)}= c^{(i)}$ for some $c^{(i)}$, so that the functions are learnable by convolutional networks with a GAP readout layer. Note that we also remove ${\bf\color{plot_8}Y_5}$ from the target function $Y$ since ${\bf\color{plot_8}Y_5}$ is not in the function space defined by $\text{Conv}(p)^{\otimes 3}\text{-Flatten}$ or $\text{Conv}(p)^{\otimes 3}\text{-GAP}$. 

We compare three architectures.  Fig.~\ref{fig:nn-demonstration-2} Column (a) $\text{MLP}^{\otimes 4}$, a (four-layer) MLP, the most coarse architecture used in the paper.   
Fig.~\ref{fig:nn-demonstration-2} Column (b), $\text{CNN}(p^2)^{\otimes{2}}$, a``D"-CNN that contains two convolutional layers with filter size/stride equal to $p^2$.
Fig.~\ref{fig:nn-demonstration-2} Column (c), $\text{CNN}(p)^{\otimes{4}}$, a ``HR"-CNN, the finest architecture used in the experiments, that contains four convolutional layers with filter size/stride equal to $p$. 
In all experiments except the one in Fig.~\ref{fig:gap_vs_flatten}, we use Flatten as the readout layer for the convolutional networks and add a {\it Act-Dense} layer after Flatten to improve the expressivity of the function class. However, in the Flatten vs GAP experiments, Fig.~\ref{fig:gap_vs_flatten}, we have only one dense layer after GAP/Flatten and, in particular, no non-linear activation. 

We show how to compute the frequency index, the spatial index and the learning index through three examples: ${\bf \color{plot_2} Y_2}$ / ${\bf \color{plot_3} Y_3}$ / ${\bf \color{plot_5} Y_5^*}$, which have degree 2 / 2 / 5, resp. The indices of other (basis) eigenfunctions can be computed using the same approach. 
% Note that, both functions have degree / frequency equal to 2. 
% First, for MLP the DAG is shown in Fig.~\ref{fig:nn-demonstration-2} Column (c). 
We use {\bf Dashed Lines} to represent either an edge connecting an input node to a node in the first-hidden layer or an edge associated to a dense layer. In either case, the corresponding output node of the edge has degree $O(1)$, and thus, the weights (of the DAGs) of such edges are always 0. Only {\bf Solid Lines} are relevant in computing the {\it spatial index}. Since each $Y_i$ is a linear combination of basis eigenfunctions of the same type, we only need to compute the indices of one component. For convenience, we compute the $\vk=0$ component, of which the associated MST is highlighted with colored lines in each DAG. Recall that $p=4$, $d=p^4=256$ and $m_t=32\times 10240 \sim d^{2.28}$ (training set size), i.e. the budget index is roughly $r=2.28$. 

\begin{enumerate}
    \item[(a.)]  {\bf  MLP}{ Column (a) Fig.~\ref{fig:nn-demonstration-2}. }
The NTK and NNGP kernels are inner product kernels and the associated DAGs are linked lists. 
The corresponding DAG has only one input node whose dimension is equal to $d=p^4$. The spatial index is always 0 since the degree of each hidden node is 1 (since $1=d^0$) and the frequency index is equal to the degree of the eigenfunctions. Thus $\learning({\bf \color{plot_2} Y_2}) =\learning({\bf \color{plot_3} Y_3})=2$ and $\learning({\bf \color{plot_5} Y_5^*})=5$. Changing the number of layers won't change the learning indices. In sum, learning ${\bf \color{plot_2} Y_2}$ / ${\bf \color{plot_3} Y_3}$ / ${\bf \color{plot_5} Y_5^*}$ using infinite-width MLP requires $~d^{\bf \color{plot_2}2^+}/d^{\bf \color{plot_3}2^+}/d^{\bf \color{plot_5}5^+}$ many samples /SGD steps.  Clearly, ${\bf \color{plot_5} Y_5^*}$ is completely unlearnable as $r=2.28<<5$.
In the MSE plot ${\bf \color{plot_5} Y_5^*}$ (5-th row in Fig.~\ref{fig:nn-demonstration-2}), the {\bf \color{plot_3} Red Lines} does not make any progress. 

\item[(b.)] ${ \text{\bf CNN}(p^2)^{\otimes 2}}$  Column (b) Fig.~\ref{fig:nn-demonstration-2}. The input image is partitioned into $p^2$ patches and each patch has dimension $p^2$. 
The second layer of the DAG has $p^2$ many nodes, each node represents one pixel (with many channels) in the first hidden layer of a finite-width ConvNet. After one more convolutional layer with filter size/stride $p^2$, the number of node (pixel) is reduced to one. The remaining part of the DAG is essentially a linked list ({\bf Dashed Line}) with length equal to 1, which corresponds to the {\it Act-Dense} layer. 
The frequency index $\frequency({\bf \color{plot_2} Y_2})= 2  \frac 1 2 =1$. This is because the degree of ${\bf \color{plot_2} Y_2}$ is 2 and the input dimension of a node is $p^2 = d^{1/2}$. The spatial index is equal to $1/2$, since the minimum tree containing $\vx_{\vk}  \vx_{\vk+\ve_4}$ has only one non-zero edge ({\bf \color{plot_2}Solid Lines}) whose weight is equal to 1/2 (since the degree of the output node is $p^2=d^{1/2}$); see the {\bf \color{plot_2}colored paths} in Fig.~\ref{fig:nn-demonstration-2} Column (b). Therefore the learning index of $\learning({\bf \color{plot_2} Y_2}) = 1 + 1/2 = 3/2$. Similarly $\learning({\bf \color{plot_3} Y_3}) =1 + 1/2= 3/2$, as the term $\vx_{\vk+\ve_3}  \vx_{\vk+\ve_4}$ are lying in the same patch of size $p^2$ for all $\vk$, and $\learning({\bf \color{plot_5} Y_5^*})= 5/2 + 1/2=3$.
In sum, learning ${\bf \color{plot_2} Y_2}$ / ${\bf \color{plot_3} Y_3}$ / ${\bf \color{plot_5} Y_5^*}$ using infinite-width $\text{CNN}(p^2)^{\otimes 2}$ requires $~d^{\color{plot_2}1.5^+}/d^{\color{plot_3}1.5^+}/d^{\color{plot_5} 3^+}$ many samples / SGD steps. 
While neither infinite-width $\text{CNN}(p^2)^{\otimes 2}$ nor $\text{MLP}^{\otimes 4}$ distinguishes ${\bf \color{plot_2} Y_2}$ from ${\bf \color{plot_3} Y_3}$, $\text{CNN}(p^2)^{\otimes 2}$ does improve the learning efficiency for both of them: $d^{2^+}\to d^{1.5^+}$. Note that ${\bf \color{plot_5} Y_5^*}$ is still unlearnable as $r=2.28<3=\learning({\bf \color{plot_5} Y_5^*})$. In the MSE plot ${\bf \color{plot_5} Y_5^*}$ (5-th row in Fig.~\ref{fig:nn-demonstration-2}), the {\bf \color{plot_2} Orange Line} does not make any progress. 
This is also the case for finite-width network trained by SGD; see second row in Fig.~\ref{fig:learning-eigen-modes}. 
\item[(c.)]
$\text{\bf CNN}(p)^{\otimes 4}$  Column (c) Fig.~\ref{fig:nn-demonstration-2}.
The input image is partitioned into $p^3$ patches and each patch has dimension $p$. 
The second/third/fourth/output layer of the DAG has $p^3/p^2/p/1$ many nodes. 
The frequency indices are: $\frequency({\bf \color{plot_2} Y_2})= \frequency({\bf \color{plot_3} Y_3}) = 2  \frac 1 4 =1/2$ and $\frequency({\bf \color{plot_5} Y_5^*}) = 5 \frac 14 =5/4$. This is because the size of input nodes is reduced to $p = d^{1/4}$. Unlike the above cases, the spatial indices become different. The two interacting terms in $\vx_{\vk}  \vx_{\vk+\ve_4}$ and the three interacting terms in $\vx_{\vk+\ve_4}  \vx_{\vk+2\ve_4}(\vx_{\vk}^2 - \vx_{\vk+\ve_4}^2)$ are in the same input node
  while the two interacting terms in  $\vx_{\vk+\ve_3} \vx_{\vk+\ve_4}$ and 
are in two different input nodes. As a consequence, the minimum spanning tree (MST) that contains $\vx_{\vk} $ and $\vx_{\vk+\ve_4}$ and the one contains $\vx_{\vk}$,  $\vx_{\vk+\ve_4}$ and $ \vx_{\vk+2\ve_4}$ are the same. They have {\bf \color{plot_2} 3 solid lines}. However, the MST containing $\vx_{\vk+\ve_3} $ and $ \vx_{\vk+\ve_4}$ has {\bf \color{plot_3} 4 solid lines}. Therefore $\spatial({\bf \color{plot_2} Y_2})=\spatial({\bf \color{plot_5} Y_5^*}) = 3 \times \frac 1 4$ and 
$\spatial({\bf \color{plot_3} Y_3}) = 4 \times \frac 1 4$. As such,  
$\learning({\bf \color{plot_2} Y_2}) =  \frac 5 4$, 
$\learning({\bf \color{plot_3} Y_3}) =  \frac 6 4$ and 
$\learning({\bf \color{plot_5} Y_5^*}) =  \frac 8 4$.
In sum, learning ${\bf \color{plot_2} Y_2}$ / ${\bf \color{plot_3} Y_3}$ / ${\bf \color{plot_5} Y_5^*}$ using infinite-width $\text{CNN}(p)^{\otimes 4}$ requires $~d^{ {\bf \color{plot_2} 1.25^+}}$ / $~d^{ {\bf \color{plot_3} 1.5^+}}$ / $~d^{ {\bf \color{plot_5} 2^+}}$ many samples / SGD steps, resp. 
Now $\learning({\bf \color{plot_5} Y_5^*})=2.<2.28$
and in the MSE plot ${\bf \color{plot_5} Y_5^*}$ (5-th row in Fig.~\ref{fig:nn-demonstration-2}), the {\bf \color{plot_1} Blue Line} does make significant progress (the MSE is reduced from $\sim 0.5$ to $\sim 0.2$.) Finite-width network trained by SGD does even better: the test MSE is almost zero; see third row in Fig.~\ref{fig:learning-eigen-modes}. 

\end{enumerate}

\begin{figure}[h]
\centering
    \begin{subfigure}[b]{0.1\textwidth}
     \includegraphics[width=\textwidth]{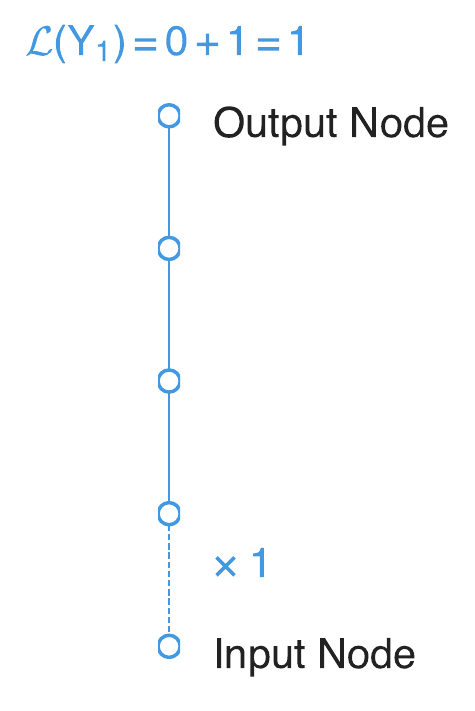}
    \end{subfigure}
    \begin{subfigure}[b]{0.3\textwidth}
     \includegraphics[width=\textwidth]{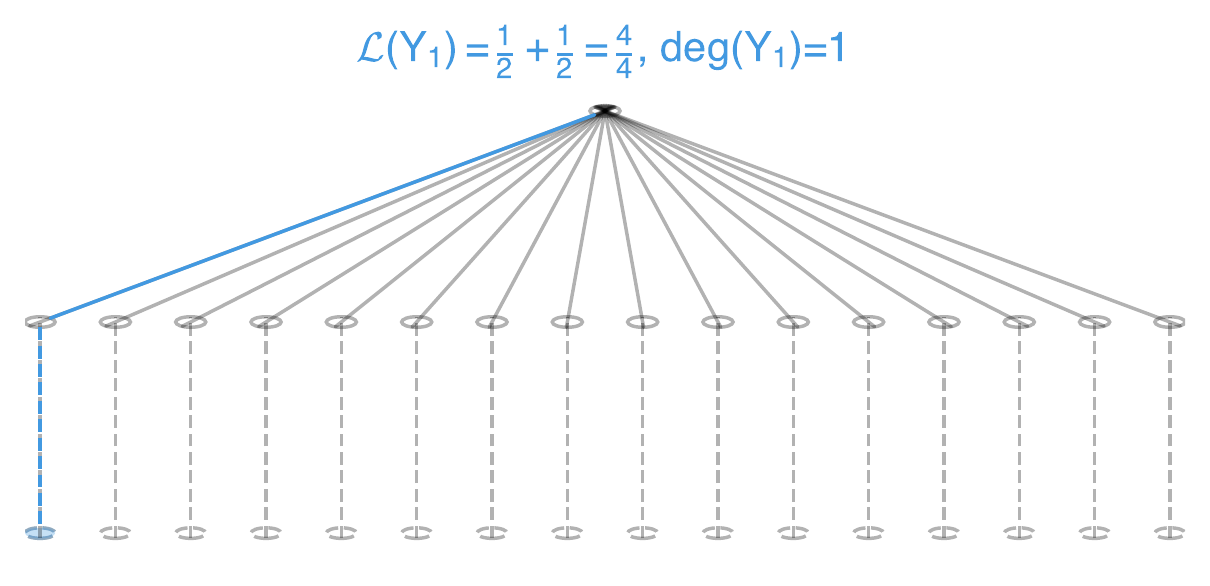}
         \end{subfigure}
    \begin{subfigure}[b]{0.3\textwidth}
    \includegraphics[width=\textwidth]{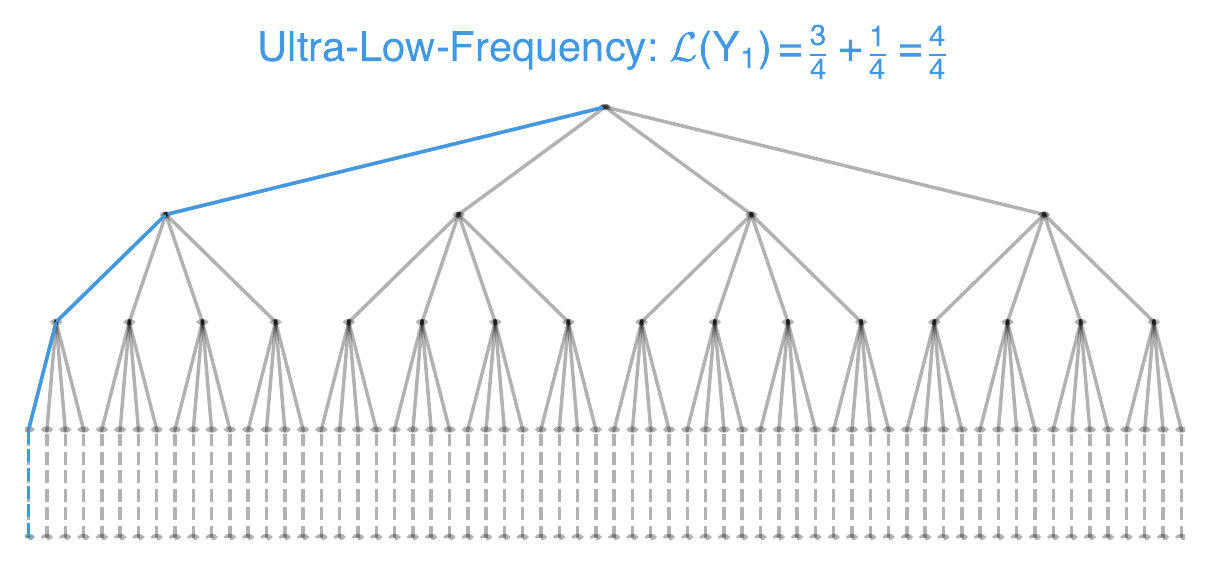}
    \end{subfigure}
        \begin{subfigure}[b]{0.25\textwidth}
    \includegraphics[width=\textwidth]{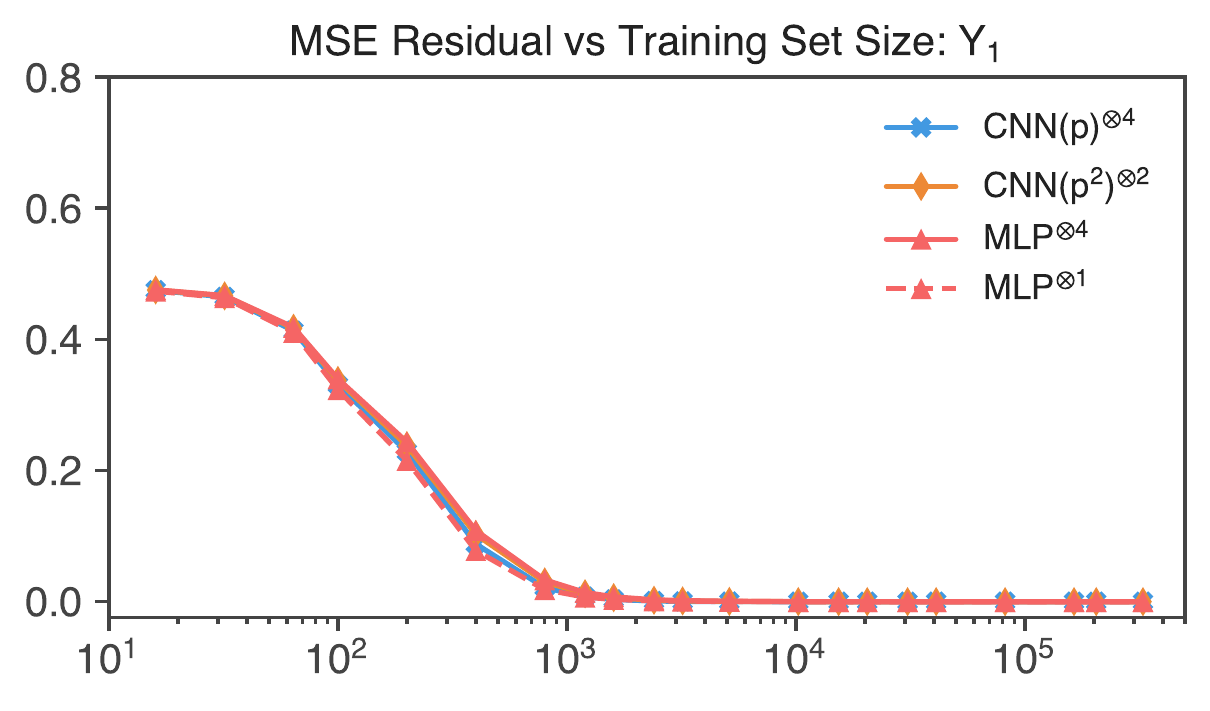}
    \end{subfigure}
    \\
     \begin{subfigure}[b]{0.1\textwidth}
     \includegraphics[width=\textwidth]{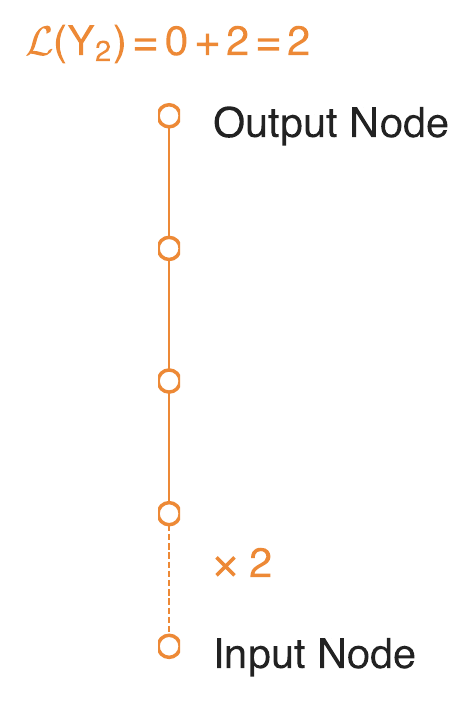}
    \end{subfigure}
    \begin{subfigure}[b]{0.3\textwidth}
     \includegraphics[width=\textwidth]{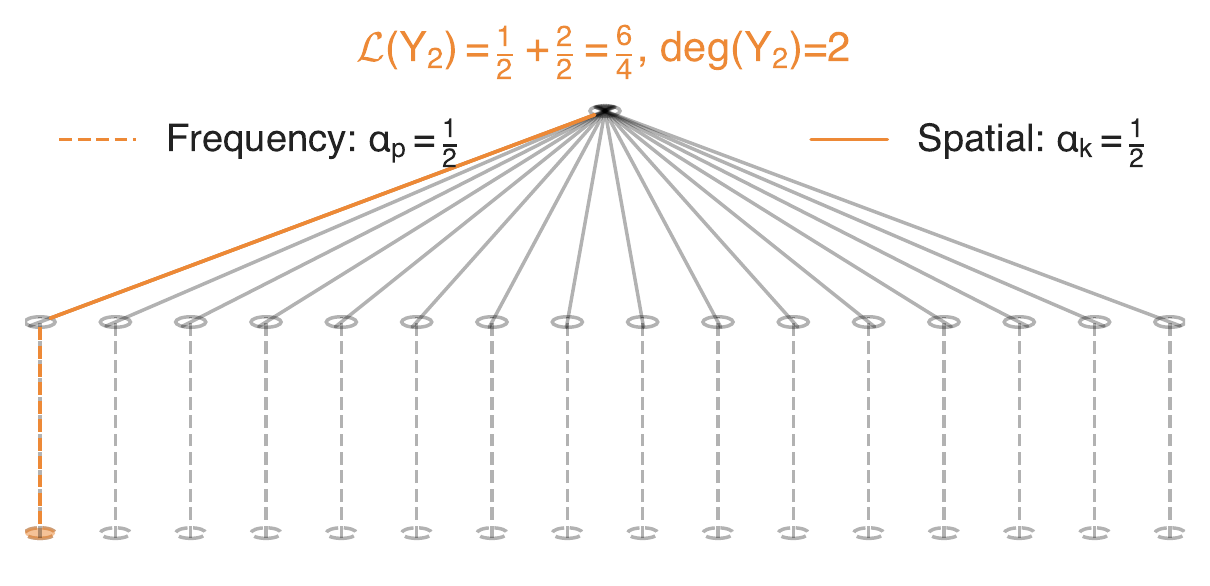}
         \end{subfigure}
    \begin{subfigure}[b]{0.3\textwidth}
    \includegraphics[width=\textwidth]{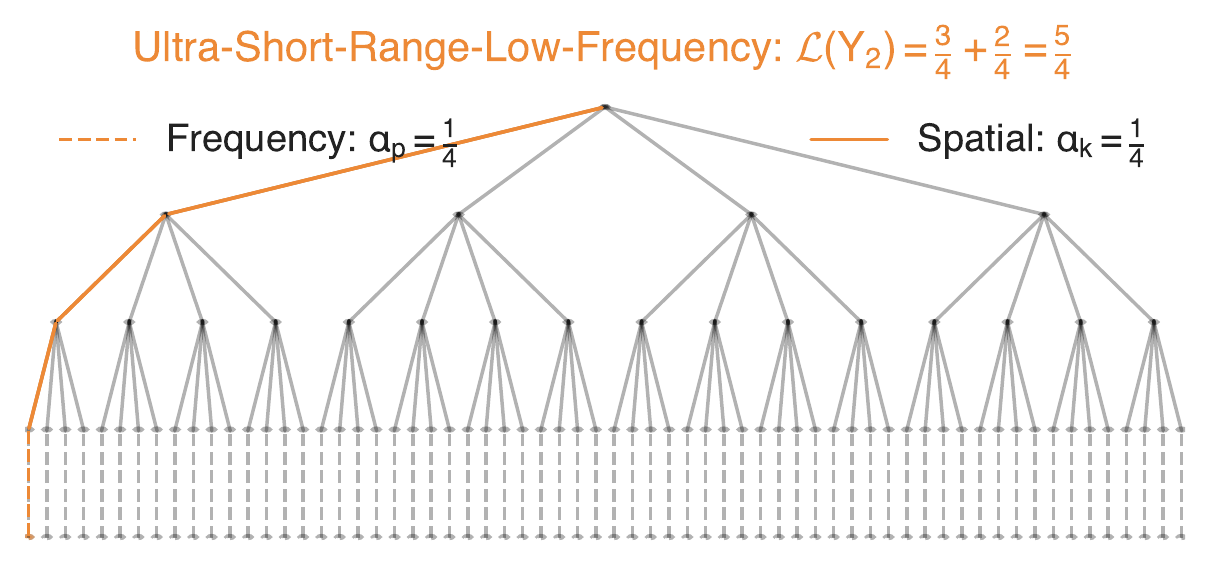}
    \end{subfigure}
        \begin{subfigure}[b]{0.25\textwidth}
    \includegraphics[width=\textwidth]{figures/Oct1-nn-graphs/NTK-Residual-1.pdf}
    \end{subfigure}
    \\
      \begin{subfigure}[b]{0.1\textwidth}
     \includegraphics[width=\textwidth]{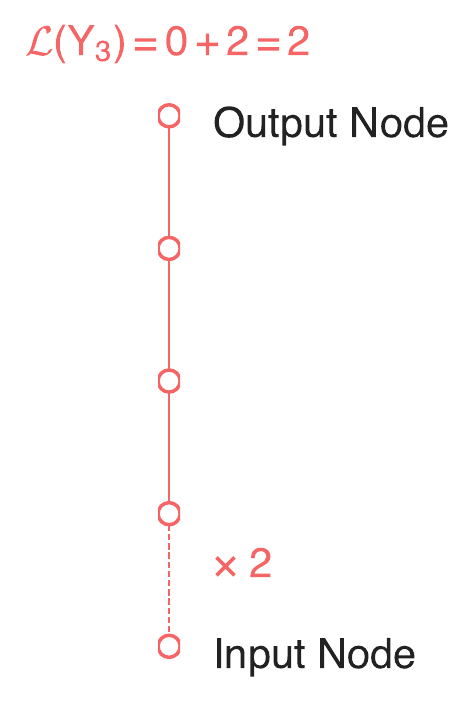}
    \end{subfigure}
    \begin{subfigure}[b]{0.3\textwidth}
     \includegraphics[width=\textwidth]{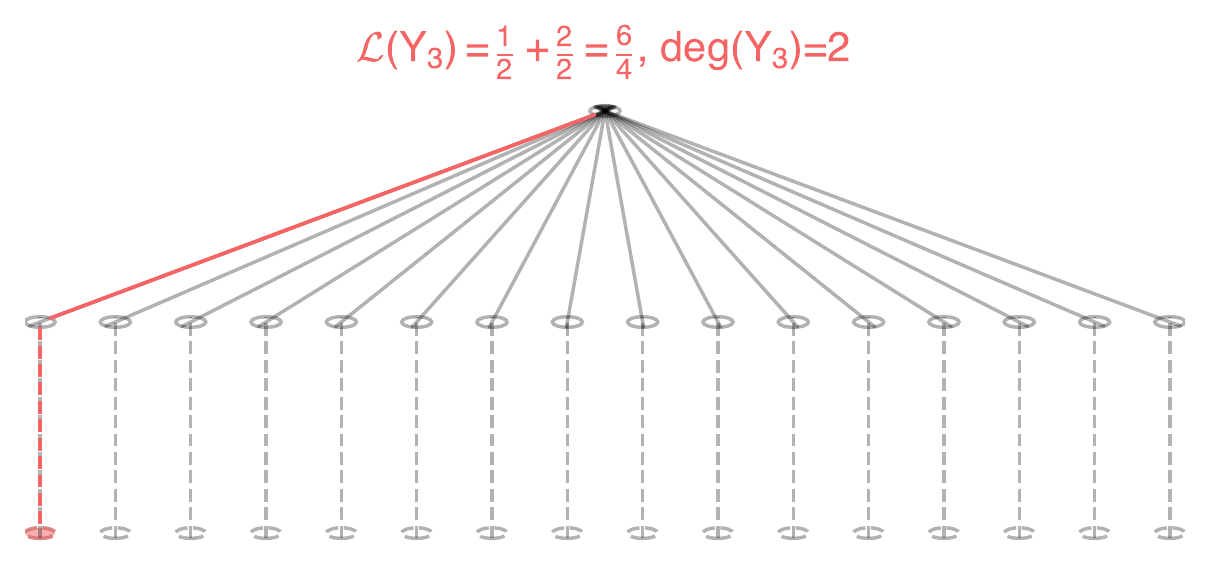}
         \end{subfigure}
    \begin{subfigure}[b]{0.3\textwidth}
    \includegraphics[width=\textwidth]{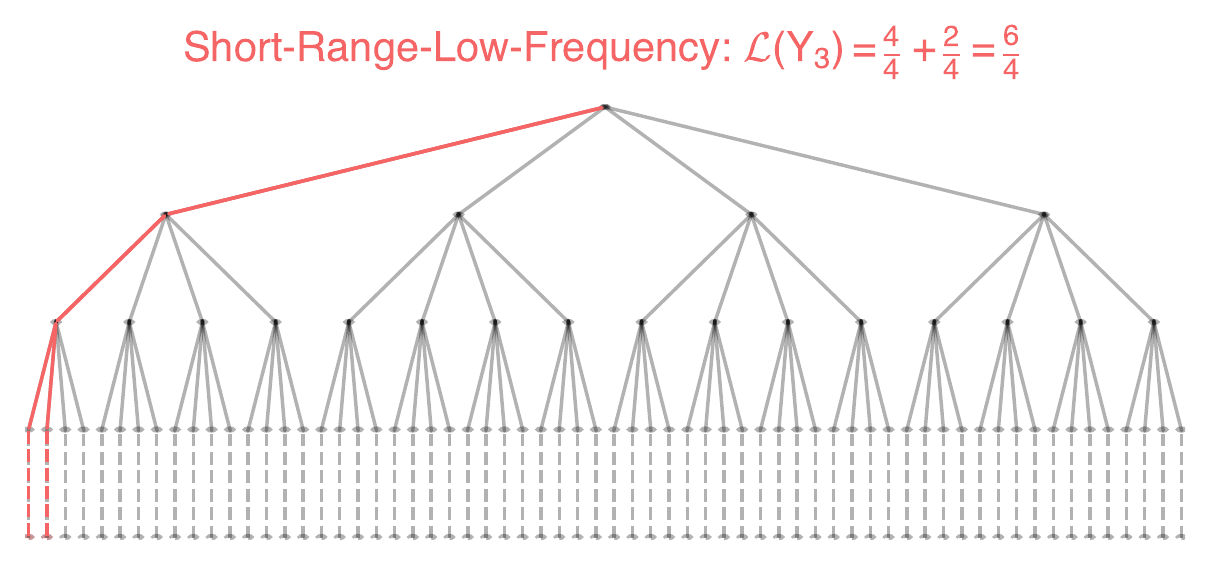}
    \end{subfigure}
        \begin{subfigure}[b]{0.25\textwidth}
    \includegraphics[width=\textwidth]{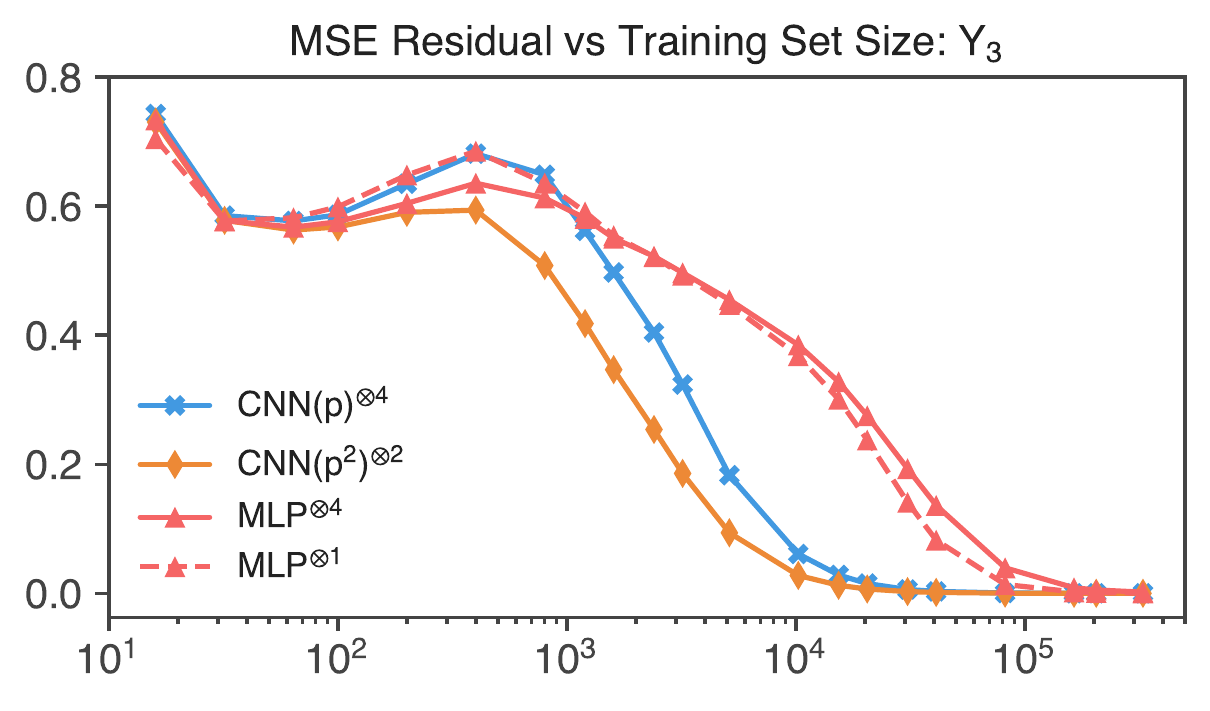}
    \end{subfigure}
    \\
      \begin{subfigure}[b]{0.1\textwidth}
     \includegraphics[width=\textwidth]{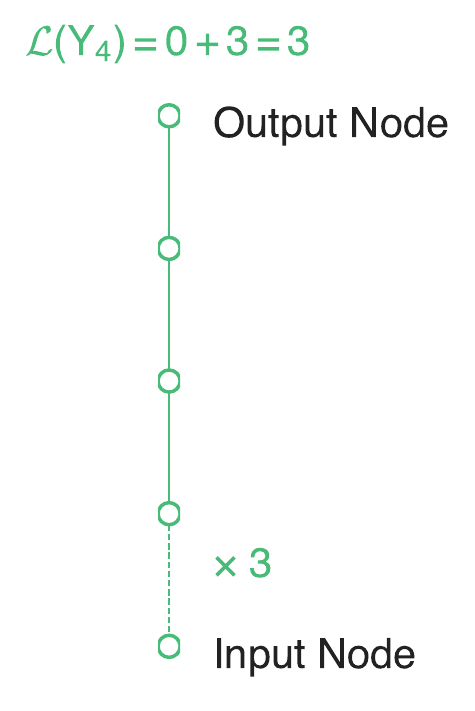}
    \end{subfigure}
    \begin{subfigure}[b]{0.3\textwidth}
     \includegraphics[width=\textwidth]{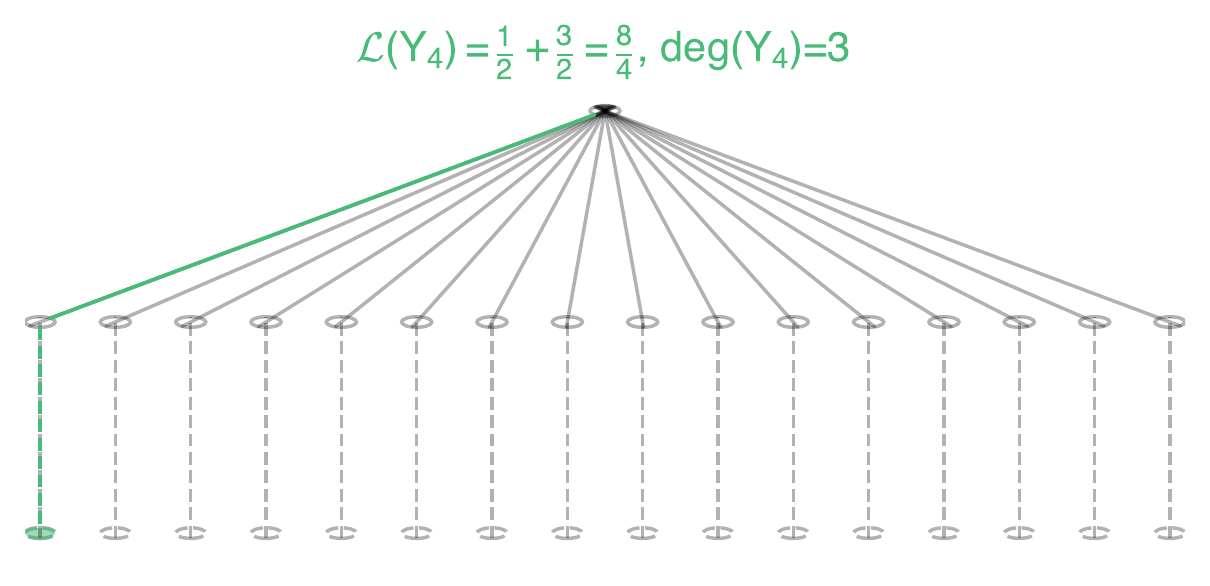}
         \end{subfigure}
    \begin{subfigure}[b]{0.3\textwidth}
    \includegraphics[width=\textwidth]{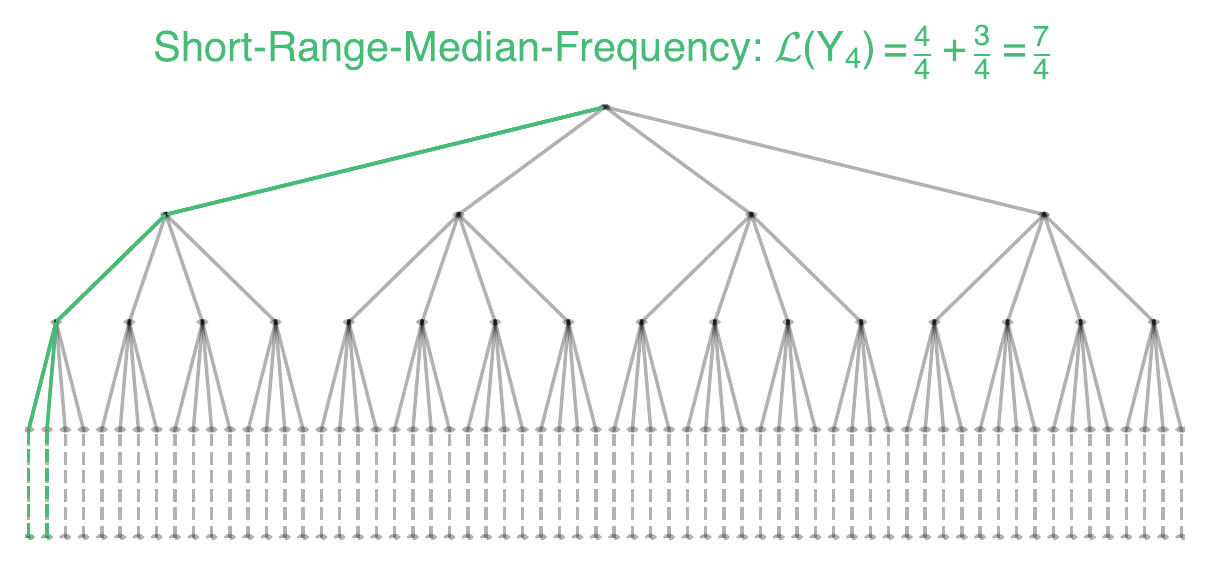}
    \end{subfigure}
        \begin{subfigure}[b]{0.25\textwidth}
    \includegraphics[width=\textwidth]{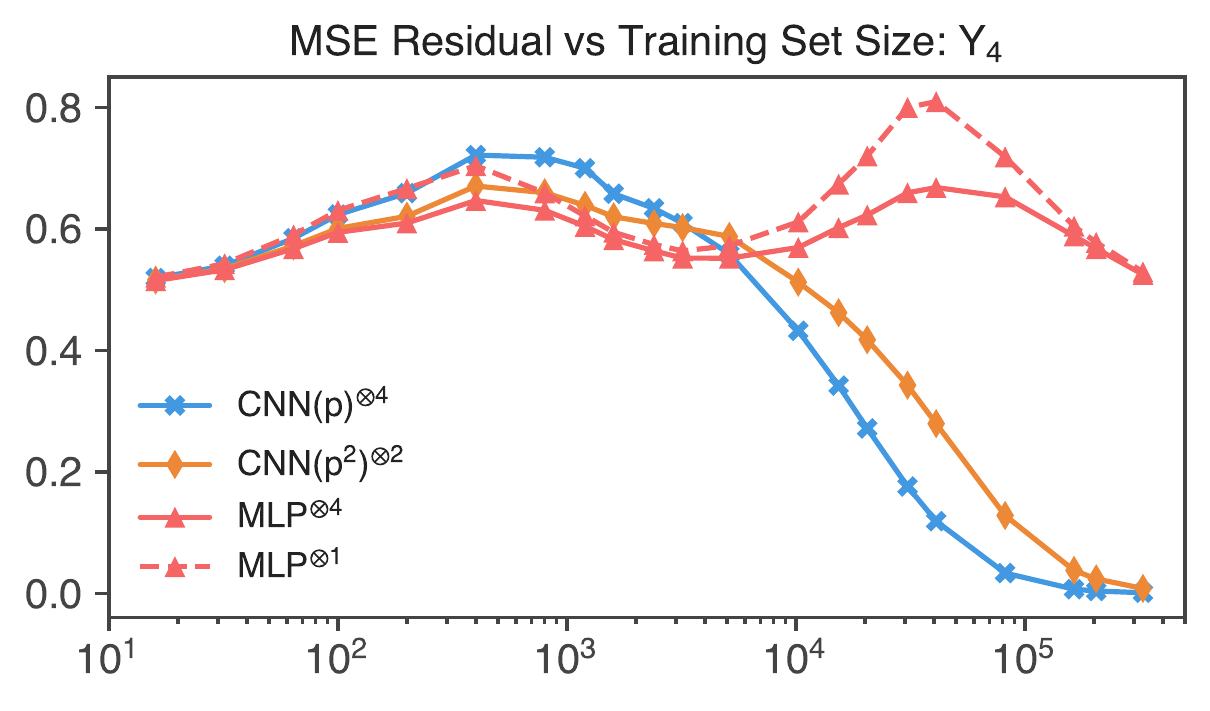}
    \end{subfigure}
    \\
      \begin{subfigure}[b]{0.1\textwidth}
     \includegraphics[width=\textwidth]{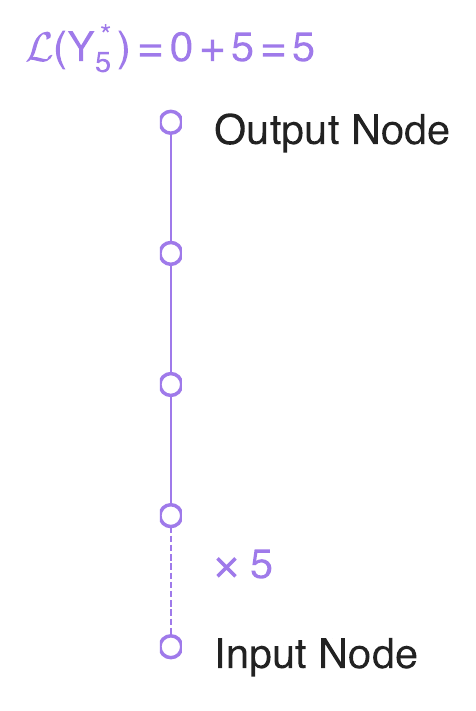}
    \end{subfigure}
    \begin{subfigure}[b]{0.3\textwidth}
     \includegraphics[width=\textwidth]{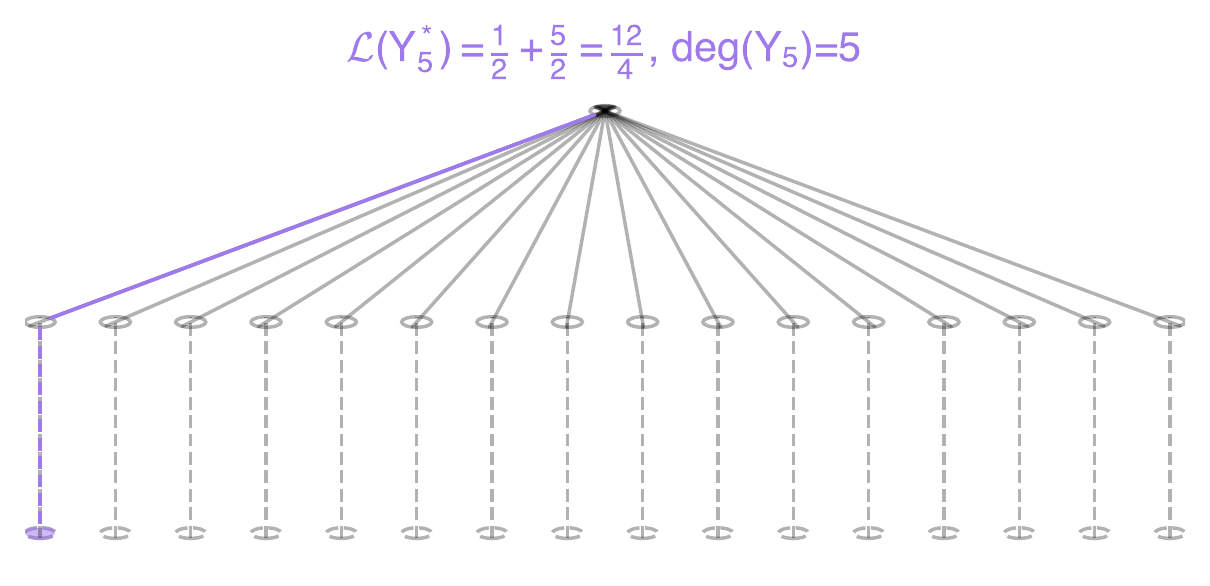}
         \end{subfigure}
    \begin{subfigure}[b]{0.3\textwidth}
    \includegraphics[width=\textwidth]{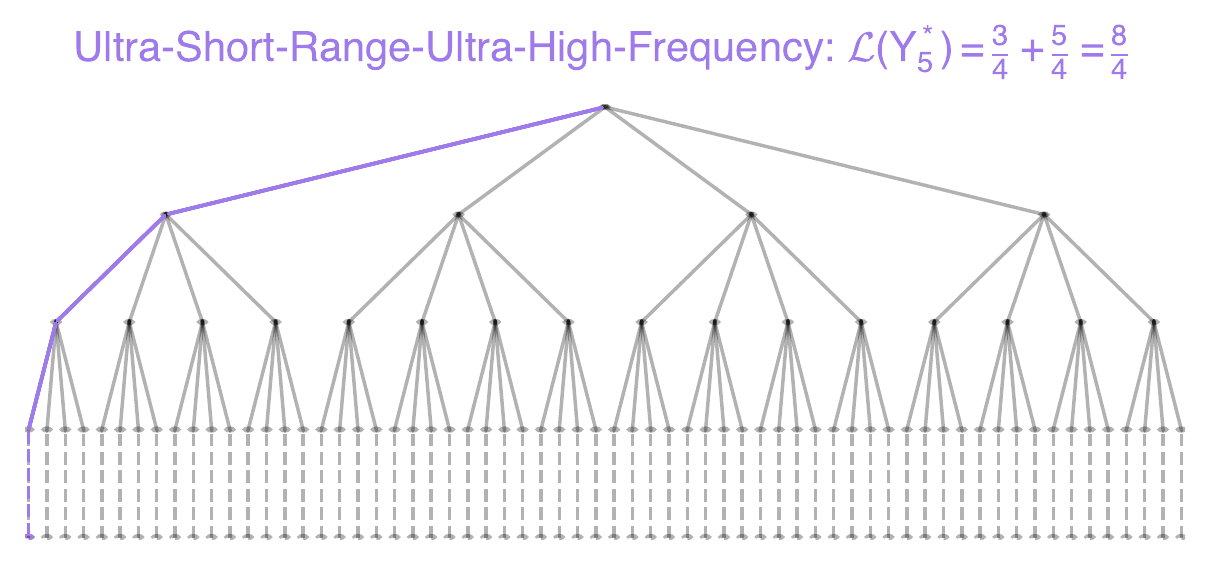}
    \end{subfigure}
        \begin{subfigure}[b]{0.25\textwidth}
    \includegraphics[width=\textwidth]{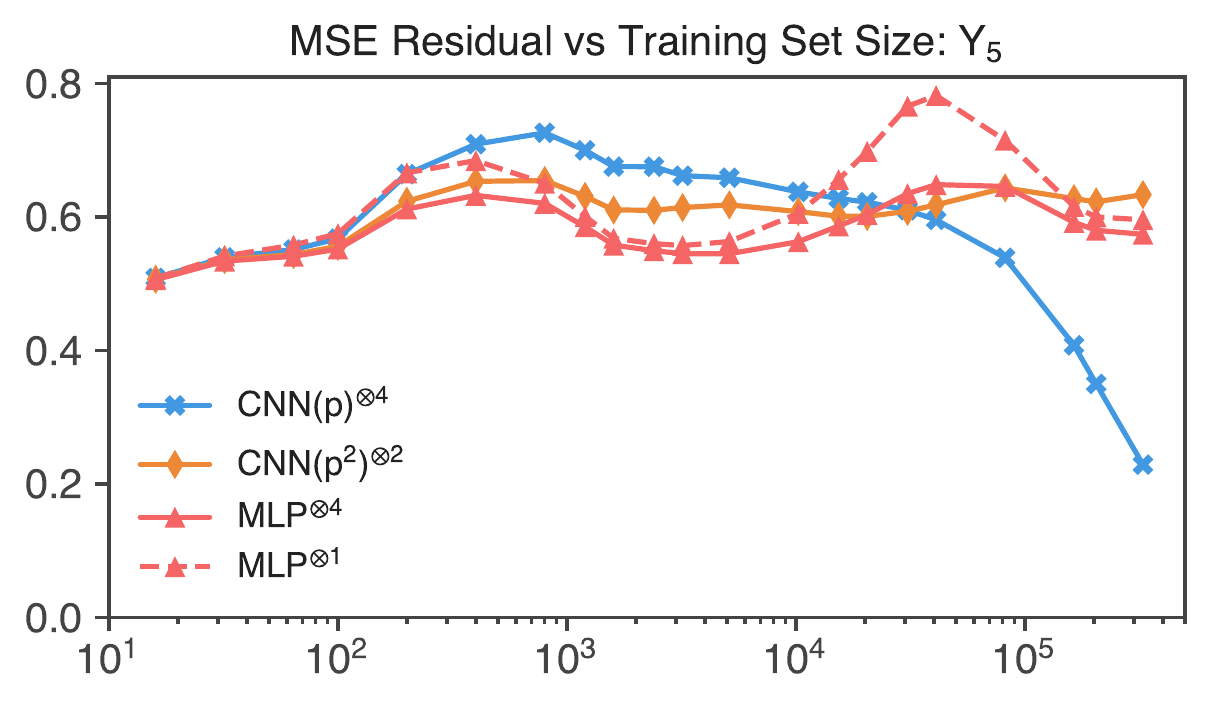}
    \end{subfigure}
    \\
      \begin{subfigure}[b]{0.1\textwidth}
     \includegraphics[width=\textwidth]{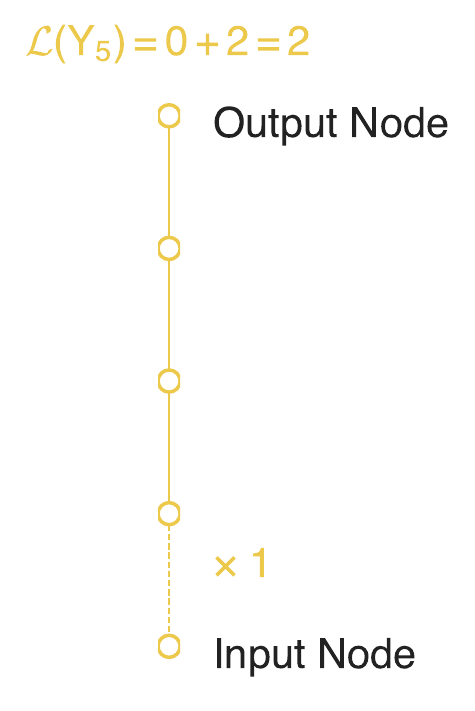}
    \end{subfigure}
    \begin{subfigure}[b]{0.3\textwidth}
     \includegraphics[width=\textwidth]{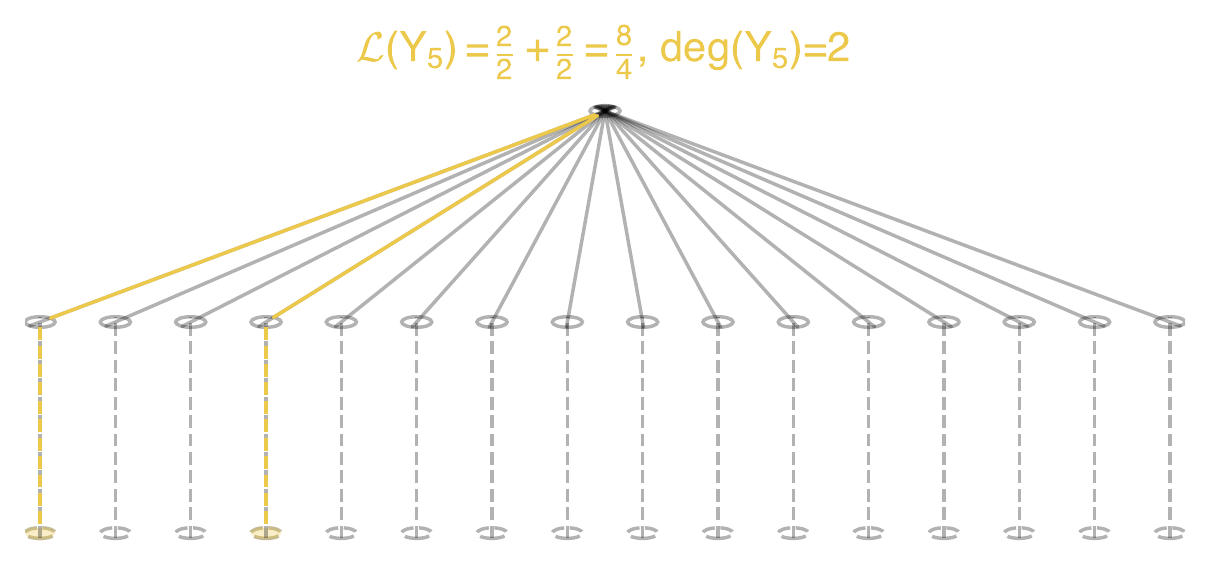}
         \end{subfigure}
    \begin{subfigure}[b]{0.3\textwidth}
    \includegraphics[width=\textwidth]{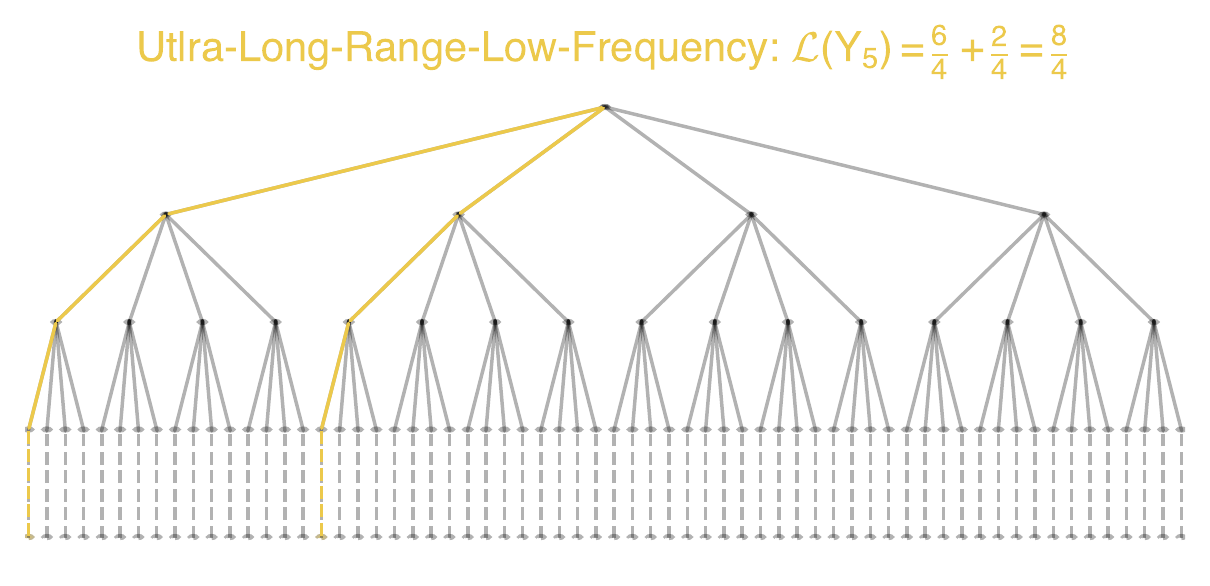}
    \end{subfigure}
        \begin{subfigure}[b]{0.25\textwidth}
    \includegraphics[width=\textwidth]{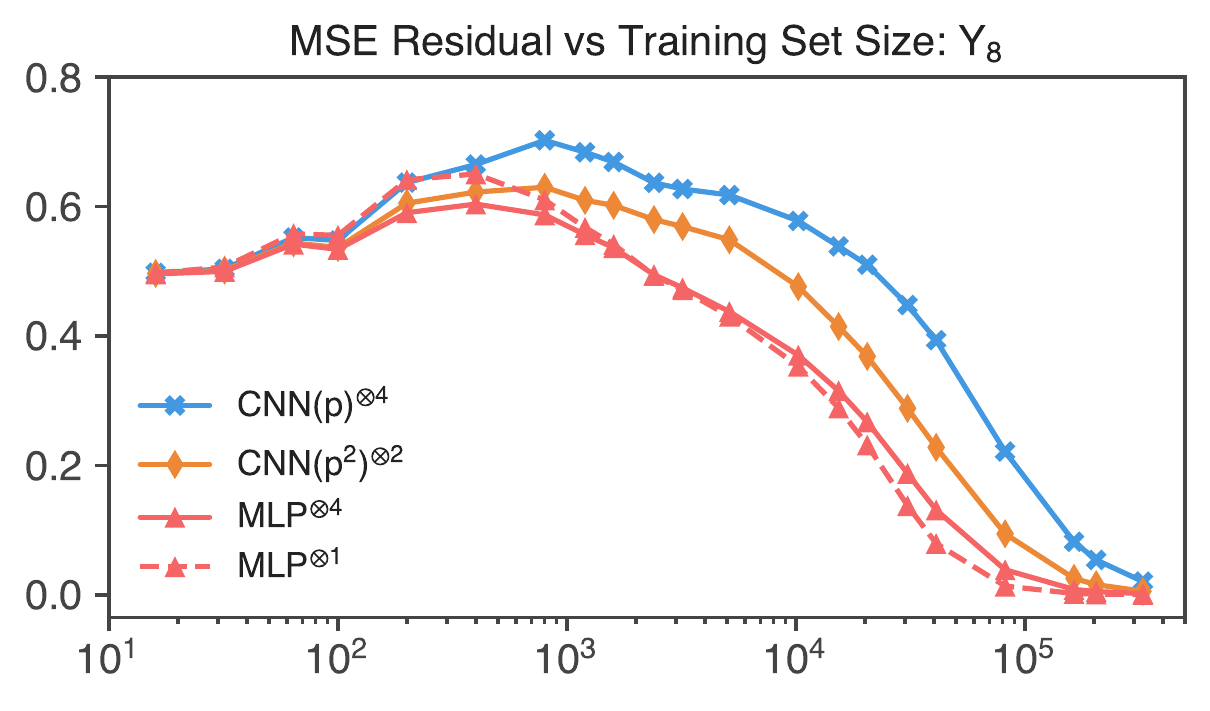}
    \end{subfigure}
    \\
      \begin{subfigure}[b]{0.1\textwidth}
     \includegraphics[width=\textwidth]{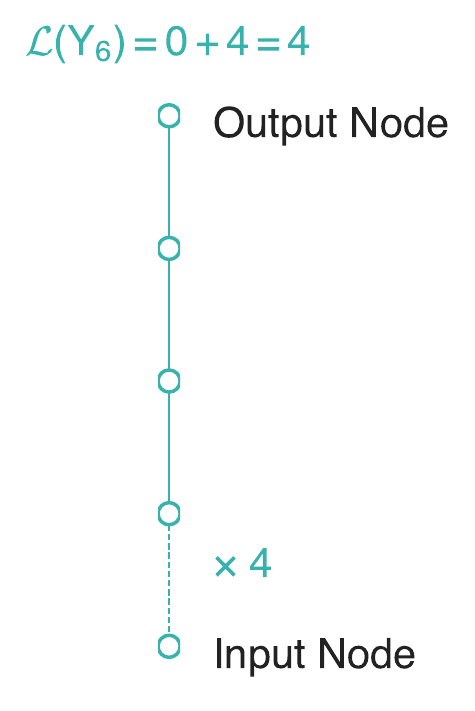}
    \end{subfigure}
    \begin{subfigure}[b]{0.3\textwidth}
     \includegraphics[width=\textwidth]{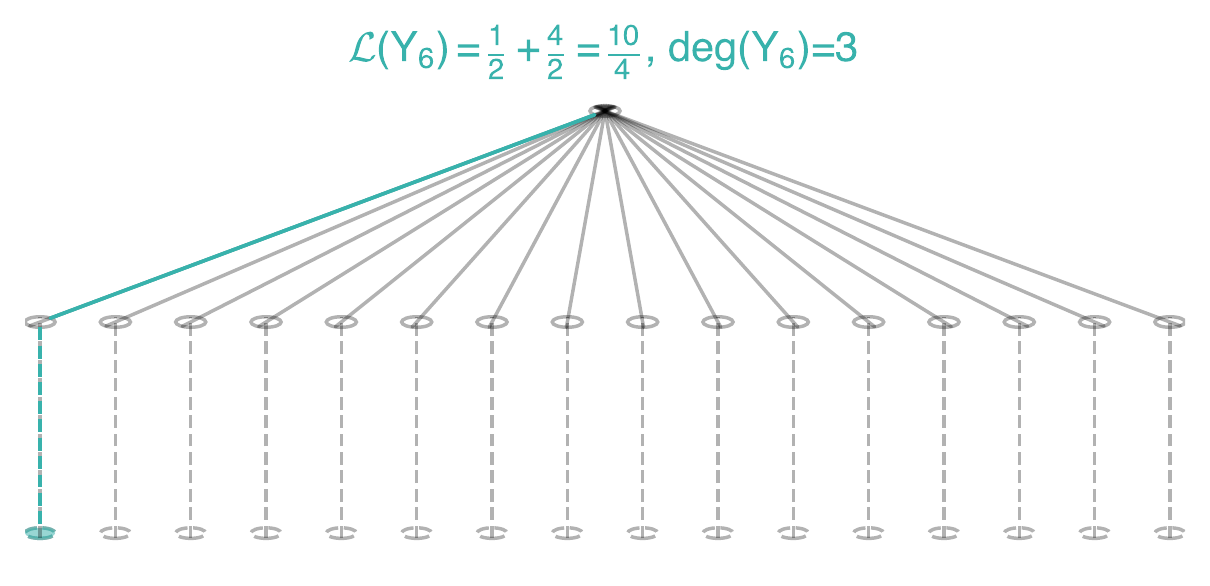}
         \end{subfigure}
    \begin{subfigure}[b]{0.3\textwidth}
    \includegraphics[width=\textwidth]{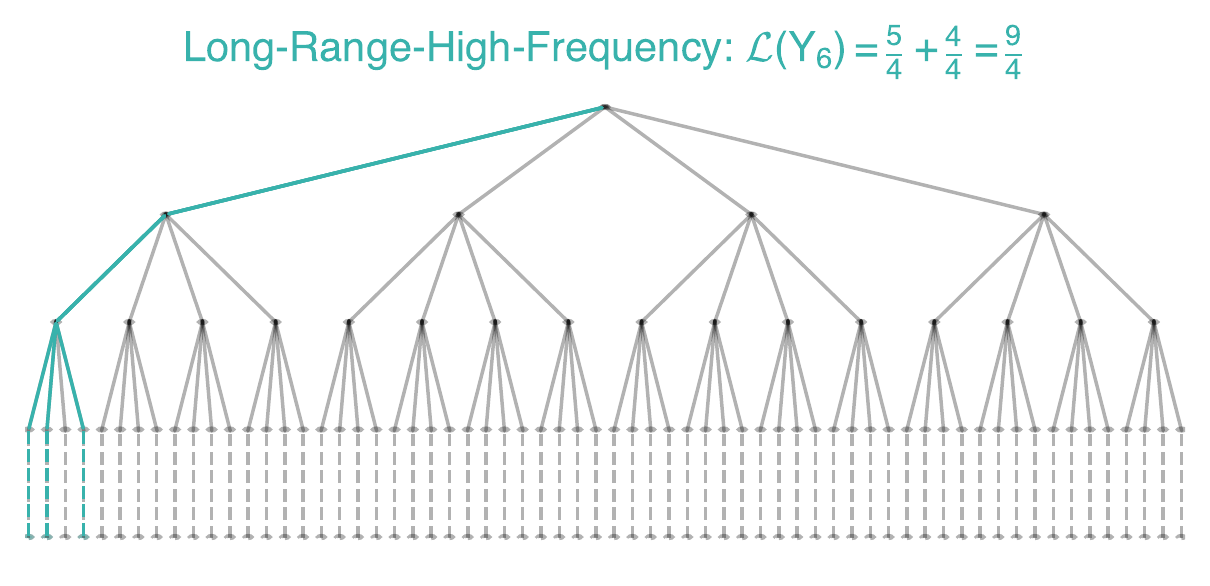}
    \end{subfigure}
        \begin{subfigure}[b]{0.25\textwidth}
    \includegraphics[width=\textwidth]{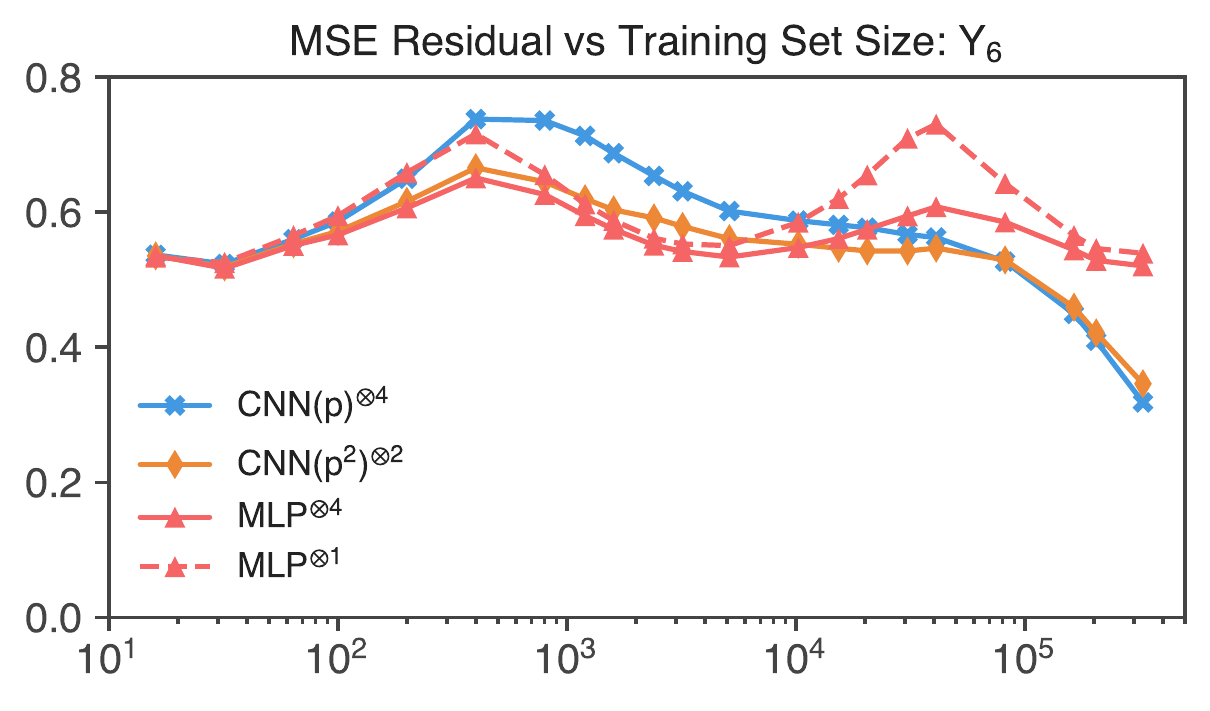}
    \end{subfigure}
    \\
      \begin{subfigure}[b]{0.1\textwidth}
     \includegraphics[width=\textwidth]{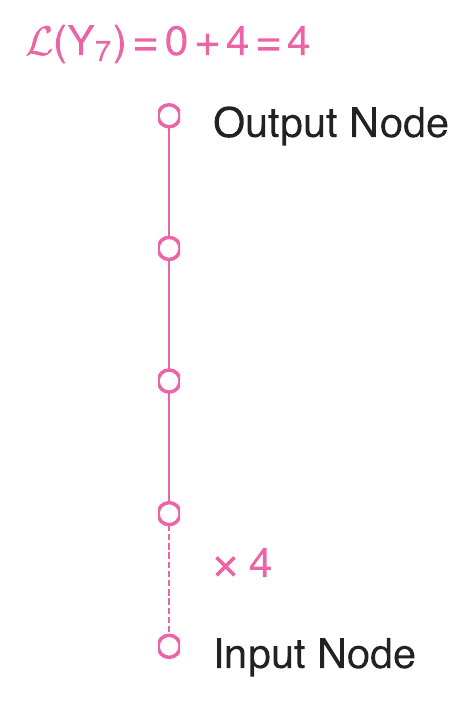}
           \caption{$\text{MLP}$}
    \end{subfigure}
    \begin{subfigure}[b]{0.3\textwidth}
     \includegraphics[width=\textwidth]{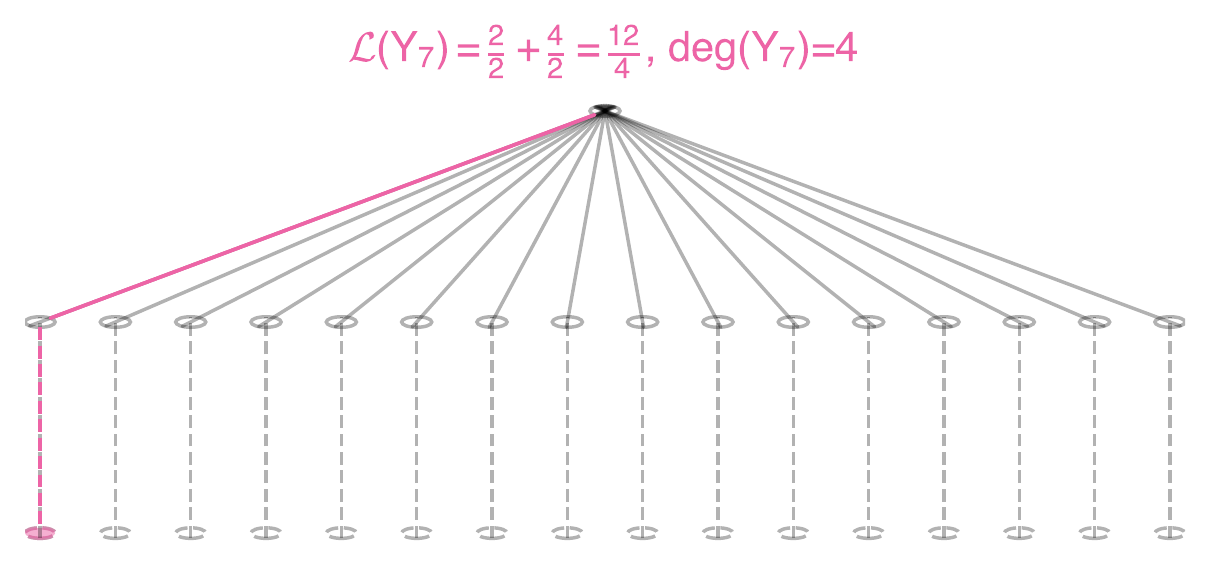}
      \subcaption{$\text{CNN}(p^2)^{\otimes{2}}$}
         \end{subfigure}
    \begin{subfigure}[b]{0.3\textwidth}
    \includegraphics[width=\textwidth]{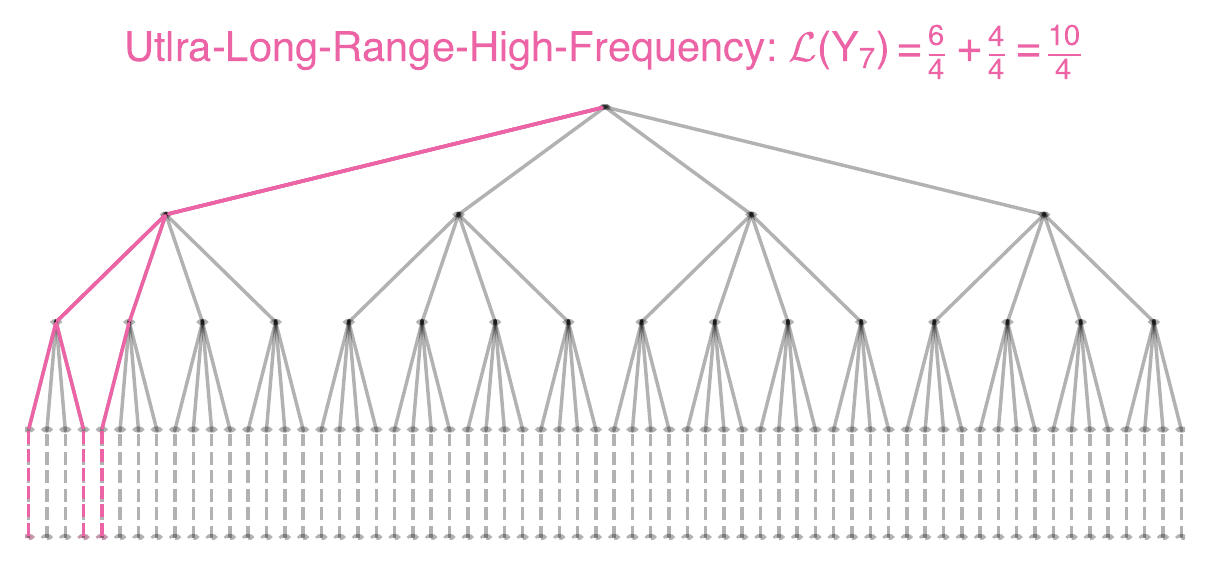}
     \subcaption{$\text{CNN}(p)^{\otimes{4}}$}
    \end{subfigure}
        \begin{subfigure}[b]{0.25\textwidth}
    \includegraphics[width=\textwidth]{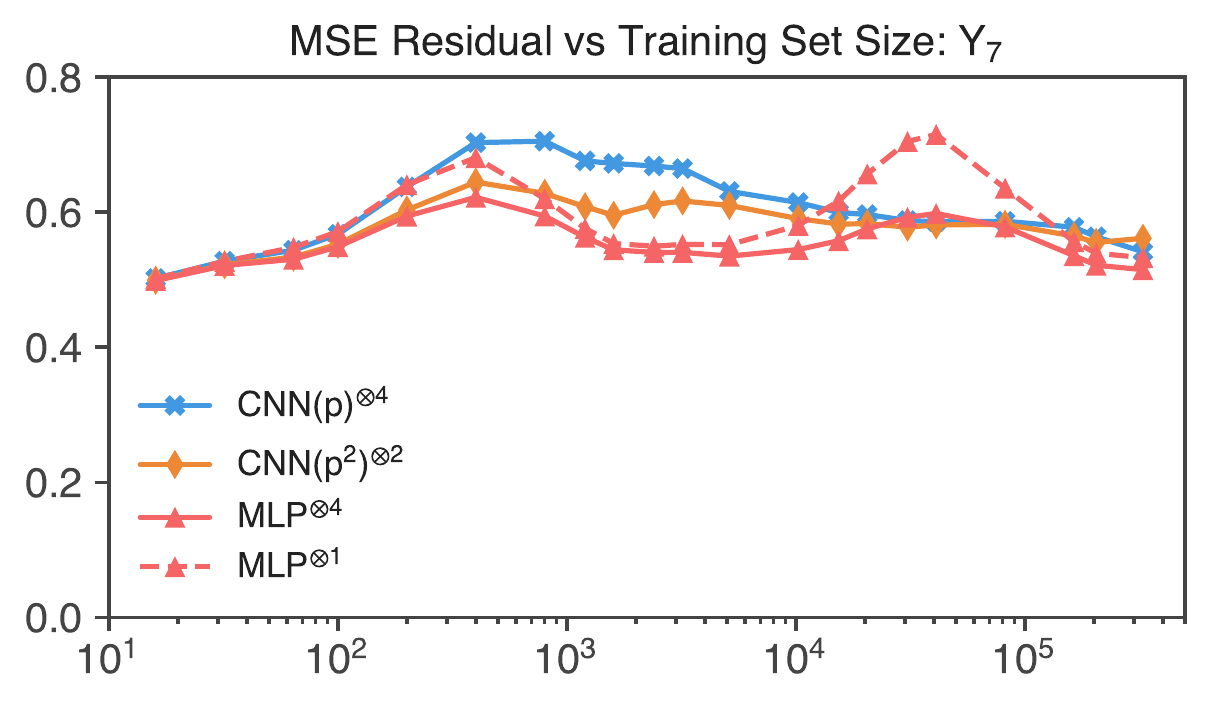}
     \caption{MSE Residual}
    \end{subfigure}
    \caption{{\bf Eigenfunction vs Learning Index vs Architecture/DAG.}
    Rows: eigenfunctions $Y_i$ with various space-frequency combinations. Columns: 
     DAGs associated to (a) $\text{CNN}(p)^{\otimes{4}}$, a ``HR"-CNN. (b) $\text{CNN}(p^2)^{\otimes{2}}$, a``D"-CNN. (c) a four-layer MLP. 
    Column (d) is the MSE of the residual of the corresponding eigenfunction obtained by NTK-regression. 
    The dashed lines in each DAG correspond to the mapping from the input layer to the first hidden layer and the associated weights to the DAG is ZERO. 
    Each colored path in each DAG corresponds to the minimum spanning tree that contains all interaction terms of the corresponding eigenfunctions.    
    }
    \label{fig:nn-demonstration-2}
\end{figure} 

\newpage

\section{Figure Zoo}

\subsection{MLPs: Depth is not equal to Hierarchy}
We compare a one hidden layer MLP and a four hidden layer MLP in Fig.~\ref{fig:mlp-2}. Unlike CNNs, increasing the number of layers does not improve the performance of MLPs much for both NTK regression and SGD. This is consistent with a theoretical result from \citet{bietti2020deep}, which says the NTKs of Relu MLPs are essentially the same for any depth. 

% \begin{figure}[t]
%     \centering
%     \includegraphics[width=\textwidth]{figures/Oct4/dense1.pdf}
%     \\
%     \includegraphics[width=\textwidth]{figures/Oct4/dense4.pdf}
%     \caption{{\bf MLPs do not benefits from having more layers.} We plot the learning/training dynamics of each eigenfunction $Y_i$. Top: a one hidden layer MLP and bottom: a four hidden layer MLP. From left to right: residual MSE (per eigenfunction) of NNGP/NTK regression, test/training MSE of SGD. The learning indices of $Y_i$ in each architecture is shown in the legends. 
%     Clearly, MLPs fail to learn high frequency ($\deg \geq 3$) interactions. Having more layers does not improve performance for both kernel machines and SGD training, except, possibly, prevents severe overfitting.} 
%     \label{fig:mlp-2}
% \end{figure}

\begin{figure}[h]
    \centering
    \includegraphics[width=\textwidth]{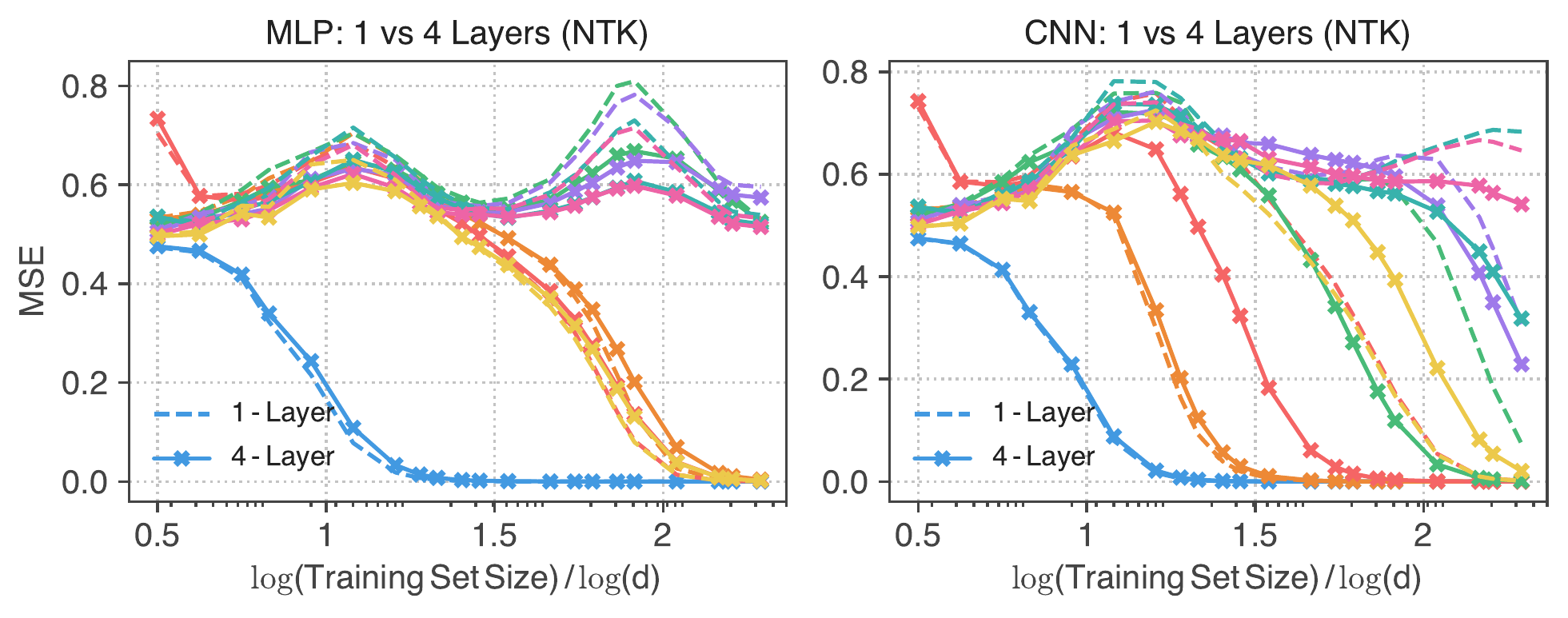}
    \\
    \includegraphics[width=\textwidth]{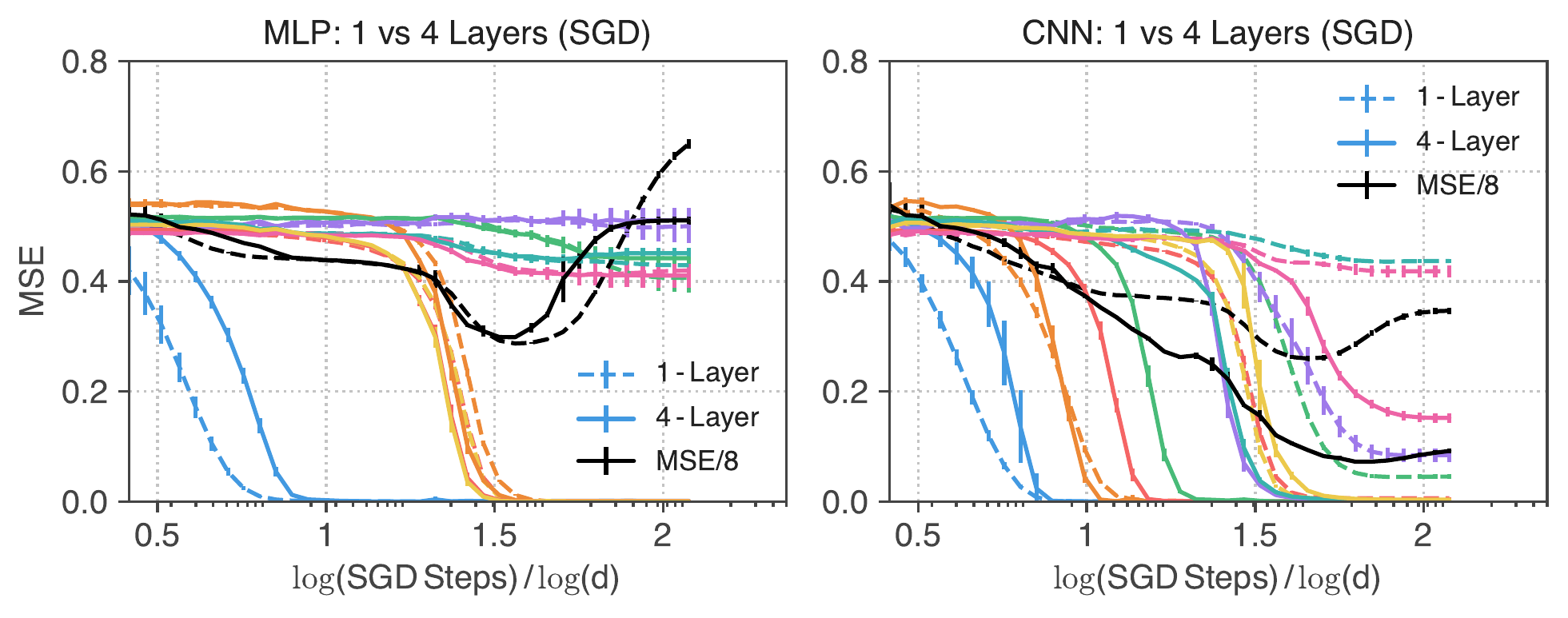}
    \caption{{\bf MLPs do not benefit from having more layers.} We plot the learning dynamics vs training set size / SGD steps for each eigenfunction $Y_i$. Top: NTK regression and bottom: SGD + Momentum. Left: MLP; right: CNN. {\bf Dashed lines / Solid lines} correspond to one-hidden / four-hidden layer networks.
    For both finite-width SGD training and infinite-width kernel regression, having more layers does not essentially improve performance of a MLP. This is in stark contrast to CNNs (right).
    By having more layers, the eigenstuctures of the kernels are refined. 
    % The exact architectures of the one-layer and the four-layer CNNs are given by {\it S\_CNN\_plus\_Act} and {\it HR\_CNN} in List.~\ref{lst:code-block}. 
    } 
    \label{fig:mlp-2}
\end{figure}

\subsection{ImageNet Plots
}
\begin{figure}
    \centering
        \includegraphics[width=1.\textwidth]{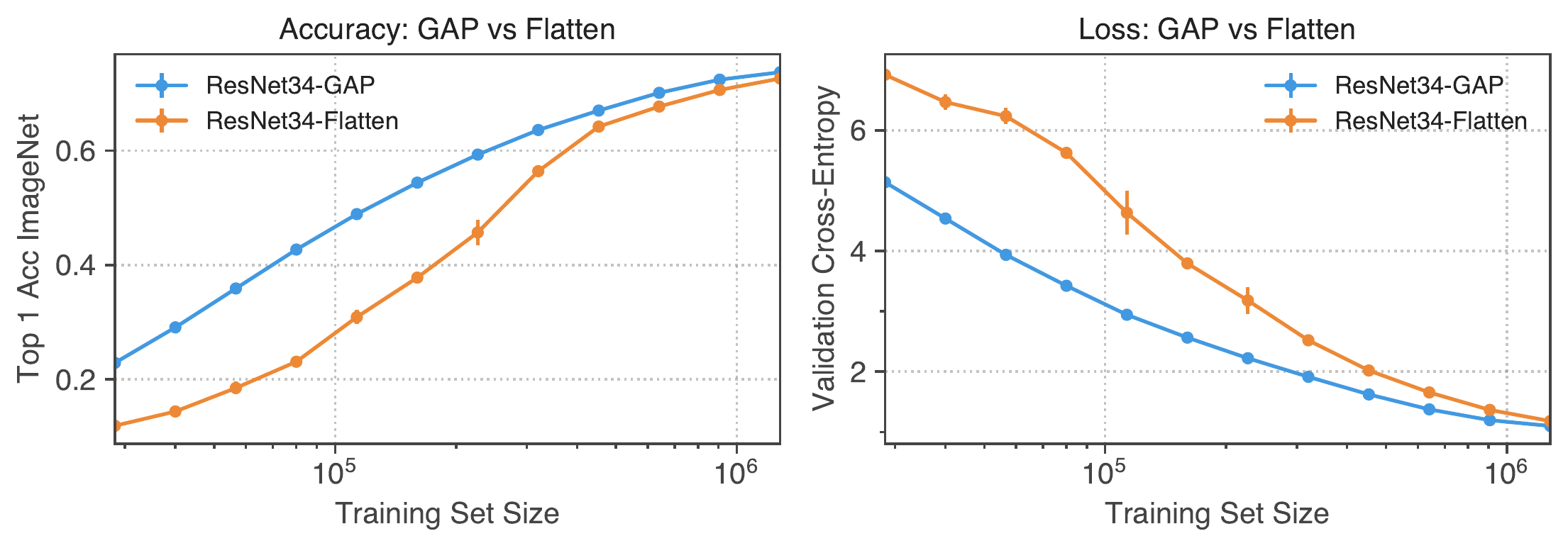}
        \\
    \includegraphics[width=1.\textwidth]{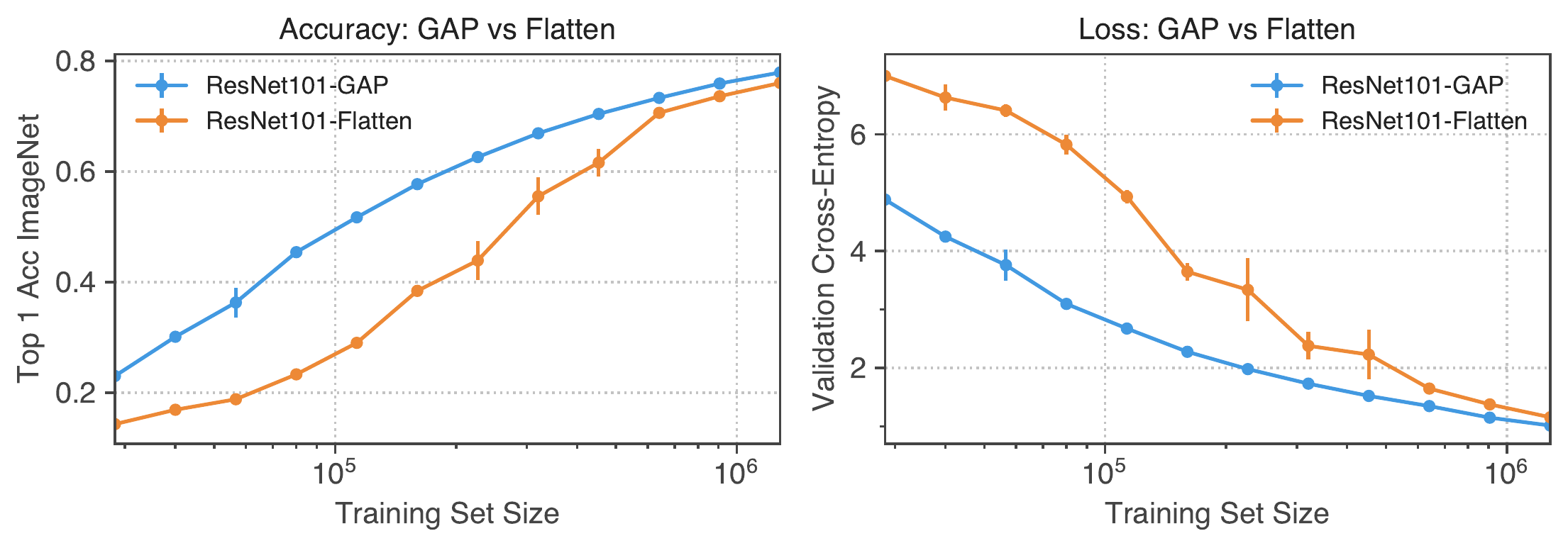}
    \caption{{\bf ResNet-GAP vs ResNet-Flatten.} As the training set size increases, the performance (accuracy and loss) gap between the two shrinks substantially. Top/bottom ResNet34/ResNet101}
    \label{fig:imagenet-34-101}
\end{figure}

 \newpage 

% \section{Assumptions} \label{sec:asumption on G}

% {\bf Assumption-$\bm\gG$. } Let $\bm \gG = (\Gd)_{d}$. There are absolute constants $c, C >0$ such that 
% \begin{enumerate}[label=(\alph*.)]
% \item For each {\it non-input} node $u\in\Nd$, there is $\alpha_u\in \shapep$ such that 
%     \begin{align}
%         \displaystyle 
%       c d^{\alpha_u}  \leq \deg(u) \leq C d^{\alpha_u} \,. 
%     \end{align}
% For each edge $uv\in\Ed$, its weight is defined to be  $\omega_{uv}\equiv\alpha_u$.
% \item For each {\it input} node $v$, there are $d_v\in\mathbb N$ and $0<\alpha_v\in\shapep$ such that 
% \begin{align}
%         \displaystyle 
%       c d^{\alpha_v}  \leq d_v\leq C d^{\alpha_v}\quad \text{and}\quad \sum_{v\in\Ndzero} \alpha_v =1. 
%     \end{align}
% \item Let $\hidden \equiv \{u: \exists v\in \gN_0^{(d)} \, \text{ s.t.}\,  uv\in \Ed \}$ be the collection of nodes in the {\it first} hidden layer. We assume that for every $u\in \hidden$, $\alpha_u=0$ and all children of $u$ are input nodes. 
% \item For every $v\in \Nd$,  
% $
%     \displaystyle  |\{u: uv\in \Ed\}| \leq C. 
% $
% Moreover, the number of {\it layers} is uniformly bounded, namely, for any node $u$, any path from $u$ to $\rootg$ contains at most $C$ edges. 
% \end{enumerate}
% The first two assumptions help to create spectral gaps between eigenspaces. When $d$ is not large, ``finite-width" effect is no longer negligible and we expect that the spectra decay more smoothly. Assumption (c.) says there is no ``skip" connections from the input layer to other layers except to the first hidden layer. 

\section{Proof of the Eigenspace Restructuring Theorem.}
The goal of this section is to prove the eigenspace restructuring theorem. 
In Sec.~\ref{sec:key-lemma}, we present a key lemma that relates the mixed derivatives of $\kk$ and $\Theta$ (as a function of $\vt$) to the architectures of the networks. We briefly recap some tools from spherical harmonics and then prove the theorem in Sec.\ref{sec:spherical harmonics}. For the rest of the paper, we will use the following notations. For $A, B: \mathbb N\to \mathbb R^{+} $, 
\begin{align}
    B(d) \gtrsim A(d) \iff A(d) \lesssim B(d) \iff	
   \exists c, \,d_0 > 0 \quad s.t. \quad B(d) \geq c A(d) >0 \quad \text{for all}\quad  d > d_0\,
\end{align}
and 
\begin{align}
    B(d) \sim A(d) \iff	
    B(d) \gtrsim A(d) \quad \text{and} \quad A(d) \gtrsim B(d)
\end{align}
\subsection{A Crucial Lemma}\label{sec:key-lemma}
\begin{lemma}\label{lemma:derivative-key} 
Same assumptions as Theorem \ref{thm:main}. 
Let $\vr\in\mathbb N^{|\Ndzero|}$. Then 
% \begin{align}
% \label{eq:zero-mode}
%     \kk_{\gG}(0),  \quad  \Theta_{\gG}(0) = 0
% \end{align}
% and 
for $\vr\neq 0$, for $d$ large enough,
\begin{align}\label{eq: lemma-derivative-key}
\kk_{\gG}^{(\vr)}(0),  \quad  \Theta_{\gG}^{(\vr)}(0) \sim d^{-\spatial(\vr)}
 \quad \text{ and}\quad
 \|\kk_{\gG}^{(\vr)}\|_{\infty}, \,\, 
  \|\Theta_{\gG}^{(\vr)}\|_{\infty} \lesssim d^{-\spatial(\vr)}
\end{align}
\end{lemma}
\begin{proof}
Recall that the recursion formulas for $\kk$ and $\infntk$ are
\begin{align}
\kk_u(\vt) &= \phi_u^*\left(\fint_{v:uv\in\gE}\kk_v(\vt)\right) 
\\
    \Theta_{u}(\vt) &= 
    \dot \phi_u^*\left(\fint_{v: uv\in\gE}  \kk_v(\vt)\right)\fint_{v:uv\in\gE}
    \left(\kk_v(\vt) +\Theta_{v}(\vt) \right)
\end{align}
 For convenience, define ${\pkk}$, $\barntk$,  $\dot \kk$ as follows
\begin{align}
\label{eq:re1}
     \pkk_u(\vt) &= \fint_{v:uv\in\gE}\kk_v(
     \vt)
     \\
    \barntk_u(\vt) &= \fint_{v:uv\in\gE}\infntk_v(
     \vt)
     \\
     \dot \kk_u(\vt) &= \dot \phi_u^*\circ \pkk_u(\vt)  \, .
\end{align}
Note that 
\[\kk_u(\vt) = \phi^*_u\circ \pkk_u (\vt) \quad 
\textand
\quad \infntk(\vt) = \dot \kk_u(\vt) (\pkk_u(\vt) + \barntk_u(\vt))\,. 
\]
and $\kk_u(\bm 0) = \infntk_u(\bm 0) =0$ if $u\neq \rootg$, which follow directly from recursions and the fact $\phi_u^*(0)=0$ for $u\in \Nd$ if $u\neq \rootg$. 

We induct on the tuple $(h, |\vr|)$, where $h\geq 0$ is the number of {\it hidden} layers in $\Gd$ and $|\vr|$ is the total degree of $\vr$.
We begin with the proof of the NNGP kernel $\kk$. 

{\bf Base Case I:} $|\vr|=1$ and $h\geq 0$ is any integer.  Let $u\in\leaf$ be such that $\vr={\ve}_u$, where $\{{\ve}_u\}_{\inputnodes}$ is the standard basis. Then 
\begin{align}
    \partial_{\vt_u} \kk_{\gG}(\vt) = \sum_{{\path}\in\gP({u\to\rootg})} \prod_{v\in \path}
    \deg(v)^{-1} \dot \phi_v^* \circ \pkk_v(\vt) \sim 
     \sum_{{\path}\in\gP({u\to\rootg})} \prod_{v\in \path} d^{-\alpha_v}\dot \phi_v^*\circ \pkk_v(\vt)%\Big|_{t=0}
\end{align}
Here $\gP({u\to u'})$ represents the set of paths from $u$ to $u'$. 
By {\bf Assumption-$\bm\gG$} and {\bf Assumption-$\phi$}, $|\gP(u\to\rootg)|$ (the cardinality) is uniformly bounded and
$\dot \phi^*_v(0)>0$ for all hidden nodes $v$. Therefore 
\begin{align*}
    \partial_{\vt_u} \kk_{\gG}(\bm 0) &\sim \sum_{{\path}\in\gP(u\to\rootg)}d^{-\sum_{v\in\path}\alpha_v}
    \dot \phi_v^* \circ \pkk_v(\bm 0)
    \\
    &= \sum_{{\path}\in\gP(u\to\rootg)}d^{-\sum_{v\in\path}\alpha_v}
    \dot \phi_v^* (0)
    \\&\sim \max_{{\path}\in\gP(u\to\rootg)}d^{-\sum_{v\in\path}\alpha_v}
    \\
    &= d^{-\spatial(\ve_u)}= d^{-\spatial(\vr)}
\end{align*}
In the above, we have used $\pkk_v(\bm 0)=0$ (which is due to $\phi^*(0) = 0$) and $\dot \phi_v^* \circ \pkk_v(\bm 0)\neq 0$ .
The second estimate $\| \partial_{t_u} \kk_{\gG}\|_{\infty} \lesssim d^{-\spatial(\vr)}$ follows from $|\dot \phi_v^* \circ \pkk_v(\vt) | \lesssim 1$.

{\bf Base Case II:} $h=0$ and $|\vr|\geq1$ is any number. Note that $\Gd$ has no hidden layer and all input nodes are linked to the output node $\rootg$. The case when the activation $\phi_{\rootg}$ is the identity function is obvious. We assume $\phi_{\rootg}$ is semi-admissible. 
\begin{align*}
        \partial^{\vr}_\vt \kk_{\gG}(\vt) = \deg(\rootg)^{-|\vr|} {\phi_{\rootg}^*}^{(|\vr|)}\left(\fint_{\inputnodes} \vt_u\right) \,. 
\end{align*}
This implies \myeqref{eq: lemma-derivative-key} since $\spatial(\vr)=0$,  $\deg(\rootg)\lesssim 1$ by {\bf Assumption-$\bm \gG$} and ${\phi_{\rootg}^*}^{(|\vr|)}(0)>0$ by
{\bf Assumption-$\phi$}. 

{\bf Induction:} $|\vr|\geq 2$, $h\geq 1$ and $\vr\in\Ad(\Gd)$. We only prove the first estimate in \myeqref{eq: lemma-derivative-key} since the other one can be proved similarly. WLOG, we assume $\phi_{\rootg}$ is not the identity function and hence is semi-admissible. 
Let $\inputnodes$ be such that $\vr_u\geq  1$ and denote $\bar \vr = \vr -\ve_u$. Then 
\begin{align}
    \partial_\vt^\vr \kk_{\gG}(\vt)\Big|_{\vt=\bm 0} &= 
    \partial_\vt^{\bar \vr} (\partial_{\vt}^{\ve_u}\kk_{\gG}(\vt))\Big|_{\vt=\bm 0} =
    \sum_{\substack{{\rootg v\in\gE}
    \\{\partial_{\vt}^{\ve_u} \kk_{v}\neq0}}} \deg(\rootg)^{-1}
    \partial_\vt^{\bar \vr} \left( \dot \kk_{\rootg}(\vt)
    \partial_{\vt}^{\ve_u} \kk_{v}(\vt)\right)\Big|_{\vt=\bm 0} \\
    &=
    \sum_{\substack{{\rootg v\in\gE}
    \\{\partial_{\vt}^{\ve_u} \kk_{v}\neq0}}}\sum_{\bar \vr_1 + \bar \vr_2 =\bar \vr} 
    \deg(\rootg)^{-1} \left(\partial_\vt^{\bar \vr_1} \dot \kk_{\rootg}(\vt) \,
    \partial_{\vt}^{\bar \vr_2 + \ve_u} \kk_{v}(\vt)\right)\Big|_{\vt=\bm 0} 
    \\
    &\sim 
    \sum_{\substack{{\rootg v\in\gE}
    \\{\partial_{\vt}^{\ve_u} \kk_{v}\neq0}}}\sum_{\bar \vr_1 + \bar \vr_2 =\bar \vr}  \deg(\rootg)^{-1} d^{-\spatial (\bar \vr_1)} d^{-\spatial (\nnode(\bar \vr_2 + \ve_u; v))}\label{eq:induction-step}
    \\
    &\sim 
    \sum_{\substack{{\rootg v\in\gE}
    \\{\partial_{\vt}^{\ve_u} \kk_{v}\neq0}}}\sum_{\bar \vr_1 + \bar \vr_2 =\bar \vr}  d^{-\spatial (\bar \vr_1)} d^{-\left(\alpha_{\rootg} + \spatial (\nnode(\bar \vr_2 + \ve_u; v))\right)}
    \\
    &\sim \sup_{\substack{{\rootg v\in\gE}
    \\{\partial_{\vt}^{\ve_u} \kk_{v}\neq0}}} \sup_{\bar \vr_1 + \bar \vr_2 =\bar \vr} 
    d^{- \left(\spatial (\bar \vr_1) +\alpha_{\rootg} + \spatial (\nnode(\bar \vr_2 + \ve_u; v))\right)}
    \label{eq:sum2sup}
\end{align}
We have applied induction twice in \myeqref{eq:induction-step}: one to obtain the estimate $\partial_\vt^{\bar \vr_1}\dot \kk_{\rootg}^*(\bm 0) \sim d^{-\spatial (\bar \vr_1)}$ (with $|\bar \vr_1|< |\vr|$ and $\dot \phi_{\rootg}^*$ semi-admissible) and one to $\partial_{\vt}^{\bar \vr_2 + \ve_u} \kk_{v}(\vt)\sim d^{-\spatial (\nnode(\bar \vr_2 + \ve_u; v))}$, in which the sub-graph with $v$ as the output node has depth at most $(h-1)$. 
The last line follows from that both the cardinality of the tuple $(\bar\vr_1, \bar\vr_2)$ with $\bar\vr_1, \bar\vr_2\geq \bm 0$ and  $\bar\vr_1+\bar\vr_2=\bar\vr$
and the cardinality of $v\in \Ndzero$ with $\rootg v\in \gE$ and $\partial_\vt^{\ve_u}\kk_{v}\neq 0$ are finite and independent of $d$.
From the definition of MST, it is clear that for all $(\bar \vr_1, \bar \vr_2)$
\begin{align}
    \spatial (\bar \vr_1) +\alpha_{\rootg} + \spatial (\nnode(\bar \vr_2 + \ve_u; v))
    \geq 
    \spatial (\bar \vr_1) + \spatial (\bar \vr_2 + \ve_u) \geq 
    \spatial(\vr)
\end{align}
It remains to show that there exists at least one pair $(\bar \vr_1, \bar \vr_2)$ such that the above can be an equality. Let $\gT \subseteq \Gd$ be a MST containing all nodes in $\nnode(\vr)$. If $\rootg$ has only one child $v$ in $\gT$, then we choose $\bar \vr_1 = 0$ and notice that 
\begin{align}
    \spatial (\bar \vr_1) +\alpha_{\rootg} + \spatial (\nnode(\bar \vr_2 + \ve_u; v))
     = 0 +  
     \spatial ( \bar\vr_2 + \ve_u) = 
    \spatial(\vr)
\end{align}
since $\spatial (\bar \vr_1) =0$ and $\bar \vr_2 + \ve_u=\vr$. 
Else, $\rootg$ contains at least two children in $\gT$ and therefore at least two disjoint branches. Let $\gT_u\subseteq \gT$ be the branch that contains $u$ and choose $\bar \vr_2\leq \bar \vr$ be such that all the nodes of $(\bar \vr_2 + \ve_u)$ are contained in $\gT_u$ and all the nodes of $\bar \vr_1 \equiv \vr - (\bar \vr_2 + \ve_u)$ are contained in $\gT\backslash \gT_u$. Clearly 
\begin{align}
    \spatial(\vr) = \spatial (\bar \vr_1)  +
    \spatial (\bar \vr_2 + \ve_u) = 
    \spatial (\bar \vr_1)  +
    \spatial (\nnode (\bar \vr_2 + \ve_u; v)) + \alpha_{\rootg}, 
\end{align}
where $v$ is the unique child of $\rootg$ in $\gT_u$.  

This completes the proof of the NNGP kernel $\kk$. As the proof of the NTK part is quite similar, we will be brief and focus only on the induction step.

{\bf Induction Step of $\Theta$:} $|\vr|\geq 2$, $h\geq 1$ and $\vr\in\Ad(\Gd)$. Recall that the formula of $\Theta$ is
\begin{align}
    \Theta_{u}(\vt) = 
    \dot \phi_u^*\left(\fint_{v: uv\in\gE}  \kk_v(\vt)\right)\fint_{v:uv\in\gE}
    \left(\kk_v(\vt) +\Theta_{v}(\vt) \right)
\end{align}
For $\vr\in\mathbb N^{|\Ndzero|}$, 
\begin{align}
    \partial_\vt^{\vr} \Theta_{\rootg}(\vt) = 
    \sum_{\bar\vr_1 +\bar\vr_2=\vr}  \partial_\vt^{\bar\vr_1} \dot \phi_{\rootg}^*\left(\fint_{v: \rootg v\in\gE}  \kk_v(\vt)\right)
     \partial_\vt^{\bar\vr_2} \fint_{v:\rootg v\in\gE}
    \left(\kk_v(\vt) +\Theta_{v}(\vt) \right)
\end{align}
Note that $\dot \phi_{\rootg}^*$ is semi-admissible. We apply the result of $\kk$ to conclude that 
\begin{align}
    &\partial_\vt^{\bar\vr_1} \dot \phi_{\rootg}^*\left(\fint_{v: \rootg v\in\gE}  \kk_v(\vt)\right) \Bigg|_{\vt=\bm 0} \sim d^{-\spatial(\vr_1)}
    \\
    &\left\| \partial_\vt^{\bar\vr_1} \dot \phi_{\rootg}^*\left(\fint_{v: \rootg v\in\gE}  \kk_v(\vt)\right)\right\|_{\infty} \lesssim d^{-\spatial(\bar \vr_1)}
\end{align}
and the inductive step to conclude that 
\begin{align}
&\partial_\vt^{\bar\vr_2} \fint_{v:\rootg v\in\gE}
\left(\kk_v(\vt) +\Theta_{v}(\vt) \right)\Bigg|_{\vt=\bm 0} 
    \sim
    \sum_{\substack{v:\rootg v\in\gE \\
    \partial_\vt^{\bar\vr_2} (\kk_v +\infntk_v)\not\equiv \bm  0}}
    d^{-\spatial(\bar\vr_2)}
    \quad \text{if} \quad \bar\vr_2 \neq \bm 0  \quad \text{else} \quad 0 
    \\
    &\partial_\vt^{\bar\vr_2} \fint_{v:\rootg v\in\gE}
    \left(\kk_v(\vt) +\Theta_{v}(\vt) \right)
    \lesssim \sum_{\substack{v:\rootg v\in\gE \\
    \partial_\vt^{\bar\vr_2} (\kk_v +\infntk_v)\not\equiv \bm 0}} d^{-\spatial(\bar \vr_2)} \,. 
\end{align}
Note that
\[|\{{v:\rootg v\in\Ed \,\,\,\,\textand \,\,\,\,
    \partial_\vt^{\bar\vr_2} (\kk_v +\infntk_v)\not\equiv \bm  0}\}|\lesssim 1 \,. \]
Thus 
\begin{align}
    \left\|\partial_\vt^{\vr} \Theta_{\rootg}(\vt)\right\|_{\infty} \lesssim 
    \sum_{\bar\vr_1 +\bar\vr_2=\vr}  d^{-\spatial(\bar \vr_1)-\spatial(\bar \vr_2)} \lesssim d^{-\spatial(\vr)}. 
\end{align}
To control the lower bound, let $\gT$ be a MST containing $\nnode(\vr)$. If $\deg(\rootg; \gT)=1$, then we can choose $\bar \vr_1=0$ and $\bar \vr_2=\vr \neq \bm 0$. Notice that there is at least one child node $v$ of $\rootg$ with $\partial_\vt^{\bar\vr_2} (\kk_v +\infntk_v)\not\equiv \bm 0$. Therefore  
\begin{align}
    \sum_{v:
    \partial_\vt^{\bar\vr_2} (\kk_v +\infntk_v)\not\equiv\bm  0}
    d^{-\spatial(\bar\vr_2)}  \gtrsim d^{-\spatial(\bar\vr_2)}
    = d^{-\spatial(\vr)}
\end{align}
Combining with 
\begin{align}
    \dot \phi_{\rootg}^*\left(\fint_{v: \rootg v\in\gE}  \kk_v(\bm 0)\right) = \dot \phi_{\rootg}^*(0) >0
\end{align}
we have 
\begin{align*}
      \partial_\vt^{\vr} \Theta_{\rootg}(\bm 0) \gtrsim  d^{-\spatial(\vr)}
\end{align*}
It remains to handle the $\deg(\rootg; \gT)>1$ case. We choose $(\bar \vr_1, \bar \vr_2)$ such that one branch of $\gT$ is the MST that contains $\nnode(\bar\vr_2)$ and $\rootg$, and the remaining branch(es) is a MST that contains $\nnode(\bar\vr_1)$ and $\rootg$. Then 
\begin{align*}
      \partial_\vt^{\vr} \Theta_{\rootg}(\bm 0) 
      \gtrsim& 
      \partial_\vt^{\bar\vr_1} \dot \phi_{\rootg}^*\left(\fint_{v: \rootg v\in\gE}  \kk_v(\vt)\right)
     \partial_\vt^{\bar\vr_2} \fint_{v:\rootg v\in\gE}
    \left(\kk_v(\vt) +\Theta_{v}(\vt) \right)
    \Bigg |_{\vt=\bm 0}
    \\
    \gtrsim& d^{-\spatial(\vr_1)-\spatial(\vr_2)} 
    =d^{-\spatial(\vr)}
\end{align*}
\end{proof}

\subsection{Legendre Polynomials, Spherical Harmonics and their Tensor Products.}\label{sec:spherical harmonics}
Our notation follows closely from \citep{frye2012spherical}.  

\paragraph{Legendre Polynomials.}
Let ${\din}\in\mathbb N^*$
and $\omega_{\din}$ be the measure defined on the interval  $I=[-1, 1]$
\begin{align}
    \omega_{\din}(t) = (1-t^2)^{(\din-3)/2}
\end{align}
The Legendre polynomials\footnote{More accurate, this should be called \href{https://en.wikipedia.org/wiki/Gegenbauer_polynomials}{Gegenbauer Polynomials}. 
However, we stick to the terminology in \citep{frye2012spherical}} $\{P_r(t):r\in\mathbb N\}$ is an orthogonal basis for the Hilbert space $L^2(I, \omega_{\din})$, i.e. 
\begin{align}
    \int_I P_r(t) P_{r'}(t) \omega_{\din}(t) dt = 0 \quad \text{if} \quad r\neq r' \quad \text{else} \quad 
    N(\din, r)^{-1} \left(\frac{|\mathbb S_{\din-1}|}{|\mathbb S_{\din-2}|}\right)
\end{align}
Here $P_r(t)$ is a degree $r$ polynomials with $P_r(1) = 1$ that satisfies the formula below,  $N(\din, r)$ is the cardinality of degree $r$ spherical harmonics in $\mathbb R^\din$ and $|\mathbb S_{\din-1}|$ is the surface area of $\mathbb S_{\din-1}$. 
\begin{lemma}[Rodrigues Formula. Proposition 4.19 \citep{frye2012spherical}] 
\label{lemma:rodrigues formula}
\begin{align}
    P_r(t) = c_r \omega^{-1}_{\din}(t) \left(\frac{d}{dt}\right)^r (1-t^2)^{r + (\din-3)/2}\,,
\end{align}
where 
\begin{align}
    c_r = \frac{(-1)^r}{2^r(r+(\din-3)/2)_r}
\end{align}
\end{lemma}
In the above lemma, $(x)_l$ denotes the falling factorial 
\begin{align}
    (x)_l &\equiv  x(x-1) \cdots (x-l+1) \\
    (x)_0 &\equiv 1 
\end{align}

\paragraph{Spherical Harmonics.}
 Let $d\sphered$ define the (un-normalized) uniform measure on the unit sphere $\sphered$. Then 
\begin{align}
    |\sphered| \equiv \int_{\sphered} d\sphered = \frac{2\pi^{\din/2}}{\Gamma(\frac {\din}{2})}\,.
\end{align}
The normalized measure on this sphere is defined to be 
\begin{align}
    d\sigma_{\din} = \frac{1}{|\sphered|} d \sphered  \quad \text{and}\quad \int_{\sphered} d\sigma_{\din} =1  \,.
\end{align}
The spherical harmonics $\{Y_{r, l}\}_{r, l}$ in $\R^{\din}$ are homogeneous harmonic polynomials that form an orthonormal basis in $L^2(\sphered, \sigma_{\din})$ 
\begin{align}
    \int_{\xi\in\sphered} Y_{r, l} (\xi) Y_{r', l'}(\xi) d\sigma_{\din} = \delta_{(r, l)={(r',l')}}\,.
\end{align}
Here $Y_{r, l}$ denotes the $l$-th spherical harmonic whose degree is $r$, where $r\in \mathbb N$, $ l\in [N(\din, r)]$ and 
\begin{align}
    N(\din, r) = 
    \frac{2r +\din -2}{r} 
{\din+r-3 \choose r-1}
\sim (\din)^{r}/{r}! \quad \text{as} \quad \din\to \infty \,. 
\end{align}
The Legendre polynomials and spherical harmonics are related through the addition theorem.  
\begin{lemma}[Addition Theorem. Theorem 4.11 \citep{frye2012spherical}]
\label{lemma: addition theorem}
\begin{align}
    P_r(\xi^T\eta) = \frac{1}{N(\din, r)} \sum_{l\in [N(\din, r)]} Y_{r, l}(\xi)
    Y_{r, l}(\eta), \, \quad
    \xi,\eta\in \sphered\,. 
\end{align}
\end{lemma}

\paragraph{Tensor Products.} Let $\vd = (\vd_u)_{\inputnodes}\in \mathbb N^{|\Ndzero|}$,  $\vr \in\mathbb N^{|\Ndzero|}$, $\bm I = I^{|\Ndzero|}= [-1, 1]^{|\Ndzero|}$ and $\bm \omega = \bigotimes_{\inputnodes}\omega_{\vd_u}$ be a product measure on $\bm I$. Then the (product of) Legendre polynomials 
\begin{align}
    \vP_{\vr}(\vt) = \prod_{u\in \Ndzero}P_{\vr_u}(\vt_u)\, , \quad \vt= (\vt_u)_{u\in\Ndzero}\in \bm I \, , 
\end{align}
which form an orthogonal basis for the Hilbert space $L^2(\bm I, \bm \omega) = \bigotimes_{\inputnodes}L^2(I, \omega_{\vd_u})$.
% Let
% \begin{align}
%     \vd = (d_u)_{\inputnodes}
% \end{align}
Similarly, the tensor product of spherical harmonics 
\begin{align}
    \vY_{\vr, \vl} = \prod_{\inputnodes} Y_{\vr_u, \vl_u}, \quad \vl =(\vl_u)_{\inputnodes} \in [\vN(\vd, \vr)] \equiv \prod_{\inputnodes}[N(\vd_u, \vr_u)] \, 
\end{align}
form an orthonormal basis for the product space 
\begin{align}
    % L^2(\spheredp, \sigmadp) = L^2(\sphered, \sigma_{\din})^{\otimes p}.
    % \\
L^2(\inputx, \bm \sigma)\equiv\bigotimes_{\inputnodes}  L^2(\sphere_{\vd_u-1}, \sigma_{\vd_u})
\end{align}
Elements in the set $\{\vY_{\vr, \vl}\}_{\vl \in [\vN(\vd, \vr)]}$ are called degree (order) $\vr$ spherical harmonics in $L^2(\inputx, \bm \sigma)$ and also degree $r$ spherical harmonics if $|\vr|=r\in\mathbb N$.

\begin{theorem}\label{thm:2nd main} 
Same assumptions as Theorem \ref{thm:main}. 
We have the following, for $\K=\kk_{\Gd}$ or $\K=\infntk_{\Gd}$ 
\begin{align}
\K(\vt) = 
\sum_{\vr\in\mathbb N^{|\Ndzero|}} \hat \K(\vr) \vP_\vr(\vt) \qquad \text{with} \qquad  \hat \K(\vr) \sim d^{-\spatial(\vr)} \quad \text{if} \quad \vr\neq 0 \quad \text{and}\quad |\vr|\lesssim 1. 
\end{align}
\end{theorem}
Note that Theorem \ref{thm:main} follows from this theorem and the addition theorem. 
\begin{proof}[Proof of Theorem \ref{thm:main}]
Assume $\vr\neq \bm 0$.
Indeed, setting 
\[
\bm\xi = (\bm\xi_u)_{\inputnodes}\in \inputx, \quad \bm\eta = (\bm\eta_u)_{\inputnodes} \in \inputx\quad {and} \quad \vt = (\vt_u)_{\inputnodes}  = (\bm\xi_u^T  \bm\eta_u / \vd_u )_{\inputnodes}\,, 
\]
we have 
\begin{align}
    \vP_{\vr}(\vt) &= \prod_{\inputnodes} P_{\vr_u}(\vt_u) =  \prod_{\inputnodes} N(\vd_u, \vr_u) ^{-1} \sum_{\vl_u\in N(\vd_u, \vr_u)}Y_{\vr_u, \vl_u}(\bm\xi_u/\sqrt {\vd_u})
    Y_{\vr_u, \vl_u}(\bm\eta_u/\sqrt {\vd_u})
    \\&= \vN(\vd, \vr)^{-1}
    \sum_{\vl\in [\vN(\vd, \vr)]}\overline{\bm\vY}_{\vr, \vl}
    (\bm\xi)\overline{\bm\vY}_{\vr, \vl}(\bm\eta)\,.
\end{align}
Then Theorem \ref{thm:main} follows by noticing
\begin{align*}
    \hat\K(\vr)  \vN(\vd, \vr)^{-1}
    \sim d^{-\spatial(\vr)} \prod_{\inputnodes}\vd_u^{-\vr_u} 
    \sim d^{-\spatial(\vr)} d^{-\sum_{\inputnodes}\alpha_{u}\vr_u}
    =
 d^{-\spatial(\vr) - \learning(\vr)} 
    = d^{-\learning(\vr)}\, .
\end{align*}

\end{proof}

\begin{proof}[Proof of Theorem \ref{thm:2nd main}]
% To ease the notation, we set $\kk = \kk_{\Gd}$ in the proof. 
From the orthogonality, 
\begin{align}\label{eq: fourier coefficients}
\hat \K(\vr) = \langle \K, \vP_\vr \rangle / \|\vP_\vr\|_{L^2(\bm I, \bm\omega)}^2
\end{align}
We begin with the denominator. Note that 
\begin{align}
    \|\vP_\vr\|_{L^2(\bm I, \bm\sigma})^2 = \prod_{\inputnodes} \|P_{\vr_u}\|_{L^2(I, \omega_{\vd_u})}^2
    = \vN(\vd; \vr)^{-1} \prod_{\inputnodes}(|\mathbb S_{\vd_u-1}|/|\mathbb S_{\vd_u-2}|) 
\end{align}

By applying Lemma \ref{lemma:rodrigues formula}, integration by parts and continuity of $\K^{(\vr)}$ on the boundary $\partial \bm I$
\begin{align}
     \langle \K, \vP_{\vr}\rangle_{L^2(\bm I, \bm\omega)}  
     &= c_{\vr} \int_{\bm I} \K(\vt) \left(\frac{d}{d\vt}\right)^{\vr} \left(1 - \vt^2\right)^{\vr + (\vd -3)/2} d\vt 
     \\
     &= (-1)^{\vr}c_{\vr} \int_{\bm I} \K^{(\vr)}(\vt)  \left(1 - \vt^2\right)^{\vr + (\vd -3)/2} d\vt 
     \\
     &= (-1)^{\vr}c_{\vr}  \left(\mathcal M(\K, \vd) +  \epsilon(\K, \vd) \right) 
    %  (-1)^{\vr}c_{\vr} \int_{\bm I} \kk^{(\vr)}(0)  \left(1 - \vt^2\right)^{\vr + (\din -3)/2} d\vt  + 
    %  \epsilon(\kk, \din)
\end{align}
where $\K^{(\vr)}$ is the $\vr$ derivative of $\K$, the coefficient $c_{\vr}$ is given by Lemma \ref{lemma:rodrigues formula} 
\begin{align}
     c_{\vr} = \prod_{\inputnodes}  c_{\vr_u} =  \prod_{\inputnodes} \frac{(-1)^{\vr_u}}{2^{\vr_u}({\vr_u}+(\vd_u-3)/2)_{\vr_u}} 
     \sim  \prod_{\inputnodes} (-1)^{\vr_u}\vd_u^{-\vr_u} = (-1)^{\vr} \vd^{-\vr}
\end{align}
(note that only $\lesssim 1$ many $\vr_u\neq0$) and the major and error terms are given by 
\begin{align}
    \mathcal M(\K, \vd) &\equiv \K^{(\vr)}(\bm 0)  \int_{\bm I} \left(1 - \vt^2\right)^{\vr + (\vd -3)/2} d\vt  
     = \K^{(\vr)}(\bm 0)   \prod_{\inputnodes} \frac{|\sphere_{2\vr_u+\vd_u-1}|}{|\sphere_{2\vr_u+\vd_u-2}|}
    \\
    \mathcal \epsilon(\K, \vd) &\equiv \int_{\bm I} ( \K^{(\vr)}(\vt) - \K^{(\vr)}(\bm 0) ) \left(1 - \vt^2\right)^{\vr + (\vd -3)/2} d\vt  
\end{align}
We first show that the error term is small. The mean value theorem gives  
\begin{align}
    |\K^{(\vr)}(\vt) - \K^{(\vr)}(\bm 0)| \leq  \sum_{\inputnodes}
    \|\K^{(\vr + \ve_u)}\|_{L^{\infty}(\bm I)} |\vt_u|
\end{align}
and the error term $ |\mathcal \epsilon(\K, \vd)| $ is bounded above by 
\begin{align}
    & \int_{\bm I} \left(1 - \vt^2\right)^{\vr + (\vd -3)/2} d\vt   \sum_{\inputnodes}  \|\K^{(\vr + \ve_u)}\|_{L^{\infty}(\bm I)} \left(\frac{
     \int_{I} |\vt_u|\left(1 - \vt_u^2\right)^{\vr_u + (\vd_u -3)/2} d \vt_u}{
     \int_{I}\left(1 - \vt_u^2\right)^{\vr_u + (\vr_u -3)/2} d \vt_u}  \right)
     \\
     &\lesssim    \left( \prod_{\inputnodes} \frac{|\sphere_{2\vr_u+\vd_u-1}|}{|\sphere_{2\vr_u+\vd_u-2}|} \right)
    \left( \sum_{\inputnodes}   \|\K^{(\vr + \ve_u)}\|_{L^{\infty}(\bm I)}
    \vd_u^{-1} \left (\frac{|\sphere_{2\vr_u+\vd_u-1}|}{|\sphere_{2\vr_u+\vd_u-2}|} \right)^{-1} \right)\,. 
\end{align}
Since for any $\alpha \in \mathbb N$, as $\vd_u\to\infty$, 
\begin{align}
     \frac{|\sphere_{\alpha+\vd_u-1}|}{|\sphere_{\alpha+\vd_u-2}|} = \pi^{\frac 1 2 } \Gamma((\alpha +\vd_u -1)/2) /  \Gamma((\alpha +\vd_u)/2) 
     \sim  \pi^{\frac 1 2 } (\vd_u/2)^{-\frac 1 2} \sim (\vd_u)^{-\frac 1 2 }\,,
\end{align}
we have 
\begin{align}
    |\mathcal \epsilon(\K, \vd)| \lesssim 
   \sum_{\inputnodes}  \|\K^{(\vr + \ve_u)}\|_{L^{\infty}(\bm I)}
   \vd_u^{-\frac 1 2 }   \prod_{\inputnodes} \frac{|\sphere_{2\vr_u+\vd_u-1}|}{|\sphere_{2\vr_u+\vd_u-2}|} \, . 
\end{align}
We claim that (which will be proved later) 
\begin{align}\label{eq:claim1 mst}
    \sum_{\inputnodes}  \|\K^{(\vr + \ve_u)}\|_{L^{\infty}(\bm I)} \lesssim d^{-\spatial(\vr)}
\end{align}
which implies 
\begin{align}
     \langle \K, \vP_{\vr}\rangle_{L^2(\bm I, \omega_{\din}^p)}  = c_{\vr}\left(\K^{(\vr)}(0) +  \mathcal O\left(d^{-\spatial(\vr )}(\min_{\inputnodes}\vd_u)^{-\frac 1 2 }  \right)\right )\prod_{\inputnodes} \frac{|\sphere_{2\vr_u+\vd_u-1}|}{|\sphere_{2\vr_u+\vd_u-2}|} 
\end{align} 
Plugging back to \myeqref{eq: fourier coefficients}, we have 
\begin{align}
    \hat \K(\vr) = (-1)^{\vr}
    c_{\vr} \vN(\vd, \vr) \left(\K^{(\vr)}(\bm 0) +  \mathcal O\left(d^{-\spatial(\vr )}(\min_{\inputnodes}\vd_u)^{-\frac 1 2 }  \right)
    \right )\left(\prod_{\inputnodes} \frac{|\sphere_{2\vr_u+\vd_u-1}|}{|\sphere_{2\vr_u+\vd_u-2}|}  \left(\frac{|\sphere_{\vd_u-1}|}{|\sphere_{\vd_u-2}|}\right)^{-1} \right) 
\end{align}
Since, for $\vr$ fixed and as $\vd_u\to\infty$ for all $\inputnodes$  
\begin{align}
    \frac{c_\vr} { (-1)^{\vr} \vd^{-\vr}} \to 1 \quad \text{and} \quad 
    \frac {\vN(\vd, \vr)}{ \vd^\vr / \vr !}   \to 1  \quad \text{and}\quad \left(\prod_{\inputnodes} \frac{|\sphere_{2\vr_u+\vd_u-1}|}{|\sphere_{2\vr_u+\vd_u-2}|}  \left(\frac{|\sphere_{\vd_u-1}|}{|\sphere_{\vd_u-2}|}\right)^{-1} \right) \to 1 
\end{align}
The last limit have used the fact that $|\vr|$ is bounded (i.e. $\lesssim 1$) and $\vr_u\neq 0$ for at most $|\vr|$ many $u$. 
Therefore  
\begin{align}
    \hat \K(\vr) \sim {\vr!}^{-1} \left(\K^{(\vr)}(\bm 0) + \mathcal O\left(d^{-\spatial(\vr )}\left(\min_{\inputnodes}\vd_u\right)^{-\frac 1 2 }  \right) \right )
\end{align} 
It remains to verify \myeqref{eq:claim1 mst}. By Lemma \ref{lemma:derivative-key}, we only need to show that 
\begin{align}
    \sum_{\inputnodes} d^{-\spatial(\vr + \ve_u)} \lesssim d^{-\spatial(\vr )}
\end{align}
We prove this by induction on the number of hidden layers of $\Gd$. The base case is obvious. Now suppose the depth of $\Gd$ is $h$. Let $\gC(\vr)$ be the set of children of $\rootg$ that are ancestors of at least one node of $\nnode(\vr)$. We split $\Ndzero$ into two disjoint sets 
\[\gQ(\vr)\equiv \{u\in\Ndzero: \exists v\in\gC(\vr) \,\, \text{s.t.} \,\,\gP(u\to v)\neq \emptyset\} \,
\quad 
\text{and}\quad
\Ndzero\backslash\gQ(\vr)
\,.\]
For $u\notin\gQ(\vr)$, we have $\spatial(\vr + \ve_u) = \spatial(\vr) + \spatial(\ve_u)$ and hence 
\begin{align}
\label{eq:est-1-tree}
    \sum_{u\notin \gQ(\vr)} 
    d^{-\spatial(\vr + \ve_u)}
    = 
    \sum_{u\notin \gQ(\vr)} 
    d^{-\spatial(\vr) - \spatial(\ve_u)}
    =d^{-\spatial(\vr)} \sum_{u\notin \gQ(\vr)} 
    d^{-\spatial(\ve_u)} \lesssim d^{-\spatial(\vr)}  \, . 
\end{align}
In the last inequality above, we have used  
\begin{align}
   \sum_{u\notin \gQ(\vr)} d^{-\spatial(\ve_u)}  \leq 
    \sum_{u\in  \Ndzero}d^{-\spatial(\ve_u)} \sim 1 \,. 
\end{align}
To estimate the remaining, we use induction. Note that $|\gC(\vr)|$ is finite (i.e. $\lesssim 1$) and independent of $d$. Then 
\begin{align}
    \sum_{u\in \gQ(\vr)} 
    d^{-\spatial(\vr + \ve_u)}
    &\leq \sum_{v\in \gC(\vr)}
    \sum_{u\in \gQ(\vr)} 
    d^{-\alpha_{\rootg} -\spatial(\nnode(\vr + 
    \ve_u; v))}
    \\
    &= d^{-\alpha_{\rootg}}\sum_{v\in \gC(\vr)} \sum_{u\in \gQ(\vr)} 
    d^{-\spatial(\nnode(\vr + \ve_u; v))} 
    \\
&\lesssim |\gC(\vr)|
    d^{-\alpha_{\rootg}}
    \max_{v\in \gC(\vr)}d^{-\spatial(\nnode(\vr; v))} \sim  d^{-\spatial(\vr)}
\end{align}
We have used induction on the sub-graph with $v$ as the output node. 
\end{proof}
\section{Proof of Theorem \ref{thm:generalization}} 
Let $\Gd$ be a DAG associated to the convolutional networks whose filter sizes in the $l$-th layer is $k_l=[d^{\alpha_l}]$, for $0\leq l\leq L+1$, in which we treat the {\it flatten-dense} readout layer as a convolution with filter size $[d^{\alpha_{L+1}}]$. Note that we have set $\alpha_p=\alpha_0$ and $\alpha_w=\alpha_{L+1}$. We also assume an {\it activation} layer after the {\it flatten-dense} layer, which does not essentially alter the topology of the DAG.  

We need the following dimension counting lemma.  
\begin{lemma}\label{lemma:dimension}
Let $r\in \learning(\Gd).$ Then 
\begin{align}
\dim\left(\Span\left\{\bm \barY_{\vr, \vl}: \quad \learning(\vr) =r, \,\, \vl\,\, \in\bm N(\vd, \vr) \right\}\right) \sim d^{r}
\end{align}
\end{lemma}

% For convenience, we prove the theorem for a slightly more general network, allowing each hidden layer to have different filter size. Let $k_l \sim d^{\alpha_l}$ be the filter size of the $l$-hidden layer, where $1\leq l\leq L$. Set the patch size $p=d^{\alpha_p}\equiv \alpha_0$ and $w\sim d^{\alpha_w}\equiv d^{\alpha_w}$. 
To prove Theorem \ref{thm:generalization}, we only need to verify the assumptions of Theorem 4 in \citet{mei2021generalization}. For convenience, we briefly recap the assumptions and results from \citet{mei2021generalization} in Sec.\ref{sec: kernel concentraction}. 

It is convenient to group the eigenspaces together according to the learning indices $\learning(\Gd)$. Recall that $\learning(\Gd) = (r_1\leq r_2 \leq r_3 \dots)$. Let 
\begin{align}
    E_i = \Span\{\bm \barY_{\vr, \vl}: \learning(\vr) = r_i\}
\end{align}
Then by Theorem \ref{thm:main} and Lemma \ref{lemma:dimension}, 
\begin{align}\label{eq: dimension}
    \dim (E_i) \sim d^{r_i} \quad \textand\quad  \lambda(g) \sim d^{-r_i} \quad \forall g\in E_i, \,\, g\neq \bm 0\, ,
\end{align}
where $ \lambda(g) $ denotes the eigenvalue of $g$. 
We proceed to verify Assumptions 4 and 5 in Sec.~\ref{sec: kernel concentraction}.  They follow directly from Theorem \ref{thm:main}, Lemma \ref{lemma:dimension} and the hypercontractivity of spherical harmonics \citet{beckner1992sobolev}. 

\subsection{Verifying {\bf Assumption 4}}
We need the following. 
\begin{proposition}\label{prop:hyper}
For $0< s\in\mathbb R$, let  $D_{s} = \text{span}\{\bm \barY_{\vr, \vl}: |\learning(\vr)|< s\}$. 
Then for $\vf\in D_s$, 
\begin{align}
    \|\vf \|_q^2 \leq (q-1)^{s/\alpha_0}\|\vf\|_2^2
\end{align}
\end{proposition}
\begin{proof}[Proof of Proposition \ref{prop:hyper}.]
The lemma follows from the tensorization of hypercontractivity.
Let $f=\sum_{k\geq 0} Y_k  \in L^2(\sphere_{n})$ where $Y_k$ is a degree $k$ spherical harmonics in $\sphere_n$. Define the Poisson semi-group operator 
\begin{align}
    P_\epsilon f(x)
    = \sum_{k\geq 0}  \epsilon^k Y_k(x)
\end{align}
Then we have the hypercontractivity inequality \citep{beckner1992sobolev}, for $1\leq p\leq q$ and $ \epsilon\leq \sqrt{\frac{p-1}{q-1}} $ 
\begin{align}
    \|P_\epsilon f\|_{L^q(\sphere_n)} \leq \|f\|_{L^p(\sphere_n)} 
\end{align}
One can then tensorize \citep{beckner1975inequalities} it to obtain the same bound in the tensor space. 
\begin{lemma}[Corollary 11 \citet{montanaro2012some}]
Let $\vf: (\sphere_n)^k\to \mathbb R$. If $1\leq p\leq q$ and $ \epsilon\leq \sqrt{\frac{p-1}{q-1}} $, then 
\begin{align}
    \|P^{\otimes k}_\epsilon \vf\|_{L^q((\sphere_n)^k)} \leq \|\vf \|_{L^p((\sphere_n)^k)}\,. 
\end{align}
\end{lemma}
Let $\vf = \sum_{\vr, \vl} a_{\vr, \vl}  \bm\barY_{\vr, \vl}\in D_s$. Choosing $\epsilon= \sqrt {\frac 1 {q-1}}$ and $p=2$ in the above lemma, we have 

\begin{align}
    \|\vf\|_q^2 &= \| \sum_{\vr, \vl} a_{\vr, \vl}  \bm\barY_{\vr, \vl} \|_q^2
    \\&=  \| P^{\otimes |\Ndzero|}_{\epsilon} \sum_{\vr, \vl} a_{\vr, \vl} \epsilon^{-\vr} \bm\barY_{\vr, \vl} \|_q^2 
    \\&\leq \|  \sum_{\vr, \vl} a_{\vr, \vl} \epsilon^{-\vr} \bm\barY_{\vr, \vl} \|_2^2
   \\&=
     \sum_{\vr, \vl} a_{\vr, \vl}^2 \epsilon^{-2\vr}
    \| \bm\barY_{\vr, \vl} \|_2^2
    \\&\leq 
    \epsilon^{-2\max|\vr|}
     \sum_{\vr, \vl} a_{\vr, \vl}^2 
    \| \bm\barY_{\vr, \vl} \|_2^2 
    \\&= (q-1)^{\max|\vr|} \|\vf\|_2^2  \leq  (q-1)^{s/\alpha_0} \|\vf\|_2^2  
\end{align}
\end{proof}

Since $r\notin\learning(\Gd)$, there is a $j$ such that $r_j<r< r_{j+1}$. 
Let $n(d)= d^{r}$ and 
\begin{align}
    m(d) =\dim\left( \Span\{\bm \barY_{\vr, \vl}: \learning(\vr) \leq r_j\}\right) = \dim(\Span \bigcup_{i\leq j} E_i)
\end{align}
Clearly, $m(d)\sim d^{r_j}$. 
We list all eigenvalues of $\gK$ in non-ascending order as  $\{\lambda_{d, i}\}$. In particular, we have 
\begin{align}
    \lambda_{d, m(d)}&\sim d^{-r_j} > d^{-r}
    >  d^{-r_{j+1}} \sim \lambda_{d, m(d)+1} \,. 
\end{align}

{\bf Assumption 4 (a).} 
We choose $u(d)$ to be 
\begin{align}
    u(d) = \dim\left(\Span \bigcup_{i: r_i\leq 2r+100} E_i\right) \,. 
\end{align}
% and list all eigenvalues $\lesssim d^{-s}$ by $\{\lambda_j\}_{j\geq j_0}$ in non-ascending order. Then $j_0=u(d)+1$ and $\lambda_{j_0} \sim d^{-s} < d^{-(2r+100)}$. 
{\bf Assumption 4 (a)} follows from Proposition \ref{prop:hyper}.  
\paragraph{Assumption 4 (b).}
Let $s=\inf\{\bar r\in \learning(\Gd): \bar r> 2r+100\}$.  
For $l> 1$, we have
\begin{align}
    \sum_{j=u(d)+1} \lambda_{d, j}^l \sim \sum_{r_i: r_i\geq s} (d^{-r_i})^l \dim(E_i^{}) \sim  d^{-s (l-1)}
\end{align}
which also holds for $l=1$ since 
\begin{align}
     \sum_{j=u(d)+1} \lambda_{d, j}  \sim 1 
\end{align}
Thus 
\begin{align}
\frac{    (\sum_{j=u(d)+1} \lambda_{d, j}^l)^2 }{\sum_{j=u(d)+1} \lambda_{d, j}^{2l}} \sim \frac{d^{-2s(l-1)}}{d^{-s(2l-1) }} = d^{s} > d^{2r+100 } > n(d)^{2+\delta} \sim d^{(2+\delta)r}.
\end{align}
as long as $\delta < 100/r$. 
% \newline 
% We list all eigenvalues $\lesssim d^{-s}$ by $\{\lambda_j\}_{j\geq j_0}$. By \myeqref{eq: dimension}, for $l > 1$,
% \begin{align}
% \sum_{j\geq j_0} \lambda_j^{l} \sim \sum_{i: r_i\geq s} \dim (E_i)d^{-lr_i}  \sim \sum_{i: r_i\geq s} d^{r_i}d^{-lr_i}
% = d^{-(l-1)s}
% \end{align}
% and for $l=1$, 
% \begin{align}
%     \sum_{j\geq j_0} \lambda_j^{l}  \sim 1 . 
% \end{align}
% Thus 
% \begin{align}
% \frac{(\sum_{j\geq j_0} \lambda_j^{l})^2 } 
% {\sum_{j\geq j_0} \lambda_j^{2l}}
% \sim  \frac{d^{-2s(l-1)}}{d^{-s(2l-1)}} = d^{s} > d^{2r+ 1}.
% \end{align}

{\bf Assumption 4 (c).} 
Denote 
\begin{align}
    \K_{d, >m(d)}(\bm\xi, \bm \eta) =  
    \sum_{\vr, \vl: \learning(\vr) > r} \lambda_\K(\vr)\bm\barY_{\vr, \vl}(\bm\xi)\bm\barY_{\vr, \vl}(\bm\eta)
\end{align}
We have
\begin{align}
    \mathbb E_{\bm \eta}\K_{d, >m(d)}(\bm\xi, \bm \eta)^2
    =&\mathbb E_{\bm \eta} \sum_{\vr, \vl: \learning(\vr) > r} \lambda_\K(\vr)^2 |\bm\barY_{\vr, \vl}(\bm\xi)\bm\barY_{\vr, \vl}(\bm\eta)|^2
    \\=& \sum_{\vr, \vl: \learning(\vr) > r} \lambda_\K(\vr)^2 |\bm\barY_{\vr, \vl}(\bm\xi)|^2 \\
    =&
    \sum_{\vr: \learning(\vr) > r} \lambda_\K(\vr)^2 \sum_{\vl}|\bm\barY_{\vr, \vl}(\bm\xi)|^2
    \\=&
    \sum_{\vr: \learning(\vr) > r} \lambda_\K(\vr)^2 \vN(\vd, \vr)
    \vP_\vr(\bm1)
    = \sum_{\vr: \learning(\vr) > r} \lambda_\K(\vr)^2 \vN(\vd, \vr)\,.
\end{align}
Similarly, 
\begin{align}
    &\mathbb E_{\bm \xi, \bm \eta} \sum_{\vr, \vl: \learning(\vr) > r} \lambda_\K(\vr)^2 |\bm\barY_{\vr, \vl}(\bm\xi)\bm\barY_{\vr, \vl}(\bm\eta)|^2 =\sum_{\vr: \learning(\vr) > r} \lambda_\K(\vr)^2 \vN(\vd, \vr) \, .
\end{align}
Thus 
\begin{align}
    \mathbb E_{\bm \eta} \sum_{\vr, \vl: \learning(\vr) > r} \lambda_\K(\vr)^2 |\bm\barY_{\vr, \vl}(\bm\xi)\bm\barY_{\vr, \vl}(\bm\eta)|^2 - \mathbb E_{\bm \xi, \bm \eta} \sum_{\vr, \vl: \learning(\vr) > r} \lambda_\K(\vr)^2 |\bm\barY_{\vr, \vl}(\bm\xi)\bm\barY_{\vr, \vl}(\bm\eta)|^2 = 0\,.
\end{align}
For the diagonal terms,  
\begin{align}
\K_{d, >m(d)}(\bm\xi, \bm \xi) =    \sum_{\vr, \vl: \learning(\vr) > r} \lambda_\K(\vr) |\bm\barY_{\vr, \vl}(\bm\xi)|^2
    = \sum_{\vr: \learning(\vr) > r} \lambda_\K(\vr) \vN(\vd, \vr) =\mathbb E_{\bm \xi}
    \K_{d, >m(d)}(\bm\xi, \bm \xi) 
\end{align}
which is deterministic. 

\subsection{Verifying {\bf Assumption 5}.}
Recall that $n(d)\sim d^r$ and $r_j < r < r_{j+1}$\,. 
{\bf Assumption 5(a)} follows from \myeqref{eq: dimension}. Indeed, for $l>1$ 
\begin{align}
\lambda_{d, m(d)+1}^{-l}\sum_{k = m(d)+1} \lambda_{d, k}^{l}\sim &
 (d^{-r_{j+1}})^{-l}\sum_{i: r_i\geq r_{j+1}} \dim (E_i)d^{-lr_i}  \\\sim &d^{lr_{j+1}}\sum_{i: r_i\geq r_{j+1}} d^{r_i}d^{-lr_i}
\\=& d^{r_{j+1}} > n(d)^{1 +\delta}
\end{align}
for some $\delta >0$ since $r<r_{j+1}$. Similarly, for $l=1$, since 
\begin{align}
    \sum_{k = m(d)+1} \lambda_{d, k}\sim 1 
\end{align}
we have 
\begin{align}
    (d^{-r_{j+1}})^{-1}\sum_{k = m(d)+1} \lambda_{d, k} \sim  d^{r_{j+1}} > n(d)^{1+\delta}\,. 
\end{align}

{\bf Assumption 5(b)} follows from $r>r_{j}$. 
\newline
{\bf Assumption 5(c).} Note that  
% To verify the spectral gap property, we need for some $\delta>0$, 
\begin{align}
\lambda_{d, m(d)}^{-1}\sum_{k = m(d)+1} \lambda_{d, k} \sim 
\lambda_{d, m(d)}^{-1} \sim d^{r_j} < n(d)^{(1-\delta)}\sim d^{r(1-\delta)}
    % (d^{-r_{j}})^{-1}\sum_{i: r_i\geq r_{j+1}} \sum_{\vr, \vl:\learning(\vr)=r_i} \lambda_{\K}(\vr) < 
    % |X|^{1-\delta} < |X|^{1+ \delta} < (d^{-r_{j+1}})^{-1}\sum_{i: r_i\geq r_{j+1}} \sum_{\vr, \vl:\learning(\vr)=r_i} \lambda_{\K}(\vr) 
\end{align}
for some $\delta>0$ since $r_j<r$. 
% i.e. 
% $$
% d^{r_j} <  |X|^{1-\delta} < |X|^{1+ \delta} < d^{r_{j+1}}
% $$
% which is obvious since $r_j< r< r_{j+1}$. 

\section{Proof of Lemma \ref{lemma:dimension}} 
\begin{proof} 
The lemma can be proved by induction. 
{\bf Base case: $L=0$.} The network is a S-CNN. WLOG, assume $\alpha_0\neq 0$ and $\alpha_1\neq 0$. For $r\in\learning(\Gd)$, we know that $r$ can be written as a combination of $\alpha_0$ and $\alpha_1$, i.e. $r = k_0\alpha_0 + k_1\alpha_1$ for some $k_0, k_1\geq 0$. We say a tuple $(k_0, k_1)$ is $r$-feasible if in addition, there exists $\vr$ with $\spatial(\vr) = k_1\alpha_1$  and $\frequency(\vr) = k_0\alpha_0$. 
Consider the set of all $r$-feasible tuple  
\begin{align}
F(r) \equiv \{ (k_0, k_1): r\text{-feasible} \} \, . 
\end{align}
Clearly, $F(r)$ is finite. It suffices to prove that for each $r$-feasible tuple $(k_0, k_1)$, 
\begin{align}
\text{dim}\left(\text{span}\left\{\bm \barY_{\vr, \vl}: \frequency(\vr) = k_0\alpha_0 \quad  \spatial(\vr) = k_1\alpha_1, \,\, \vl\in\bm N(\vd, \vr) \right\}\right) \sim d^{r}
\end{align}
Note that there are $\sim~(d^{\alpha_1})^{k_1}$ many ways to choose $k_1$ nodes in the penultimate layer. Then the dimension of the above set is about 
\begin{align}
   (d^{\alpha_1})^{k_1}  \sum_{\substack{(k_{0, 1}, \dots, k_{0, k_1})
   \\ k_{0,1} +\dots + k_{0, k_1} = k_0} }  \prod_{j=1}^{k_1} N(d^{\alpha_0}, k_{0, j}) &\sim 
   (d^{\alpha_1})^{k_1}  \sum_{\substack{(k_{0, 1}, \dots, k_{0, k_1})
   \\ k_{0,1} +\dots + k_{0, k_1} = k_0} } \prod_{j=1}^{k_1} (d^{\alpha_0})^{k_{0, j}} 
   \\
   &\sim d^{k_0\alpha_0 + k_1\alpha_1} = d^r \, , 
\end{align}
since the cardinality of the set 
$$\{(k_{0, 1}, \dots, k_{0, k_1}):  k_{0,1} +\dots + k_{0, k_1} = k_0,\quad  k_{0, j}\geq 1 \quad k_{0, j} \in\mathbb N\}$$
is finite. 

{\bf Induction step: $L\geq 1$.} For $\vr$ with $\learning(\vr)=r$, let $k$ be the number of children of a MST of $\nnode(\vr; \rootg)$. Clearly, $k\in [1, [r/\alpha_{L+1}]]$. Then we can classify $\bm\barY_{\vr, \vl}$ into at most $[r/\alpha_{L+1}]$ bins: $\{\Omega_k\}_{k=1, \dots, [r/\alpha_{L+1}]}$ depending on the number of children of $\rootg$ in a MST. Let $\Omega_k$ be non-empty. We only need to prove the number of $\bm \barY_{\vr, \vl}$ in $\Omega_k$ is $d^{r}$. Note that there are $\sim (d^{\alpha_{L+1}})^k$ many ways to choose $k$ children from $\rootg$. Let $\{u_j\}_{j=1, \dots, k}$ be one fixed choice and $\{\gG_j\}$ be the subgraphs with $\{u_j\}$ as the output nodes. Next, we partition $(r- k\alpha_{L+1})$ into $k$ components,  
$$r- k\alpha_{L+1} = r_1 + \dots + r_k$$
so that each $r_j$ is a combination of $\{\alpha_j\}_{0\leq j\leq L}$. The cardinality of such partition is also finite. We fix one of such partition $(r_1, \dots, r_k)$ so that each $r_j$ is a learning index of $\gG_j$. We can apply induction to each $(\gG_j, r_j)$ to conclude that the cardinality of $\bm \barY_{\vr_j, \vl_j}$ with $\learning_{\gG_j}(\vr_j) =r_j$ is $\sim d^{r_j}$, where $\learning_{\gG_j}(\vr_j)$ is the learning index of $\vr_j$ of $\gG_j$. Therefore, we have 
\begin{align}
    \text{dim}\left(\text{span}\left\{\bm \barY_{\vr, \vl}: \quad \learning(\vr) =r, \,\, \vl\,\, \in\bm N(\vd, \vr) \right\}\right) \sim (d^{\alpha_{L+1}})^{k} \prod_{j\in [k]} d^{r_j} = d^{r}\,. 
\end{align}

\end{proof} 

\section{CNN-GAP: CNNs with Global Average Pooling}\label{Sec:gap}
Consider convolutional networks whose readout layer is a global average pooling (GAP) and a flattening layer (namely, without pooling), resp. 
\begin{align}
&\textbf{CNN + GAP:}   \quad  &&[\textit{Conv}(p)\textit{-Act}]\to[\textit{Conv}(k)\textit{-Act}]^{\otimes L}\to[\textit{GAP}]\to[\textit{Dense}] \label{eq:gap-architectures}
\\
&\textbf{CNN:}   \quad  &&[\textit{Conv}(p)\textit{-Act}]\to[\textit{Conv}(k)\textit{-Act}]^{\otimes L}\to[\textit{Flatten}]\to[\textit{Dense}]
\end{align}
% The architectures is the same as that of $\textbf{D-CNN}$ except (1.) the flattening layer is replaced by the GAP (2.) no activation after the GAP layer.
Concretely, 
the input space is $\bmx = (\spbar)^{k^L\times w}\subseteq \mathbb R^{p\times 1\times k^L\times w}$, where $p$ is the patch size of the input convolutional layer, $k$ is the filter size in {\it hidden} layers, $L$ is the number of {\it hidden} convolution layers and $w$ is the spatial dimension of the penultimate layer. The total dimension of the input is $d=p\cdot k^L\cdot w$, and the number of input nodes is $|\leaf| = k^L\cdot w$. Since the stride is equal to the filter size for all convolutional layers, the {\it spatial} dimension is reduced by a factor of $p$ in the first layer, a factor of $k$ by each hidden layer. 
The penultimate layer (before pooling/flattening) has spatial dimension $w$ and is reduced to 1 by the {\it GAP-dense} layer or the {\it Flatten-dense} layer. 
The DAGs associated to these two architectures are identical which is denoted by $\gG$. However, the kernel/neural network computations are slightly different. 
If the penultimate layer has $n$ many channels and $f_{pen}: \bm\gX\to \mathbb R^{n\times w}$ is the mapping from the input layer to the penultimate layer, then the outputs of the {\it CNN-GAP}
and {\it CNN-Flatten} are 
\begin{align}
f_{\text{CNN-GAP}}(\vx) &=   n^{-\frac12}\sum_{j\in [n]}\omega_j \fint_{i\in [w]}f_{pen}(\vx)_{j, i}
\\
f_{\text{CNN-Flatten}}(\vx) &=   {(nw)}^{-\frac12}\sum_{j\in [n], i\in [w]}\omega_{ji} f_{pen}(\vx)_{j, i}\, ,
\end{align}
resp., where $w_j$ and $w_{ji}$ are parameters of the last layer and are usually initialized with standard iid Guassian $w_j, w_{ji}\sim \gN(0, 1)$.
Let $\penultimate\subseteq \gN$ be the nodes in the penultimate layer, then $|\penultimate| = w$. 
Let
$\bm\xi = (\bm \xi_{v})_{v\in \penultimate}\in \bmx$, where $\bm \xi_v\in (\spbar)^{k^L}$. Thus, each $\bm\xi_v$ contains $k^L$ many input nodes $\{\bm\xi_{u, i}\}_{i\in [k^L]}$. Define 
\begin{align}
    \vt_{uv} = \left(\langle \bm\xi_{u, i}, \bm\eta_{v, i}\rangle /p\right)_{i\in [k^L]} \in [-1, 1]^{k^L}. 
\end{align}
Recall that in the case without pooling 
\begin{align}
    \kk_{u}(\vt) = \phi^*(\fint_{uv\in\gE} \kk_{v}(\bm t) ),  \quad \kk_{\gG} = \fint_{\rootg v\in\gE} \kk_{v}(\bm t)  = \fint_{ v\in\penultimate} \kk_{v}(\bm t) 
\end{align}
where $\vt \in [-1, 1]^{k^L\times w}$, which is obtained by $\vt = (\langle \bm\xi_{u, i}, \bm\eta_{u, i}\rangle /p )_{u\in\penultimate, i\in [k^L]}$. Indeed, for each $v\in\penultimate$,  $\kk_v$ depends only on the {\it diagonal} terms $\vt_{vv} = (\langle \bm\xi_{v, i}, \bm\eta_{v, i}\rangle /p )_{ i\in [k^L]}\in [-1, 1]^{k^L}$. We can find a function 
\begin{align}
    \kkpen: [-1, 1]^{k^L} \to [-1, 1] \quad \text{s.t.}\quad \kk_v(\vt) = \kkpen (\vt_{vv}) \quad \forall v\in \penultimate
\end{align}
Therefore, without pooling the NNGP kernel is 
\begin{align}
    \kk_{\text{CNN-Flatten}}(\vt) = \fint_{v\in\penultimate} \kkpen(\vt_{vv})  = \frac 1 w \sum_{v\in\penultimate} \kkpen(\vt_{vv}) 
\end{align}
Note that the kernel does not depend on any {\it off-diagonal} terms $\vt_{uv}$ with $u\neq v$ because there isn't weight-sharing in the last layer. 
Let $\vd_{pen} = (p, p, \dots, p)\in \mathbb N^{ k^L}$. Then $\|\vd_{pen}\|_1=pk^L$ is the effective dimension of the input to any node  $u\in\penultimate$. Assume $k=d^{\alpha_k}$, $p=d^{\alpha_p}$ and $w=d^{\alpha_w}$ and $\alpha_k, \alpha_p, \alpha_w>0$. 
Applying Theorem \ref{thm:main} to $\kkpen$, we have 
\begin{align}
    \kk_{\text{CNN-Flatten}}(\vt) =
    \sum_{\vr \in \mathbb N^{k^L}} 
    \frac 1 w \lambda_{\kk_{pen}}(\vr) 
    \sum_{v\in\penultimate} \sum_{\vl\in \vN(\vd_{pen}, \vr)} \bm\barY_{\vr, \vl}(\bm\xi_v), \bm\barY_{\vr, \vl}(\bm\eta_v)
    % \fint_{v\in\penultimate} \kkpen(\vt_{vv})  = \frac 1 w \sum_{v\in\penultimate} \kkpen(\vt_{vv}) 
\end{align}
Clearly, the eigenfunctions are $\{\bm\barY_{\vr, \vl}(\bm\xi_v)\}_{\vr, \vl, v}$ and the corresponding eigenvalues are
$\{\frac 1 w \lambda_{\kk_{pen}}(\vr) \}_{\vr, \vl, v}$\,. Note that \begin{align}
    \frac 1 w \lambda_{\kk_{pen}}(\vr) = \lambda_{\kk_\gG}(\vr)\sim d^{-\learning(\vr)}
\end{align}
Here and in what follows, we also treat $\vr\in \mathbb N^{k^L}$ as an element in $\mathbb N^{wk^L}$.
\newline 
When the readout layer is a GAP, the weights of the penultimate layer are shared across different spatial locations, namely, all nodes in $\penultimate$ use the same weight. As such, the kernel corresponding to the CNN-GAP depends on both the diagonal and off-diagonal terms $\vt=(\vt_{uv})_{u,v\in\penultimate} \in [-1, 1]^{k^L\times k^L}$, which can be written as \citep{novak2018bayesian} 
\begin{align}
     \kk_{\text{CNN-GAP}}(\vt) = \fint_{u,v\in\penultimate} \kkpen(\vt_{uv})  = \frac 1 {w^2} \sum_{u, v\in\penultimate} \kkpen(\vt_{uv}) 
\end{align}
Applying Theorem \ref{thm:main} to $\kkpen$, we have 
\begin{align}
    \kk_{\text{CNN-GAP}}(\vt) &= \frac{1}{w^2}
    \sum_{\vr \in \mathbb N^{k^L}} \lambda_{\kk_{pen}}(\vr) \sum_{u, v\in\penultimate}
    \sum_{\vl\in \vN(\vd_{pen}, \vr)} \bm\barY_{\vr, \vl}(\bm\xi_u), \bm\barY_{\vr, \vl}(\bm\eta_v)
    \\&= 
    \sum_{\vr \in \mathbb N^{k^L}} \frac 1 w \lambda_{\kk_{pen}}(\vr)
    \sum_{\vl\in \vN(\vd_{pen}, \vr)} 
    \left( w^{-\frac 1 2}\int_{u\in\penultimate} \bm\barY_{\vr, \vl}(\bm\xi_u)\right) 
    \left( w^{-\frac 1 2}\int_{u\in\penultimate} \bm\barY_{\vr, \vl}(\bm\eta_u)\right)
    \\&=   \sum_{\vr \in \mathbb N^{k^L}} \frac 1 w \lambda_{\kk_{pen}}(\vr)
    \sum_{\vl\in \vN(\vd_{pen}, \vr)}  \bm\barY_{\vr,\vl}^{\text{Sym}}(\bm\xi)
    \bm\barY_{\vr,\vl}^{\text{Sym}}(\bm\eta)
\end{align}
where 
\begin{align}\label{eq:sym-polynomials}
    \bm\barY_{\vr,\vl}^{\text{Sym}}(\bm\xi) \equiv  w^{-\frac 1 2}\int_{u\in\penultimate} \bm\barY_{\vr, \vl}(\bm\xi_u)\quad, \text{for}\quad  \vr\in \mathbb N^{k^L} \text{and} \quad\vl\in \vN(\vd_{pen}, \vr) 
\end{align}
That is the eigenfunctions and eigenvalues of $\kk_{\text{CNN-GAP}}$ are 
$\{\bm\barY^{\text{Sym}}_{\vr, \vl}(\bm\xi)\}_{\vr, \vl}$ and 
$\{\frac 1 w \lambda_{\kk_{pen}}(\vr) \}_{\vr, \vl}$ resp. 

In sum, the eigenvalues of $\kk_{\text{CNN-Flatten}}$ and $\kk_{\text{CNN-GAP}}$ are the same (up to the multiplicity factor $w$). Each eigenspace of $\kk_{\text{CNN-GAP}}$ is given by symmetric polynomials of the form \myeqref{eq:sym-polynomials}. We can see that the GAP reduces the dimension of each eigenspace by a factor of $w$. Same arguments can also be applied to the NTKs \citep{novak2019neural, yang2019scaling}
\begin{align}
    \Theta_{\text{CNN-Flatten}}(\vt) &= \fint_{v\in\penultimate} (\kk_{pen}(\vt_{vv}) + \Theta_{pen}(\vt_{vv}))\\&=     \sum_{\vr \in \mathbb N^{k^L}} \left(
    \frac 1 w \lambda_{\kk_{pen}}(\vr) + 
    \frac 1 w \lambda_{\Theta_{pen}}(\vr) \right)
    \sum_{v\in\penultimate} \sum_{\vl\in \vN(\vd_{pen}, \vr)} \bm\barY_{\vr, \vl}(\bm\xi_v), \bm\barY_{\vr, \vl}(\bm\eta_v)
    \\
     \Theta_{\text{CNN-GAP}}(\vt) &= \fint_{u,v\in\penultimate} (\kk_{pen}(\vt_{uv}) + \Theta_{pen}(\vt_{uv}))\,
     \\&=     \sum_{\vr \in \mathbb N^{k^L}} \left(
    \frac 1 w \lambda_{\kk_{pen}}(\vr) + 
    \frac 1 w \lambda_{\Theta_{pen}}(\vr) \right)
    \sum_{v\in\penultimate} \bm\barY^\sym_{\vr, \vl}(\bm\xi), \bm\barY^\sym_{\vr, \vl}(\bm\eta)
\end{align}
where $\Theta_{pen}$ is the NTK of the penultimate layer which is the same for all nodes in $\gN_{-1}$.  
Since $\frac 1 w \lambda_{\Theta_{pen}}\sim d^{-\learning(\vr)}$, we have 
\begin{align}
    \frac 1 w \lambda_{\kk_{pen}}(\vr) + 
    \frac 1 w \lambda_{\Theta_{pen}}(\vr)  \sim d^{-\learning(\vr)}
\end{align}
\subsection{Generalization bound of CNN-GAP}
We show that GAP improves the data efficiency of D-CNNs by a factor of $w\sim d^{\alpha_w}$ under a stronger assumptions on activations $\phi$. 
\newline
{\bf Assumption Poly-$\phi$:} There is a sufficiently large $J\in\mathbb N$ such that for all hiddens nodes $u$
\begin{align}
    {\phi^{*}_u}^ {(j)}(0) \neq 0 \quad \text{for} \quad 1\leq j \leq J \quad \text{and}\quad  {\phi^{*}_u}^ {(j)}(0) =0 \quad \text{otherwise} 
\end{align}
This assumption implies that there are $0<J_1< J_2\in\mathbb R$ such that 
for all $\bm 0\neq \vr$ with $|\vr|<J_1$ 
\begin{align}
    \frac{d^{\vr}}{d\vt} \kk_{pen}(\bm 0) \neq 0  \quad  \text{and} \quad   \frac{d^{\vr}}{d\vt} \Theta_{pen}(\bm 0) \neq 0  
\end{align}
and for all $\vr$ with $|\vr|>J_2$ 
\begin{align}
    \frac{d^{\vr}}{d\vt} \kk_{pen}\equiv \bm 0    \quad  \text{and} \quad 
    \frac{d^{\vr}}{d\vt} \Theta_{pen}\equiv \bm 0 
\end{align}
Moreover, $J_1\to\infty$ as $J\to\infty$.  

Let $L_{\text{Sym}}^p(\bmx) \leq L^p(\bmx)$ be the close subspace spanned by symmetric eigenfunctions \myeqref{eq:sym-polynomials}. Let $\gK_{\text{Sym}} = \kk_{\text{CNN-GAP}}$ or $\Theta_{\text{CNN-GAP}}$.

For $X\subseteq \bmx$ and $r\notin \gL(\Gd)$, define the regressor and the projection operator to be  
\begin{align*}
    &\bm R^{\text{Sym}}_{X}(f)(x) = \gK_{\text{Sym}}(x, X) \gK_{\text{Sym}}(X, X)^{-1}f(X)
\\ &\bm P^{\text{Sym}}_{>r}(f) =
    \sum_{\vr: \gL(\vr)> r}\sum_{\vl\in \vN(\vd_{pen}, \vr)}\langle f, \bm\barY^{\text{Sym}}_{\vr, \vl}\rangle_{L^2(\bmx)} \bm\barY^{\text{Sym}}_{\vr, \vl} \, . 
\end{align*}

\begin{theorem}\label{thm:generalization-gap}
Let $\bm \gG = \{\Gd\}_d$, where each $\Gd$ is a DAG associated to the D-CNN in \myeqref{eq:gap-architectures} with $\alpha_k, \alpha_p, \alpha_w>0$. Let $r\notin \gL(\Gd)$ be fixed and the activations satisfy {\bf Assumption Poly-$\phi$} for $J=J(r)$ sufficiently large. 
Let $f\in L^2_{\text{Sym}}(\bmx)$ with $\mathbb E_{\bm\sigma} f = 0$. Then for $\epsilon >0$, 
\begin{align}
\left|    \left\|\bm R^{\text{Sym}}_X(f) - f\right\|_{L^2_{\text{Sym}}(\bmx)}^2 - 
    \left\|\bm P^{\text{Sym}}_{>r}(f)\right\|_{L^2_{\text{Sym}}(\bmx)}^2 \right| 
    = c_{d, \epsilon} \|f\|_{{L_{\text{Sym}}^{2+\epsilon}(\bmx)}}^2, \quad 
\end{align}
where $c_{d, \epsilon} \to 0$ in probability as $d\to\infty$ over $X\sim\bm\sigma^{[d^{r - \alpha_w}]}$. 
\end{theorem} 

\begin{proof}[Proof of Theorem \ref{thm:generalization-gap}]
We need the following dimension counting lemma, which follows directly from Lemma \ref{lemma:dimension}.
\begin{lemma}\label{lemma:dimension-gap}
Let $r\in \learning(\Gd).$ Then 
\begin{align}
\dim\left(\Span\left\{\bm \barY^{\sym}_{\vr, \vl}: \quad \gL(\vr) =r, \,\, \vl\,\, \in\bm N(\vd, \vr) \right\}\right) \sim d^{r - \alpha_w} 
\end{align}
\end{lemma}

Recall that $\learning(\Gd )= \{r_j\}$ in non-descending order. Similarly, let 
\begin{align}
    E_i^{\text{Sym}} = \Span\{\bm\barY_{\vr, \vl}^{\text{Sym}}:
    \quad \learning(\vr)= r_i\}
\end{align}
From the above lemma, we have 
\begin{align}
    \dim(E_i^{\sym}) \sim d^{r_i - \alpha_w}
\end{align}
Since $r\notin\learning(\Gd)$, there is a $j$ such that $r_j<r< r_{j+1}$. 
Let $n(d)= d^{r-\alpha_w}$ and 
\begin{align}
    m(d) =\dim\left( \Span\{\bm \barY_{\vr, \vl}^{\text{Sym}}: \learning(\vr) \leq r_j\}\right) = \dim(\Span \bigcup_{i\leq j} E_i^\sym )
\end{align}
Clearly, $m(d)\sim d^{r_j-\alpha_w}$. 
We list all eigenvalues of $\K_\sym$ in non-ascending order as  $\{\lambda_{d, i}\}$. In particular, we have 
\begin{align}
    \lambda_{d, m(d)}&\sim d^{-r_j} > d^{-r}
    >  d^{-r_{j+1}} \sim \lambda_{d, m(d)+1} \,. 
\end{align}
We proceed to verify Assumptions 4 and 5 in Sec.~\ref{sec: kernel concentraction}.
\paragraph{Assumptions 4 (a)} 
We choose $u(d)$ to be 
\begin{align}
    u(d) = \dim\left(\Span \bigcup_{i: r_i\leq 2r+100} E_i^\sym\right) \,. 
\end{align}
Let $s=\inf\{\bar r\in \learning(\Gd): \bar r> 2r+100\}$. 
% and list all eigenvalues $\lesssim d^{-s}$ by $\{\lambda_j\}_{j\geq j_0}$ in non-ascending order. Then $j_0=u(d)+1$ and $\lambda_{j_0} \sim d^{-s} < d^{-(2r+100)}$. 
Assumption 4 (a) follows from Proposition \ref{prop:hyper}.  
\paragraph{Assumptions 4 (b)}
For $l> 1$, we have
\begin{align}
    \sum_{j=u(d)+1} \lambda_{d, j}^l \sim \sum_{r_i: r_i\geq s} (d^{-r_i})^l \dim(E_i^{\sym}) \sim  d^{-s (l-1) -\alpha_w}
\end{align}
which also holds for $l=1$ since 
\begin{align}
     d^{\alpha_w}\sum_{j=u(d)+1} \lambda_{d, j}  \sim 1 
\end{align}
Thus 
\begin{align}
\frac{    (\sum_{j=u(d)+1} \lambda_{d, j}^l)^2 }{\sum_{j=u(d)+1} \lambda_{d, j}^{2l}} \sim \frac{d^{-2s(l-1) -2 \alpha_w}}{d^{-s(2l-1) - \alpha_w}} = d^{s-\alpha_w} > d^{2r+100-\alpha_w } > n(d)^{2+\delta} \sim d^{(2+\delta)(r-\alpha_w)}.
\end{align}

{\bf Assumption 4 (c)} This requires some work and is verified in Sec.~\ref{sec:verication2}. 

{\bf Assumption 5 (a)} 
Since $m(d) = \dim(\Span \bigcup_{i\leq j} E_i)$ and $r_{j+1}>r>r_j$, we have 
\begin{align}
\frac {1}{\lambda_{d, m(d)+1}}\sum_{j\geq m(d)+1} \lambda_{d, j}^l &\sim 
    \frac{1}{(d^{-r_{j+1}})^l}\sum_{i > j} (d^{-r_i})^l \dim(E_i^\sym)
    \\&\sim \dim(E_{j+1}^\sym) = d^{r_{j+1}-\alpha_w} 
    \\&> n(d)^{1+\delta} = d^{(r-\alpha_w)(1+\delta)}
\end{align}
as long as $\delta < (r_{j+1}-\alpha_w)/(r-\alpha_w)-1$.

{\bf Assumption 5 (b)} This is obvious since $m(d)\sim d^{r_j-\alpha_w}$, $n(d) \sim d^{r-\alpha_w}$ and $r> r_j$.

{\bf Assumption 5 (c)} This follows from $r_j<r$. Indeed, 
\begin{align}
\frac {1}{\lambda_{d, m(d)}}\sum_{j\geq m(d)+1} \lambda_{d, j} &\sim 
    \frac{1}{(d^{-r_{j}})}\sum_{i > j} (d^{-r_i}) \dim(E_i^\sym)
    \sim d^{r_{j}-\alpha_w} 
    \\&\leq n(d)^{1-\delta} = d^{(r-\alpha_w)(1-\delta)}
\end{align}
% \begin{align}
%     \frac{1}{d^{-r_{j}}}\sum_{i > j} d^{-r_i} \dim(E_i^\sym) 
%     \sim \frac{1}{d^{-r_{j}}} d^{-\alpha_w} = d^{r_j-\alpha_w}
%     \leq n(d)^{1-\delta} = d^{(r-\alpha_w)(1-\delta)}
% \end{align}
as long as $0<\delta<1- (r_j-\alpha_w) /(r-\alpha_w)$.
\end{proof}

\subsection{Verification of Assumptions 4(c).}\label{sec:verication2}
We begin with proving \myeqref{eq:diagonal-concentration-1}. 
Let $X_i$ define the random variable 
\begin{align}
    X_i \equiv \mathbb E_{\vxi \sim \bm\sigma_d} \K_{\sym, >m(d)}(\vxi_i,  \vxi)^2 \quad \text{and} \quad \Delta_i \equiv X_i - \mathbb E X_i  
\end{align}
where 
\begin{align}
    \K_{\sym, >m(d)}(\vxi,  \veta) 
    = \sum_{\vr, \vl: \learning(\vr) > r} \lambda_{\K_{\sym}}(\vr)\bm\barY^\sym_{\vr, \vl}(\bm\xi)\bm\barY^\sym_{\vr, \vl}(\bm\eta) 
\end{align}
and $\lambda_{\K_{\sym}}$ is the eigenvalue of $\bm\barY^\sym_{\vr, \vl}$. 
We need to show that 
\begin{align}
    \frac{\sup_{i\in [n(d)]} |\Delta_i | } {\mathbb E X_i}\xrightarrow[d\to\infty]{\text {in prob.}} 0 . 
\end{align}
By Markov's inequality, it suffices to show that 
\begin{align}
    ({\mathbb E X_i})^{-1}  \mathbb E {\sup_{i\in [n(d)]} |\Delta_i | } \xrightarrow[d\to\infty]{} 0
\end{align}
By orthogonality and treating $\vr\in \mathbb N^{k^L}$ as an element of $\mathbb N^{wk^L}$, we have  
\begin{align}
    \mathbb E X_i &= \mathbb E_{\vxi, \bar \vxi\sim \bm\sigma_d} \K_{\sym, >m(d)}(\vxi, \bar \vxi)^2
    \\& = \mathbb E_{\vxi, \bar \vxi\sim \bm\sigma_d}  \left |\sum_{\bm r: \learning(\bm r) >r}\lambda_{\K_{\sym}}\sum_{\vl \in\bm N(\vd_{pen}, \vr)}\bm\barY^\sym_{\bm r, \bm l}(\vxi) \bm\barY^\sym_{\bm r, \bm l}(\bar \vxi)\right|^2
    \\&= 
    \mathbb E_{\vxi, \bar \vxi\sim \bm\sigma_d}   
    \sum_{\bm r: \learning(\bm r) >r}\lambda_{\K_{\sym}}^2(\vr)\sum_{\vl \in\bm N(\vd_{pen}, \vr)}\left |\bm\barY^\sym_{\bm r, \bm l}(\vxi) \bm\barY^\sym_{\bm r, \bm l}(\bar \vxi)\right|^2
    \\&= 
    \sum_{\bm r: \learning(\bm r) >r}\lambda_{\K_{\sym}}^2(\vr)\sum_{\vl \in\bm N(\vd_{pen}, \vr)} \mathbb E_{\vxi, \bar \vxi\sim \bm\sigma_d}   \left |\bm\barY^\sym_{\bm r, \bm l}(\vxi) \bm\barY^\sym_{\bm r, \bm l}(\bar \vxi)\right|^2
    \\&= 
    \sum_{\bm r: \learning(\bm r) >r}\lambda_{\K_{\sym}}^2(\vr)\sum_{\vl \in\bm N(\vd_{pen}, \vr)}  1 
    \\&= \sum_{\bm r: \learning(\bm r) >r}\lambda_{\K_{\sym}}^2(\vr) \vN (\bm d_{pen}, \vr)
    \label{eq:temp3}
\end{align}
and 
\begin{align}
    X_i &=\mathbb E_{\vxi \sim \bm\sigma_d} \K_{\sym, >m(d)}(\vxi_i,  \vxi)^2
    \\&= \mathbb E_{\vxi\sim \bm\sigma_d}  \left |\sum_{\bm r: \learning(\bm r) >r}\lambda_{\K_{\sym}}\sum_{\vl \in\bm N(\vd_{pen}, \vr)}\bm\barY^\sym_{\bm r, \bm l}(\vxi_i) \bm\barY^\sym_{\bm r, \bm l}( \vxi)\right|^2
    \\&= 
    \mathbb E_{\vxi\sim \bm\sigma_d}   
    \sum_{\bm r: \learning(\bm r) >r}\lambda_{\K_{\sym}}^2(\vr)\sum_{\vl \in\bm N(\vd_{pen}, \vr)}\left |\bm\barY^\sym_{\bm r, \bm l}(\vxi_i) \bm\barY^\sym_{\bm r, \bm l}(\vxi)\right|^2
    \\&= 
    \sum_{\bm r: \learning(\bm r) >r}\lambda_{\K_{\sym}}^2(\vr)\sum_{\vl \in\bm N(\vd_{pen}, \vr)} \mathbb E_{  \vxi\sim \bm\sigma_d}   \left |\bm\barY^\sym_{\bm r, \bm l}(\vxi_i) \bm\barY^\sym_{\bm r, \bm l}( \vxi)\right|^2
    \\&= 
    \sum_{\bm r: \learning(\bm r) >r}\lambda_{\K_{\sym}}^2(\vr)\sum_{\vl \in\bm N(\vd_{pen}, \vr)}  |\bm\barY^\sym_{\bm r, \bm l}(\vxi_i)|^2
    \\&= 
     \sum_{\bm r: \learning(\bm r) >r}\lambda_{\K_{\sym}}^2(\vr)\sum_{\vl \in\bm N(\vd_{pen}, \vr)}\left(\frac 1 {w}  \sum_{u}\bm\barY_{\bm r, \bm l}(\vxi_{i, u})^2 
     + \frac 1 {w}  \sum_{u\neq v }\bm\barY_{\bm r, \bm l}(\vxi_{i, u}) \bm\barY_{\bm r, \bm l}(\vxi_{i, v})\right)
    \\&=  \mathbb E_{\vxi, \bar \vxi\sim \bm\sigma_d} \K_{\sym, >m(d)}(\vxi, \bar \vxi)^2
    + 
    \sum_{\bm r: \learning(\bm r) >r}\lambda_{\K_{\sym}}^2(\vr)\sum_{\vl \in\bm N(\vd_{pen}, \vr)}\left( \frac 1 {w}  \sum_{u\neq v }\bm\barY_{\bm r, \bm l}(\vxi_{i, u}) \bm\barY_{\bm r, \bm l}(\vxi_{i, v})\right)
\end{align}
Let 
\begin{align}
    X_{\vr, i} = \sum_{\vl \in\bm N(\vd_{pen}, \vr)}\left( \frac 1 {w}  \sum_{u\neq v }\bm\barY_{\bm r, \bm l}(\vxi_{i, u}) \bm\barY_{\bm r, \bm l}(\vxi_{i, v})\right)
\end{align}
then 
\begin{align}\label{eq:temp2}
    \Delta_i = \sum_{\bm r: \learning(\bm r) >r}\lambda_{\K_{\sym}}^2(\vr)  X_{\vr, i} 
\end{align}
We replace the maximal function by the $l^q$-norm for $q\geq 1$,  
\begin{align}
    \mathbb E \sup_{i\in [n(d)]}|\Delta_i| 
    \leq \mathbb E (\sum_{i\in [n(d)]}|\Delta_i| ^q)^{\frac 1q} 
    \leq (\mathbb E \sum_{i\in [n(d)]}|\Delta_i| ^q)^{\frac 1q}  
    = n(d)^{\frac 1q} (\mathbb E |\Delta_i| ^q)^{\frac 1q} 
    % \leq 
    % n(d)^{\frac 1q} C_{q}(\mathbb E |\Delta_i| ^2)^{\frac 12} 
\end{align}
where the first three expectations are taken over $\bm\xi_1, \dots, \bm\xi_{n(d)}\sim \bm\sigma_d$ and the last one is taken over $\bm\xi_i \sim \bm \sigma_d$. 
Then we replace the $L^q$-norm by the $L^2$-norm via Hypercontractivity, in which we used {\bf Assumption Poly-$\phi$}
which implies that $ \Delta_i$ is a polynomial of bounded degree 
\begin{align}\label{eq:temp1} 
    \mathbb E \sup_{i\in [n(d)]}|\Delta_i| \leq 
    n(d)^{\frac 1q} (\mathbb E |\Delta_i| ^q)^{\frac 1q} 
    \leq 
    C_qn(d)^{\frac 1q} (\mathbb E |\Delta_i| ^2)^{\frac 12} 
    =
    C_q n(d)^{\frac 1q} (\mathbb E \Delta_i ^2)^{\frac 12} 
\end{align}
We expand the $L^2$-norm and use orthogonality to ``erase" the off-diagonal terms twice: first for $\vr\neq \bar\vr $
\begin{align}
    \mathbb E_{\bm\xi_i \sim \bm \sigma_d} X_{\vr, i} X_{\bar\vr, i} = 0
\end{align}
and second for $\vl\neq \vl '$ or $u\neq v$
\begin{align}
    \mathbb E_{\bm\xi_i \sim \bm \sigma_d} X_{\vr, i}^2 &= 
    \frac 1 {w^2} \mathbb E_{\bm\xi_i \sim \bm \sigma_d} 
    \sum_{\vl \in\bm N(\vd_{pen}, \vr)}\left(  \sum_{u\neq v }\bm\barY_{\bm r, \bm l}(\vxi_{i, u}) \bm\barY_{\bm r, \bm l}(\vxi_{i, v})\right)
    \sum_{\bm l'} \left(  \sum_{u'\neq v' }\bm\barY_{\bm r, \bm l'}(\vxi_{i, u'}) \bm\barY_{\bm r, \bm l'}(\vxi_{i, v'})\right)
    \\
    &=
    \frac 2 {w^2} \mathbb E_{\bm\xi_i \sim \bm \sigma_d} 
    \sum_{\vl \in\bm N(\vd_{pen}, \vr)}\left(  \sum_{u\neq v }\bm\barY_{\bm r, \bm l}(\vxi_{i, u})^2 \bm\barY_{\bm r, \bm l}(\vxi_{i, v})^2\right)
    =\frac 2 {w^2} \sum_{\vl \in\bm N(\vd_{pen}, \vr)}w(w-1) \\&= \frac {2(w-1)} w  \sum_{\bm l}1  = \frac {2(w-1)} w \bm N(\bm d_{pen}, \vr) \leq 2  \bm N(\bm d_{pen}, \vr)
\end{align}
Combining this estimate with \myeqref{eq:temp3},  \myeqref{eq:temp2} and  \myeqref{eq:temp1} yields 
\begin{align}
    (\mathbb E X_i)^{-1} (\mathbb E \sup_{i\in [n(d)]}|\Delta_i|)  &\leq 
     C_q n(d)^{\frac 1q}
    \frac
    {(2 
     \sum_{\bm r: \learning(\bm r) >r}\lambda_{\K_{\sym}}^4(\vr)
    \bm N(\bm d_{pen}, \vr) )^{1/2} }
    {{\sum_{\bm r: \learning(\bm r) >r}\lambda_{\K_{\sym}}^2(\vr) \vN (\bm d_{pen}, \vr) }} 
    \\
    &\leq 
    \sqrt 2 C_q n(d)^{\frac 1q} \frac
    {
     \sum_{\bm r: \learning(\bm r) >r}\lambda_{\K_{\sym}}^2(\vr)
    \bm N(\bm d_{pen}, \vr)^{1/2}  }
    {{\sum_{\bm r: \learning(\bm r) >r}\lambda_{\K_{\sym}}^2(\vr) \vN (\bm d_{pen}, \vr) }} 
    \\
    &\leq 
   \sqrt 2  C_q n(d)^{\frac 1q}  
    \sup_{\bm r: \learning(\bm r) >r}\frac
    {
    \lambda_{\K_{\sym}}^2(\vr)
    \bm N(\bm d_{pen}, \vr)^{1/2}  }
    {{\lambda_{\K_{\sym}}^2(\vr) \vN (\bm d_{pen}, \vr) }} 
    \\
    &\sim  2  C_q d ^{\frac rq}  
    \sup_{\bm r: \learning(\bm r) >r}\vN (\bm d_{pen}, \vr)^{-\frac 1 2}
    \\&\sim d^{\frac r q } \sup_{\bm r: \learning(\bm r) >r} d^{-|\vr|\alpha_p/2}
    \xrightarrow[d\to\infty]{} 0  
\end{align}
by choosing $q$ (independent of $d$) sufficiently large.

The proof of \myeqref{eq:diagonal-concentration-2} is similar. 
Let $X_i$ denote the random variable 
\begin{align}
    X_i \equiv \K_{\sym, >m(d)}(\vxi_i,  \vxi_i) \quad \text{and} \quad \Delta_i \equiv X_i - \mathbb E X_i  
\end{align}
and it suffices to prove 
\begin{align}
\frac{  \mathbb E  \sup_{i\in [n(d)]}|\Delta_i|}
{\mathbb E X_i} \xrightarrow[d\to\infty]{} \,\,0 
\end{align}
We have
\begin{align}
    \mathbb E X_i &=\mathbb E_{\vxi_i \sim \bm\sigma_d} \K_{\sym, >m(d)}(\vxi_i,  \vxi_i)
    \\&= \mathbb E_{\vxi_i\sim \bm\sigma_d}   
    \sum_{\bm r: \learning(\bm r) >r}\lambda_{\K_{\sym}}(\vr)\sum_{\vl \in\bm N(\vd_{pen}, \vr)}\bm\barY^\sym_{\bm r, \bm l}(\vxi_i) \bm\barY^\sym_{\bm r, \bm l}( \vxi_i)
    \\&= 
     \sum_{\bm r: \learning(\bm r) >r}\lambda_{\K_{\sym}}(\vr) \bm N(\vd_{pen}, \vr)
\end{align}
and 
\begin{align}
    \Delta_i = X_i - \mathbb E X_i &= 
    w^{-1} 
    \sum_{\bm r: \learning(\bm r) >r}\lambda_{\K_{\sym}}(\vr)
    \sum_{\vl \in\bm N(\vd_{pen}, \vr)}
     \sum_{u\neq v}
     \bm\barY_{\bm r, \bm l}
    (\vxi_{i, u}) \bm\barY_{\bm r, \bm l}( \vxi_{i, v}),.
\end{align}
The remaining steps (replacing the maximal function by the $l^q$-norm, and then the $L^q$-norm by the $L^2$-norm using hypercontractivity, etc.) are similar to that of the proof of \myeqref{eq:diagonal-concentration-1}, which are omitted here.

\section{Kernel Concentration, Hypercontractivity and Generalization from 
Mei-Misiakiewicz-Montanari}
\label{sec: kernel concentraction}
For convenience, we briefly recap the analytical results regarding generalization bounds of kernel machines from \citet{mei2021generalization} Sec 3.

Let $(\bm X_d, \bm\sigma_d)$ be a probability space and $\hop$ be a compact self-adjoint positive definite operator from $L^2(\bm X_d, \bm\sigma_d)\to L^2(\bm X_d, \bm\sigma_d)$. We assume $\hop\in L^2(\bm X_d\times \bm X_d)$. Let $\{\psi_{d, j}\}$ and $\{\lambda_{d, j}\}$ be the eigenfunctions and eigenvalues associated to $\hop$, i.e. 
\begin{align}
    \hop \psi_{d, j}(x) \equiv \int_{y\in\bm X_d} \hop(x,y) \psi_{d, j}(y) \bm\sigma_d(y) = \lambda_{d, j} \psi_{d, j}(x). 
\end{align}
We assume the eigenvalues are in non-ascending order, i.e. $\lambda_{d, j+1} \geq \lambda_{d, j}\geq 0.$ Note that 
\begin{align}
    \sum_j \lambda_{d, j}^2 = \|\hop\|^2_{L^2(\bm X_d \times \bm X_d)} < \infty. 
\end{align}
The associated reproducing kernel Hilbert space (RKHS) is defined to be functions $f\in L^2(\bm X_d, \bm\sigma_d)$ with $\|\hop^{-\frac 1 2} f\|_{L^2(\bm X_d, \bm \sigma_d)}<\infty$. Given a finite training set $X\subseteq \bm X_d$ and observed labels $f(X)\in\mathbb R^{|X|}$, the regressor is an extension operator defined to be 
\begin{align}
    \mathscr R_X f(x) = \hop(x, X) \hop(X, X)^{-1}f(X)\,.  
\end{align}
Intuitively, when ``$X\to \bm X_d$" in some sense, we expect the following 
\begin{align}
    \mathscr R_X f(x) = \hop(x, X) \hop(X, X)^{-1}f(X) \to 
     \mathscr R_{\bm X_d} f(x) = \hop (\hop^{-1}f)(x) = f(x), 
\end{align}
namely, ``$\mathscr R_X\to\bm {I}_{\bm X_d}$" in some sense. 

Leveraging tools from the non-asymptotic analysis of random matrices \citep{vershynin2010introduction}, the work \cite{mei2021generalization} provides a very nice answer to the above question in terms of the decay property of the eigenvalues $\{\lambda_{d, j}\}$ and the hypercontractivity property of the eigenfunctions $\{\psi_{d, j}\}$. They show that $\mathscr R_X$ is essentially a projection operator onto the low eigenspace under certain regularity assumptions on the operator $\hop$. These assumptions are stated via the relationship between the number of (training) samples $n=n(d)$, the tail behavior of the eigenvalues with index $\geq m=m(d)$ and the tail behavior of the operator $\hop$
\begin{align}
    \mathscr H_{d, >m(d)}(x, \bar x) \equiv \sum_{j > m(d)+1} \lambda_j \psi_j(x) \psi_j(\bar x)
\end{align}
as the "input dimension" $d$ becomes sufficiently large. 

{\bf Assumption 4.} We say that the the sequence of operator $\{\hop\}_{d\geq 1}$ satiesfies the Kernel Concentration Property (KCP) with respect to the sequence $\{n(d), m(d)\}_{d\geq 1}$ if there exsts a sequence of integers $\{u(d)\}_{d\geq 1}$ with $u(d)\geq m(d)$ such that the following holds 
\begin{enumerate}[label=(\alph*)]
    \item ({\bf Hypercontractivity.}) Let $D_{u(d)} = \Span\{\psi_j: 1\leq j \leq u(d)\}$. Then for any fixed $q\geq 1$, and $C=C(q)$ such that for $\vf\in D_{u(d)}$ 
    \begin{align}
        \|\vf\|_{L^q(\bm X_d, \bm\sigma_d)} \leq C\|\vf\|_{L^2(\bm X_d, \bm\sigma_d)}
    \end{align}
    \item ({\bf Eigen-decay.}) There exists $\delta > 0$, such that, for all $d$ large enough, for $l=1$ and $2$, 
    \begin{align}
        n(d)^{2+\delta} \leq \frac{(\sum_{j\geq u(d)+1} \lambda_{d, j}^l)^2 }{ \sum_{j\geq u(d)+1} \lambda_{d, j}^{2l}}
    \end{align}
    \item ({\bf Concentration of Diagonals.}) For $\{x_i\}_{i\in [n(d)]} \sim \bm \sigma_d^{n(d)}$, we have: 
    \begin{align}
        &\frac{\sup_{i\in [n(d)]} \left| \mathbb E_{x\sim \bm \sigma_d}{\mathscr H}_{d, >m(d)}(x_i, x)^2 - \mathbb E_{x, \bar x\sim \bm \sigma_d} {\mathscr H}_{d, >m(d)}(x, \bar x)^2\right| }{ \mathbb E_{x, \bar x \sim \bm \sigma_d} {\mathscr H}_{d, >m(d)}(x, \bar x)^2} \xrightarrow[d\to\infty]{\text{in Prob.}} \,0
        \label{eq:diagonal-concentration-1}
        \\ 
        &\frac{\sup_{i\in [n(d)]} \left| \hop{}_{,>m(d)}(x_i, x_i) - \mathbb E_{x\sim \bm \sigma_d} {\mathscr H}_{d, >m(d)}(x, x)\right| }{ \mathbb E_{x\sim \bm \sigma_d} {\mathscr H}_{d, >m(d)}(x, x)}\xrightarrow[d\to\infty]{\text{in Prob.}} \,0
          \label{eq:diagonal-concentration-2}
    \end{align}
    where $c_d \to 0 $ in probability as $d\to\infty$. 
\end{enumerate}

{\bf Assumption 5.} Let $\hop$ and $\{m(d), n(d)\}_{d\geq1}$ be the same as above.  
\begin{enumerate}[label=(\alph*)]
    \item For $l=1$ and $2$, there exists $\delta>0$ such that 
    \begin{align}
        n(d)^{1+\delta} \leq \frac{1} {\lambda_{d, m(d)+1}^l} \sum_{k=\lambda_{m(d)+1}} \lambda_{d, k}^{l}
    \end{align}
    \item There exists $\delta >0$ such that 
    \begin{align}
        m(d) \leq n(d)^{1-\delta}
    \end{align}
    \item (Spectral Gap.) There exists $\delta >0$ such that 
    \begin{align}
        n(d)^{1-\delta} \geq \frac{1}{\lambda_{d, m(d)}} \sum_{k\geq m(d)+1} \lambda_{d, k}
    \end{align}
\end{enumerate}
Let $\mathscr P_{>k}$ (similarly for $\mathscr P_k$, $\mathscr P_{\leq k}$, etc.) denote the projection operator 
\begin{align}
    \mathscr P_{>k} f  =\sum_{j> k} \langle f, \psi_j\rangle \psi_j
\end{align}
\begin{theorem}[\citet{mei2021generalization}]
Assume $\hop$ satisfy {\bf Assumptions 4} and {\bf 5}. Let $\{f_d\}_{d\geq 1}$ be a sequence of functions and let $X\sim \bm \sigma_{d}^{n(d)}$. Then for every $\epsilon>0$,   
\begin{align}
    \|\mathscr R_X(f_d) - f_d\|_{L^2(\bm X_d, \bm\sigma_d)}^2 = \|\mathscr P_{>m(d)} f_d\|_{L^2(\bm X_d, \bm\sigma_d)}^2 + c_{d, \epsilon} \|f_d\|_{L^{2+\epsilon}(\bm X_d, \bm\sigma_d)}^2
\end{align}
where $c_{d, \epsilon}\to 0$ in probability as $d\to\infty$. 
\end{theorem}
The theorem says, $\mathscr R_X$ is essentially the projection operator $\mathscr P_{\leq m(d)}$ in the sense that when restricted to $ L^{2+\epsilon}(\bm X_d, \bm \sigma_d)$, 
\begin{align}
    \mathscr R_X= \mathscr P_{\leq m(d)} + \text{Error}_{d, \epsilon}  \, . 
\end{align}

\end{document}